\documentclass[nohyperref]{article}

\usepackage{microtype}
\usepackage{graphicx}
\usepackage{booktabs} % for professional tables

\usepackage[small]{caption}
\usepackage{subcaption}

\usepackage{threeparttable}
\usepackage{footnote}
\usepackage{tablefootnote}
\usepackage{footmisc}

\PassOptionsToPackage{hyphens}{url}
\usepackage{hyperref}

\usepackage[accepted]{icml2022}

\usepackage{amsmath}
\usepackage{amssymb}
\usepackage{mathtools}
\usepackage{amsthm}

\usepackage[capitalize,noabbrev]{cleveref}

\theoremstyle{plain}
\newtheorem{theorem}{Theorem}[section]

\newtheorem{lemma}[theorem]{Lemma}
\newtheorem{corollary}[theorem]{Corollary}
\theoremstyle{definition}
\newtheorem{definition}[theorem]{Definition}

\theoremstyle{remark}

\usepackage{xcolor}
\definecolor{custom_blue}{HTML}{1F77B4}
\definecolor{custom_orange}{HTML}{FF7F0E}
\definecolor{custom_purple}{HTML}{9467BD}
\definecolor{custom_green}{HTML}{2CA02C}
\definecolor{custom_red}{HTML}{D62728}

\usepackage{tikz}
\usepackage{pgfplots}

\newcommand{\blue}{\raisebox{2pt}{\tikz{\draw[custom_blue, solid, line width=2.3pt](0,0) -- (5mm,0);}}}
\newcommand{\orange}{\raisebox{2pt}{\tikz{\draw[custom_orange, solid, line width=2.3pt](0,0) -- (5mm,0);}}}
\newcommand{\green}{\raisebox{2pt}{\tikz{\draw[custom_green, solid, line width=2.3pt](0,0) -- (5mm,0);}}}
\newcommand{\purple}{\raisebox{2pt}{\tikz{\draw[custom_purple, solid, line width=2.3pt](0,0) -- (5mm,0);}}}

\usepackage[textsize=tiny]{todonotes}

\icmltitlerunning{Safe and Robust Experience Sharing for Deterministic Policy Gradient Algorithms}

\begin{document}

\twocolumn[
\icmltitle{Safe and Robust Experience Sharing for\\Deterministic Policy Gradient Algorithms}



% in order of appearance, and this is the preferred way.

\begin{icmlauthorlist}
\icmlauthor{Baturay Saglam}{bilkent}
\icmlauthor{Dogan C. Cicek}{bilkent}
\icmlauthor{Furkan B. Mutlu}{bilkent}
\icmlauthor{Suleyman S. Kozat}{bilkent}
\end{icmlauthorlist}

\icmlaffiliation{bilkent}{Department of Electrical and Electronics Engineering, Bilkent University, 06800 Bilkent, Ankara, Turkey}

\icmlcorrespondingauthor{Baturay Saglam}{baturay@ee.bilkent.edu.tr}

\icmlkeywords{Deep Reinforcement Learning, Off-Policy Learning, Experience Sharing}

\vskip 0.3in
]



\printAffiliationsAndNotice{}  % leave blank if no need to mention equal contribution

\begin{abstract}
Learning in high dimensional continuous tasks is challenging, mainly when the experience replay memory is very limited. We introduce a simple yet effective experience sharing mechanism for deterministic policies in continuous action domains for the future off-policy deep reinforcement learning applications in which the allocated memory for the experience replay buffer is limited. To overcome the extrapolation error induced by learning from other agents' experiences, we facilitate our algorithm with a novel off-policy correction technique without any action probability estimates. We test the effectiveness of our method in challenging OpenAI Gym continuous control tasks and conclude that it can achieve a safe experience sharing across multiple agents and exhibits a robust performance when the replay memory is strictly limited.  
\end{abstract}

%%%%%%%%%%%%%%%%%%%%%% BU CUMLEYI BIR YERDE KULLAN, CONCLUSION OLABILIR %%%%%%%%%%%%%%%%%%%%%%
% Our introduced off-policy correction method combines the gradients of multiple actors and critics to use the experiences of other agents safely
%%%%%%%%%%%%%%%%%%%%%% BU CUMLEYI BIR YERDE KULLAN, CONCLUSION OLABILIR %%%%%%%%%%%%%%%%%%%%%%

\section{Introduction}
Off-policy deep reinforcement learning requires large amounts of interactions with the environment to obtain optimal policies~\cite{trust_region_estimators}. As the observation and action spaces of an environment start to increase and more challenging tasks are introduced, the memory requirement for the experience replay~\cite{experience_replay} dramatically increases~\cite{off_policy}. Therefore, regardless of the experience replay sampling algorithms, with limited memory, off-policy deep RL algorithms should exhibit high-level performance for future real-world applications.

Sharing experience among concurrent agents remains an effective alternative when the available off-policy data is limited as it can allow faster convergence due to diverse exploration~\cite{dual_policy_dist}. However, learning from other agents' experiences may lead to the \textit{extrapolation error}, a phenomenon caused by the mismatch between the distributions corresponding to the off-policy data collected by a different agent and the latest agent's policy~\cite{off_policy}. The extrapolation error may lead unseen state-action pairs to be erroneously estimated and have unrealistic values~\cite{off_policy}. Hence, for safe and reliable experience sharing among multiple agents, off-policy correction (or importance sampling) is required to eliminate the extrapolation error induced by other agents' experiences. Although off-policy correction and experience sharing mechanisms are well-studied artifacts for discrete~\cite{impala,munos_safe,trust_region_estimators} and continuous~\cite{a3c} action domains through the action probabilities of stochastic policies, in the deterministic and continuous policy case, action probability estimation and thus, importance sampling, is not a possible option by the nature of the policies as there is not any probability distribution from which the actions are sampled.

Motivated by the possible restrictions to the allocated memory of the replay buffer and limitations of deterministic policies, we introduce an actor-critic architecture that enables a diverse and robust parallel learning. Our approach is not affected by extrapolation error by safely correcting the experiences gathered by multiple agents. An extensive set of experiments demonstrate that with only two agents that learn the environment in parallel, our architecture obtains an optimal performance when the size of the replay buffer is very limited. Moreover, we show that the introduced algorithm can significantly improve the state-of-the-art even when the replay buffer is unlimited. Ultimately, our ablation studies validate that the extrapolation occurs when the off-policy samples are not corrected, and our modifications can enable effective filtering to overcome this problem.  Our code and results are available at the GitHub repository\footnote{\url{https://github.com/baturaysaglam/DASE}\label{our_repo}}.

\section{Technical Preliminaries}
We follow the standard reinforcement learning paradigm, where at each discrete time step $t$, the agent observes a state $s$ and chooses an action $a$; then, it receives a reward $r$ and observes a new state $s'$. The policy of an agent aims to maximize the \textit{value} defined as the expected cumulative discounted returns $R_{t} = \sum_{i = 0}^{\infty}\gamma^{i}r_{t + i}$, where $\gamma \in [0, 1)$ is a discount factor to prioritize the short-term rewards. A policy $\pi_{\phi}(\cdot)$, parameterized by $\phi$, is stochastic if it maps states to action probabilities, $a \sim \pi_{\phi}(\cdot|s)$, or deterministic if it maps states to unique actions, $a = \pi_{\phi}(s)$. The action-value function (Q-function or critic) evaluates the action decisions of an agent in terms of the value $R_{t}$. The deep Q-network, $Q_{\theta}$ with parameters $\theta$, estimates action-values (or Q-values).  

In off-policy learning, an agent encounters transitions generated by a family of behavior policies. We consider a multiple agent case where $K$ agents explore the same environment asynchronously and store their experiences in a shared replay buffer. At every update step, the agent samples a batch of transitions through a sampling algorithm that may contain on- and off-policy samples:
\begin{gather}
    (\boldsymbol{S}_{I}^{|\mathcal{B}| \times m}, \boldsymbol{A}_{I}^{|\mathcal{B}| \times n}, \boldsymbol{R}_{I}^{|\mathcal{B}| \times 1}, \boldsymbol{S}_{I}^{'|\mathcal{B}| \times m}) \sim \mathcal{B}_{I},\label{eq:internal_batch} \\ (\boldsymbol{S}_{E}^{|\mathcal{B}| \times m}, \boldsymbol{A}_{E}^{|\mathcal{B}| \times n}, \boldsymbol{R}_{E}^{|\mathcal{B}| \times 1}, \boldsymbol{S}_{E}^{'|\mathcal{B}| \times m}) \sim \mathcal{B}_{E},\label{eq:external_batch}
\end{gather}
where $|\mathcal{B}|$ is the number of transitions in the sampled batch, $m$ and $n$ are the state and action dimensions, respectively, and bold letters represent vectors or matrices of row vectors. We categorize transitions into the ones executed by the agent in interest and the ones executed by other agents which we call \textit{internal} \textit{(own)} and \textit{external} experiences, respectively. We also refer to entities that corresponds to these two types of experiences as, again, \textit{external} and \textit{internal} entities, e.g., \textit{external policies}, \textit{internal actions}, \textit{external states}. Therefore, $\mathcal{B}_{I}$ and $\mathcal{B}_{E}$ are the internal and external parts of the sampled mixed batch, respectively, yielding $\mathcal{B}_{I} \cup \mathcal{B}_{E} = \mathcal{B}$ and $\mathcal{B}_{I} \cap \mathcal{B}_{E} = \emptyset$. 

\section{Method}
\subsection{Deterministic Policy Similarity}
To enable a safe experience sharing across multiple and independent agents that learn in parallel, we first aim to mitigate the extrapolation error~\cite{off_policy}. We obtain this by constructing a novel policy similarity metric for deterministic policies. Before presenting the primary component of our architecture, we start with a basic assumption on the actions chosen by the deterministic policy. We assume without loss of generality that each continuous action selected by a behavioral policy is a sample of a multivariate Gaussian distribution which is not known during the training. Intuitively, each dimension in an action vector is correlated to the rest of the dimensions. This is realistic since each dimension in an action vector often has effects on the other dimensions~\cite{mujoco}. Mainly, the mean vector represents the deterministic action chosen by the policy, and the covariance matrix represents the noise introduced by the exploration, deep function approximation, and bootstrapping in Q-learning~\cite{q_learning}.

Having our assumption made, we now show how to derive the similarity weights for the off-policy experiences. The agent samples a batch of off-policy transitions corresponding to different behavioral policies in each gradient step. We know that given the states $\boldsymbol{S}^{|\mathcal{B}| \times m}$, each action in the experience replay buffer~\cite{experience_replay} corresponds to a multivariate Gaussian distribution $\mathcal{N}(\boldsymbol{\mu}^{n \times 1}, \boldsymbol{\Sigma}^{n \times n})$ with mean vector $\boldsymbol{\mu}^{n \times 1}$ and covariance matrix $\boldsymbol{\Sigma}^{n \times n}$. To measure the similarity between the current policy and the policies that executed the off-policy transitions in the external batch $\mathcal{B}_{E}$, we first forward pass the states from $\mathcal{B}_{E}$ through the behavioral actor network corresponding to the current policy:
\begin{equation}
\label{eq:det_current_actions}
    \boldsymbol{\hat{A}}_{E}^{|\mathcal{B}_{E}| \times n} = \pi_{\phi}(\boldsymbol{S}_{E}^{|\mathcal{B}_{E}| \times m}).
\end{equation}
We now have the batch of current policy's decisions on the states from the off-policy transitions $\boldsymbol{\hat{A}}_{E}^{|\mathcal{B}_{E}| \times n}$, and the batch of past policies' decisions $\boldsymbol{A}_{E}^{|\mathcal{B}_{E}| \times n}$ from $\mathcal{B}_{E}$. Let $\boldsymbol{\dot{A}}_{E}^{|\mathcal{B}_{E}| \times n}$ be the batch of numerical differences in the action decisions:
\begin{equation}
\label{eq:action_diff_batch_def}
    \boldsymbol{\dot{A}}^{|\mathcal{B}_{E}| \times n} \coloneqq \boldsymbol{A}_{E}^{|\mathcal{B}_{E}| \times n} - \boldsymbol{\hat{A}}_{E}^{|\mathcal{B}_{E}| \times n}.
\end{equation}
Observe that $\boldsymbol{\dot{A}}_{E}^{|\mathcal{B}_{E}| \times n}$ indicates the deviation between the current policy and previous behavioral policies of the agent that generated the off-policy transitions. To construct a multivariate Gaussian distribution from the action difference batch, let:
\begin{gather}
\label{eq:multi_var_gaussian_mean}
\boldsymbol{\dot{\mu}}^{n \times 1} = \frac{1}{|\mathcal{B}_{E}|}\smashoperator{\sum_{i = 1}^{|\mathcal{B}_{E}|}}(\boldsymbol{\dot{A}}^{|\mathcal{B}_{E}| \times n}_{i})^\top, \\
\boldsymbol{\dot{\Sigma}}^{n \times n} = \frac{1}{|\mathcal{B}_{E}| - 1}\smashoperator{\sum_{i = 1}^{|\mathcal{B}_{E}|}}\boldsymbol{a}^{n \times 1}_{i}(\boldsymbol{a}^{n \times 1}_{i})^\top, \label{eq:multi_var_gaussian_cov_matrix}
\end{gather}
where $\boldsymbol{\dot{A}}^{|\mathcal{B}_{E}| \times n}_{i}$ represents the action as a row vector corresponding to the $i^{\text{th}}$ transition and $\boldsymbol{a}^{n \times 1}_{i} = (\boldsymbol{\dot{A}}^{|\mathcal{B}_{E}| \times n}_{i})^\top - \boldsymbol{\dot{\mu}}^{n \times 1}$. Then, define the dissimilarity measure as:
\begin{equation}
\label{eq:complete_dissimilarity_measurement}
    \rho = \mathrm{JSD}(\mathcal{N}(\boldsymbol{\dot{\mu}}^{n \times 1}, \boldsymbol{\dot{\Sigma}}^{n \times n}) \parallel \mathcal{N}(\boldsymbol{0}^{n \times 1}, \sigma\boldsymbol{I}^{n \times n})),
\end{equation}
where $\mathrm{JSD}$ is the Jensen-Shannon divergence, $\sigma$ is the standard deviation of the exploration noise, and $\boldsymbol{I}$ is the identity matrix. We do not directly compare $\mathcal{N}(\boldsymbol{\dot{\mu}}^{n \times 1}, \boldsymbol{\dot{\Sigma}}^{n \times n})$ with a zero multivariate Gaussian since the policies closer to the current policy may be rejected as the actions may deviate from the policy's actual action decisions due to the additive exploration noise. Furthermore, we choose JSD for asymmetric similarity measurement as the similarity of two policy distributions should not be assumed to be directed. Although KL-divergence is well-known for penalizing a distribution that is completely different from the distribution in interest, two policies in the same environment cannot be completely distinct~\cite{sutton_book}. Naturally, if all the internal and external experiences correspond to the same policy, then $\rho = 0$ and $\rho \in (0, \infty)$ otherwise. To project the similarity measure into the interval $[0, 1]$, a non-linear transformation can be applied:
\begin{equation}
\label{eq:complete_similarity_measurement}
    \boldsymbol{\lambda}^{|\mathcal{B|}_{E} \times 1} = [e^{-\rho},\ e^{-\rho},\ \dots,\ e^{-\rho}]^\top.
\end{equation}
We choose the exponential function for non-linear transformation to slowly smooth the dissimiliarity. A sharp smoothing would be very greedy and may penalize the external transitions too much. Observe that two identical policies have $\boldsymbol{\lambda}^{|\mathcal{B}_{E}| \times 1} = \boldsymbol{1}$, and distinct policies have $\boldsymbol{\lambda}^{|\mathcal{B}_{E} \times 1} = \boldsymbol{0}$, making Equation (\ref{eq:complete_similarity_measurement}) a similarity measure between two policies. This forms the backbone of our architecture, which we refer Deterministic Policy Similarity (DPS), summarized in Algorithm \ref{alg:js_importance_weighting}.

Intuitively, DPS first computes the numerical difference between the actions chosen by the current and external policies, then compares the difference with zero. This is equivalent to comparing the distributions of the current policy and the policies that executed the external transitions under the multivariate Gaussian distribution assumption. One concern with DPS may be that the minority of the transitions within the batch $\mathcal{B}_{E}$ may be executed by the policies very similar to the current agent's policy. Since we take the average of external action batch in computing the similarity weight, i.e., Equation (\ref{eq:multi_var_gaussian_mean}), those transitions may be weighted by a fixed weight close to 0, which results in loss of information. Nevertheless, since the function approximators in off-policy actor-critic methods are often optimized through mini-batch learning, it should be expected that policies correspond to the majority of the transitions in $\mathcal{B}_{E}$ must be close to the current policy~\cite{off_policy}. 

\subsection{Deterministic Actor-Critic with Shared Experience}
Now, we are ready to introduce our architecture, Deterministic Actor-Critic with Shared Experience (DASE). DASE considers multiple agents, each of which explores different copies of the same environment, i.e., they learn in parallel and does not interact with each other except for a shared experience replay buffer. At every update step, each agent samples a batch of transitions and updates its actor and critic networks by combining internal gradients and DPS weighted external gradients. Through DPS, our architecture enables agents to safely use other agents' experiences by resolving the issues with stability due to the potential exploding gradients, i.e., similarity weights restricted are within the interval $[0, 1]$, and extrapolation error~\cite{off_policy}, i.e., by allowing only the transitions correlated to the distribution under the current policy. 

% Evaluation Results Unlimited Buffer Condition
\begin{figure*}[h]
    \centering
    \vspace{-0.4cm}
    \begin{align*}
        &\text{{\blue} TD3 (single agent)} \quad &&\text{{\orange} TD3 + DASE ($1^{\text{st}}$ agent)} \\
        &\text{\text{{\purple} TD3 + DPD (average of two agents)}} \quad &&\text{{\green} TD3 + DASE ($2^{\text{nd}}$ agent)}
    \end{align*}
	\subfloat[\textbf{Replay Size:} 100,000]{
		\includegraphics[width=1.29in, keepaspectratio]{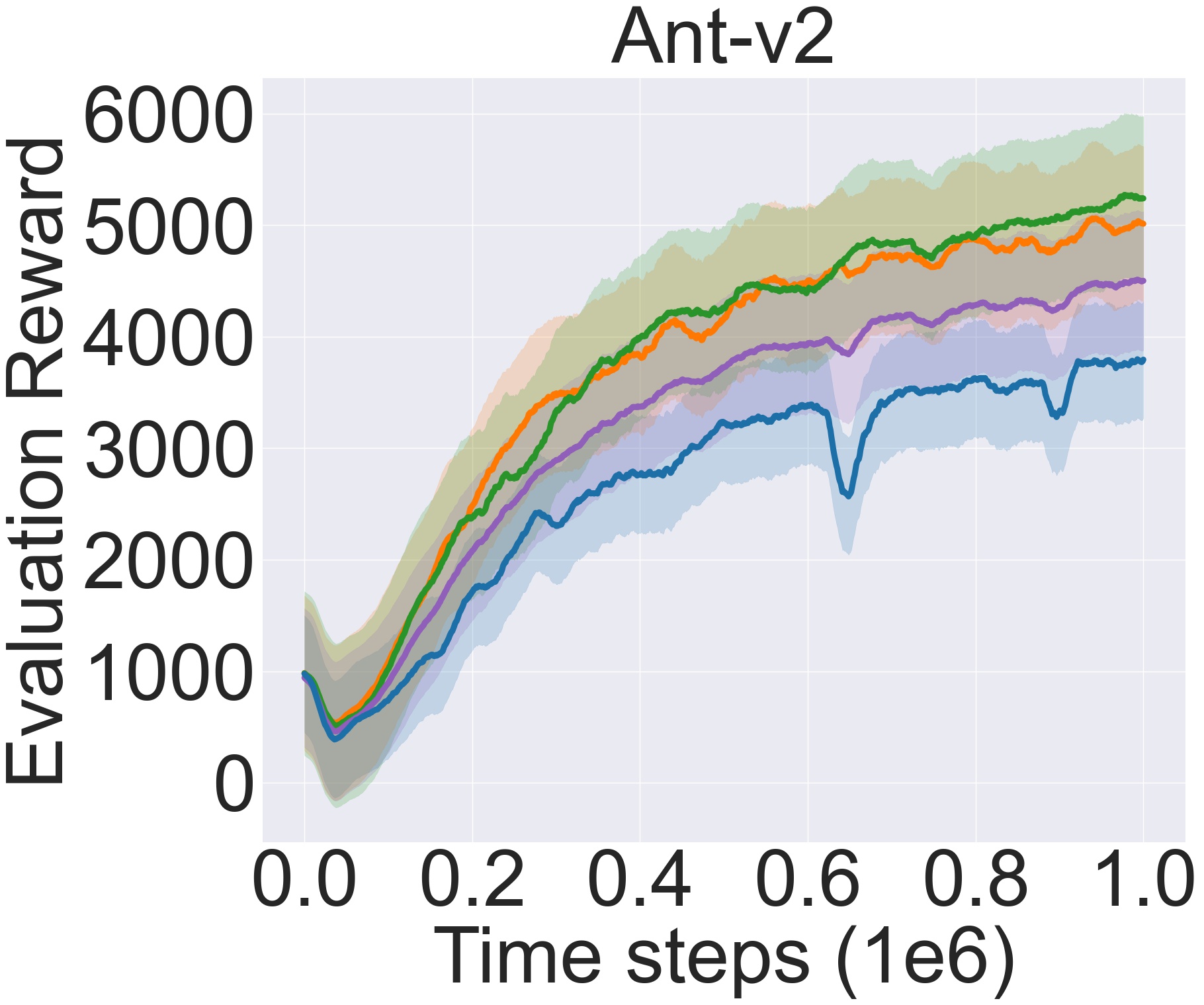}
		\includegraphics[width=1.29in, keepaspectratio]{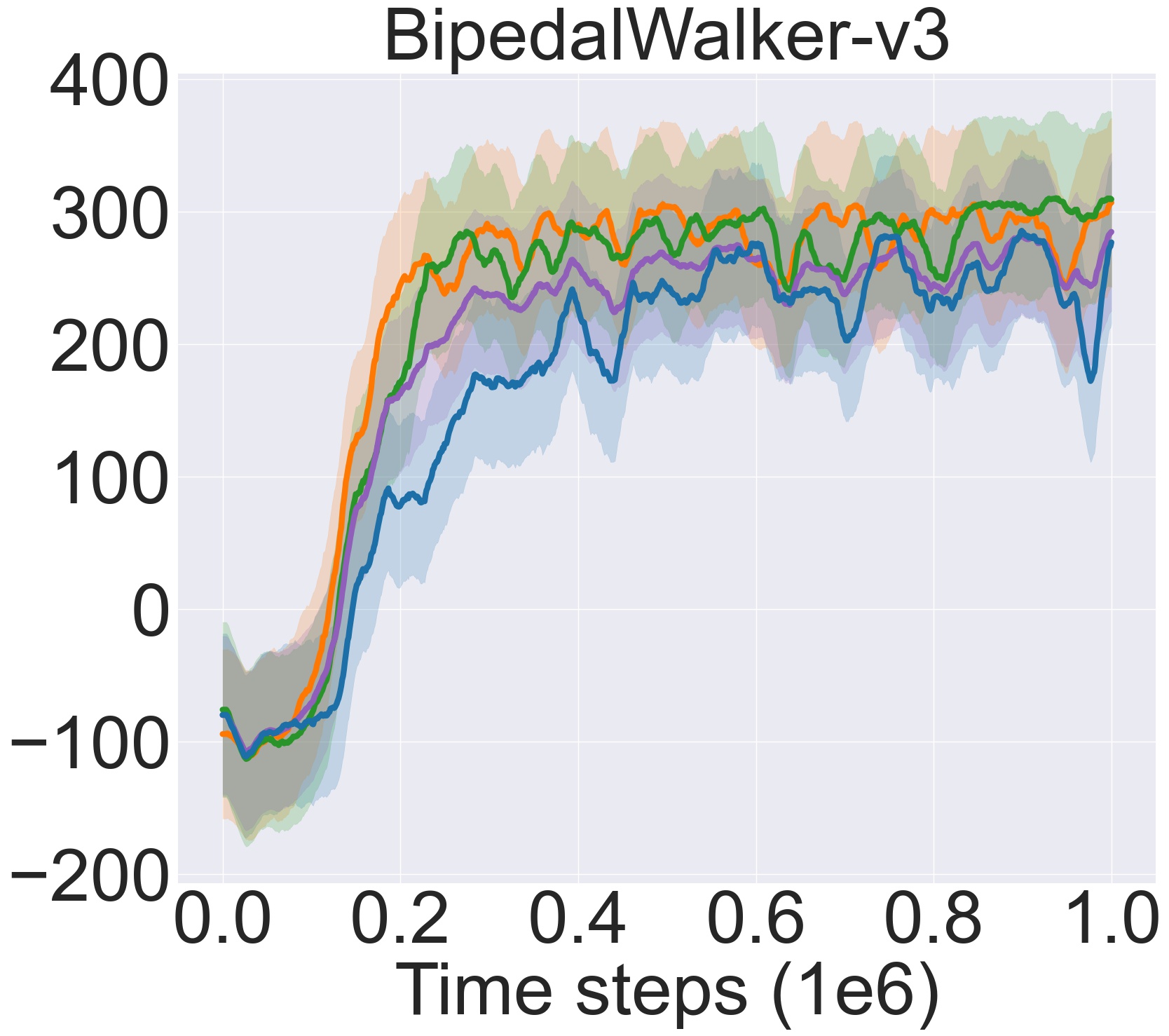}
		\includegraphics[width=1.29in, keepaspectratio]{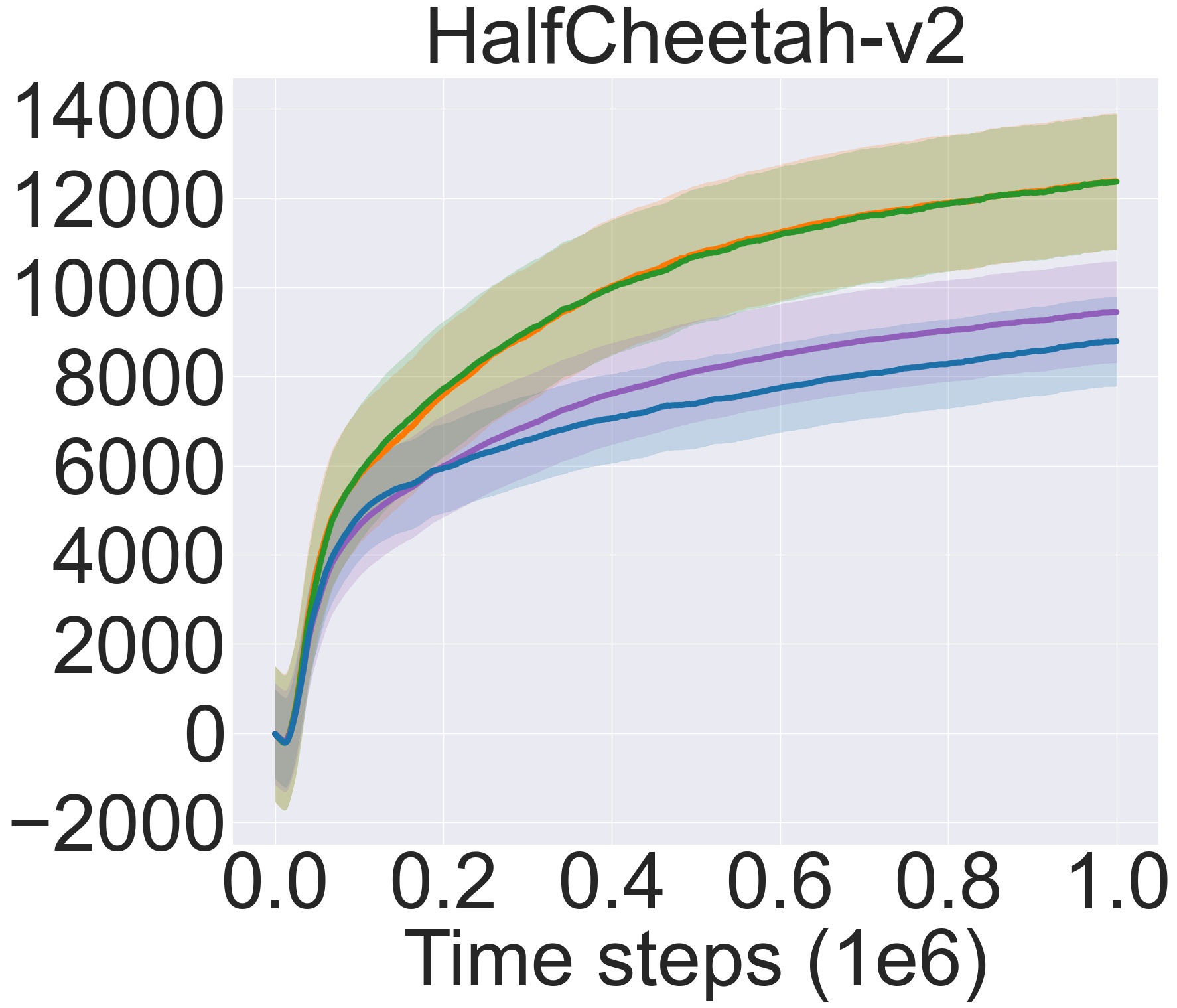}
		\includegraphics[width=1.29in, keepaspectratio]{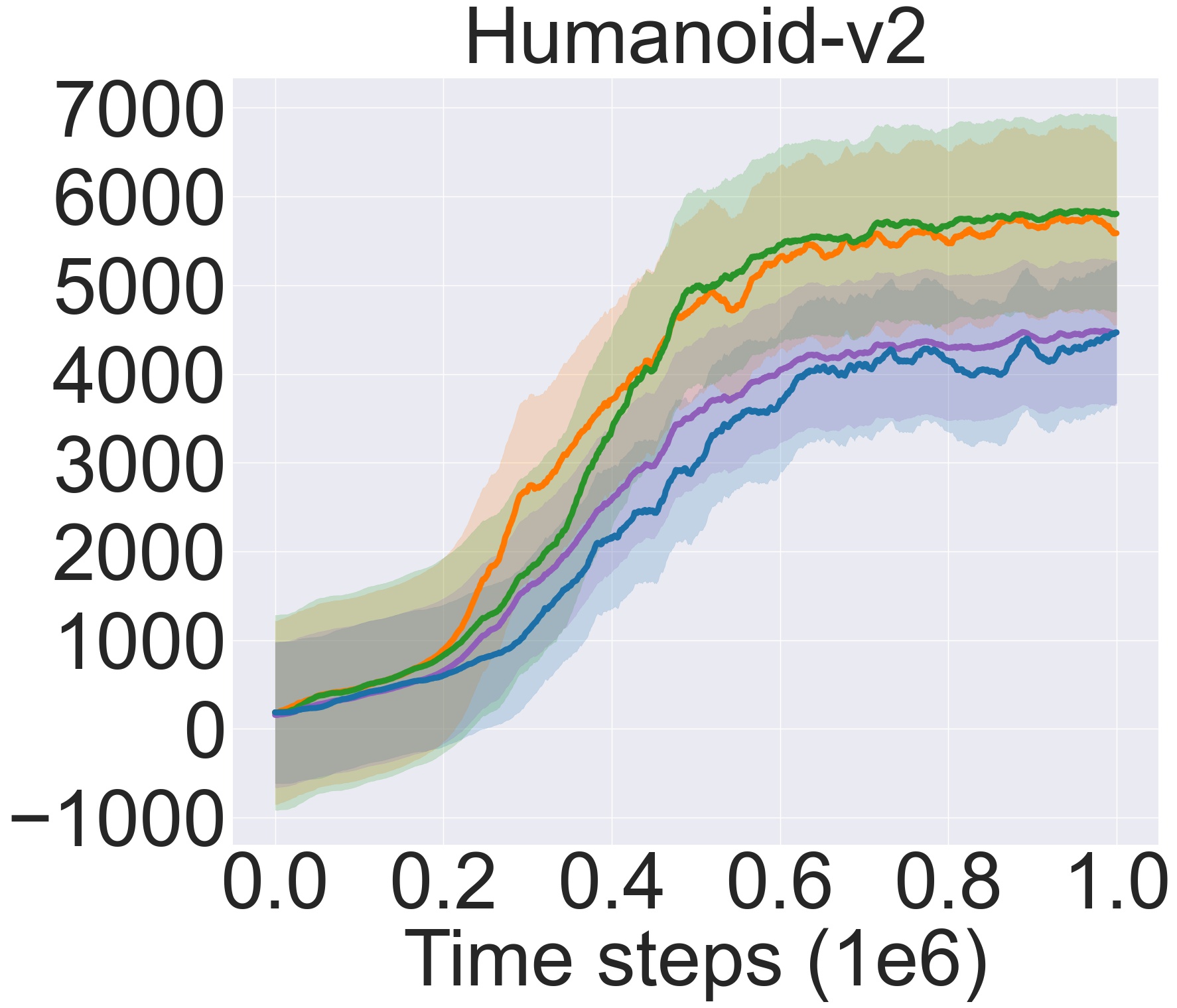}
		\includegraphics[width=1.29in, keepaspectratio]{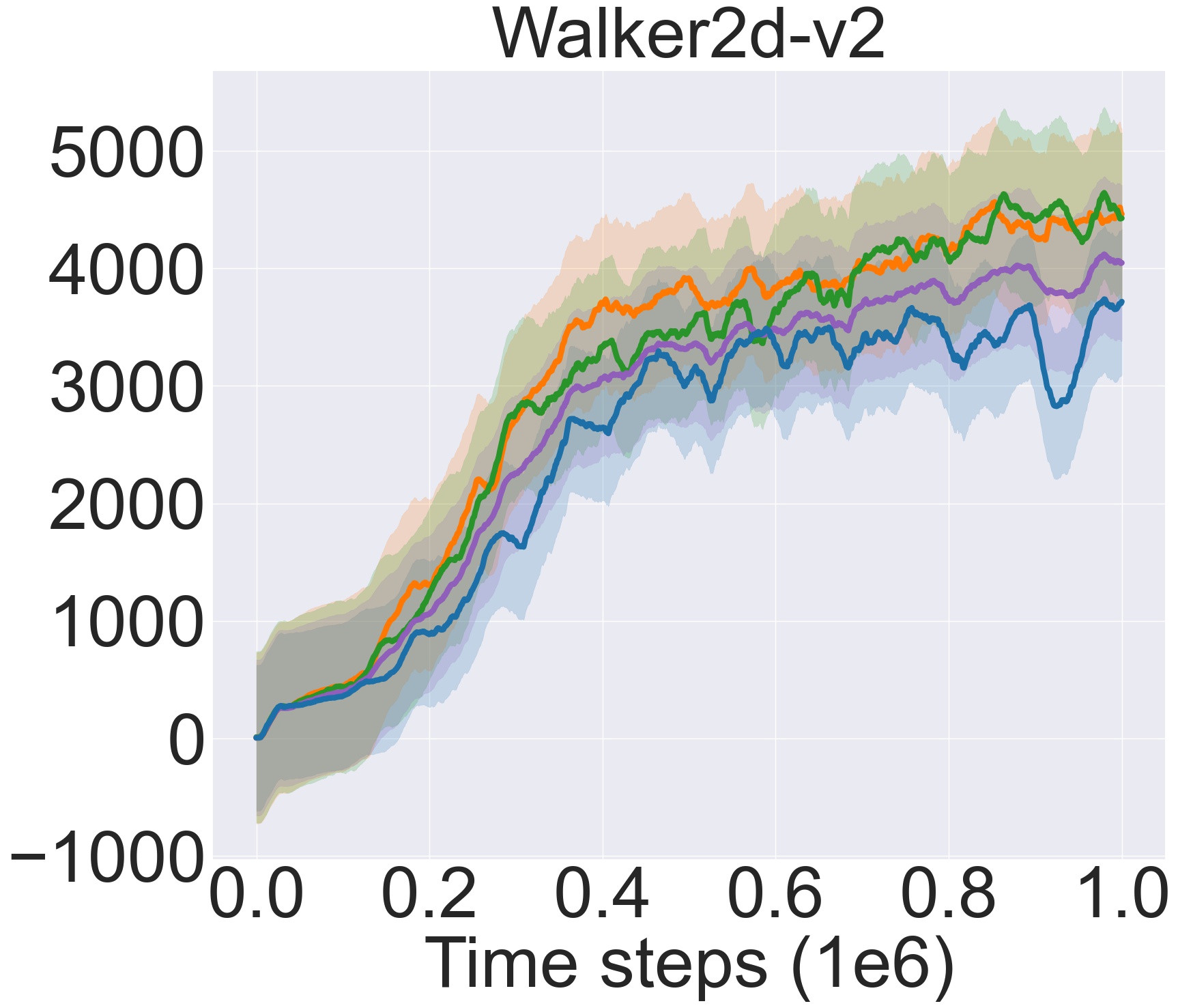}
	} \\ 
	\subfloat[\textbf{Replay Size:} 1,000,000]{
	\includegraphics[width=1.29in, keepaspectratio]{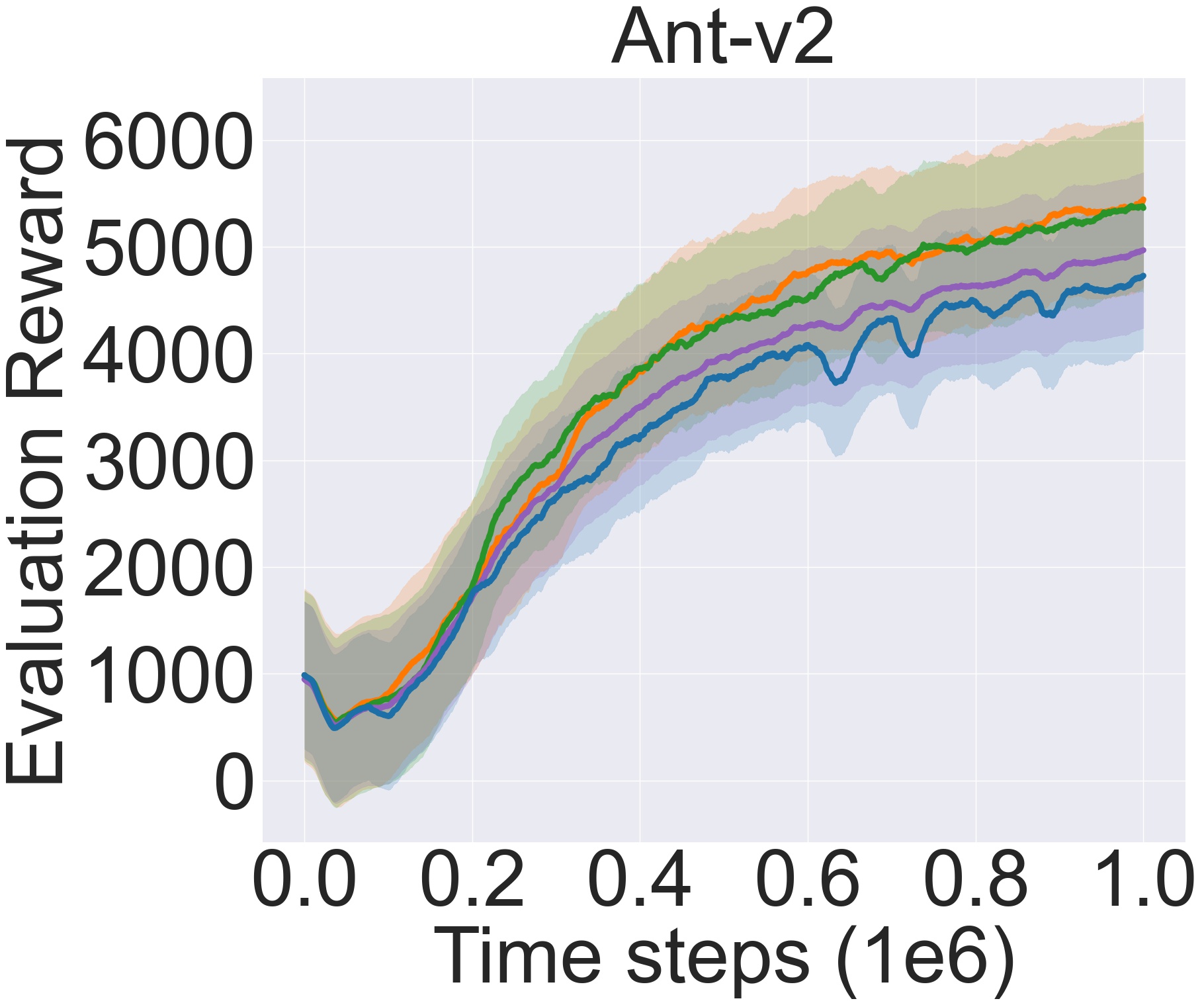}
	\includegraphics[width=1.29in, keepaspectratio]{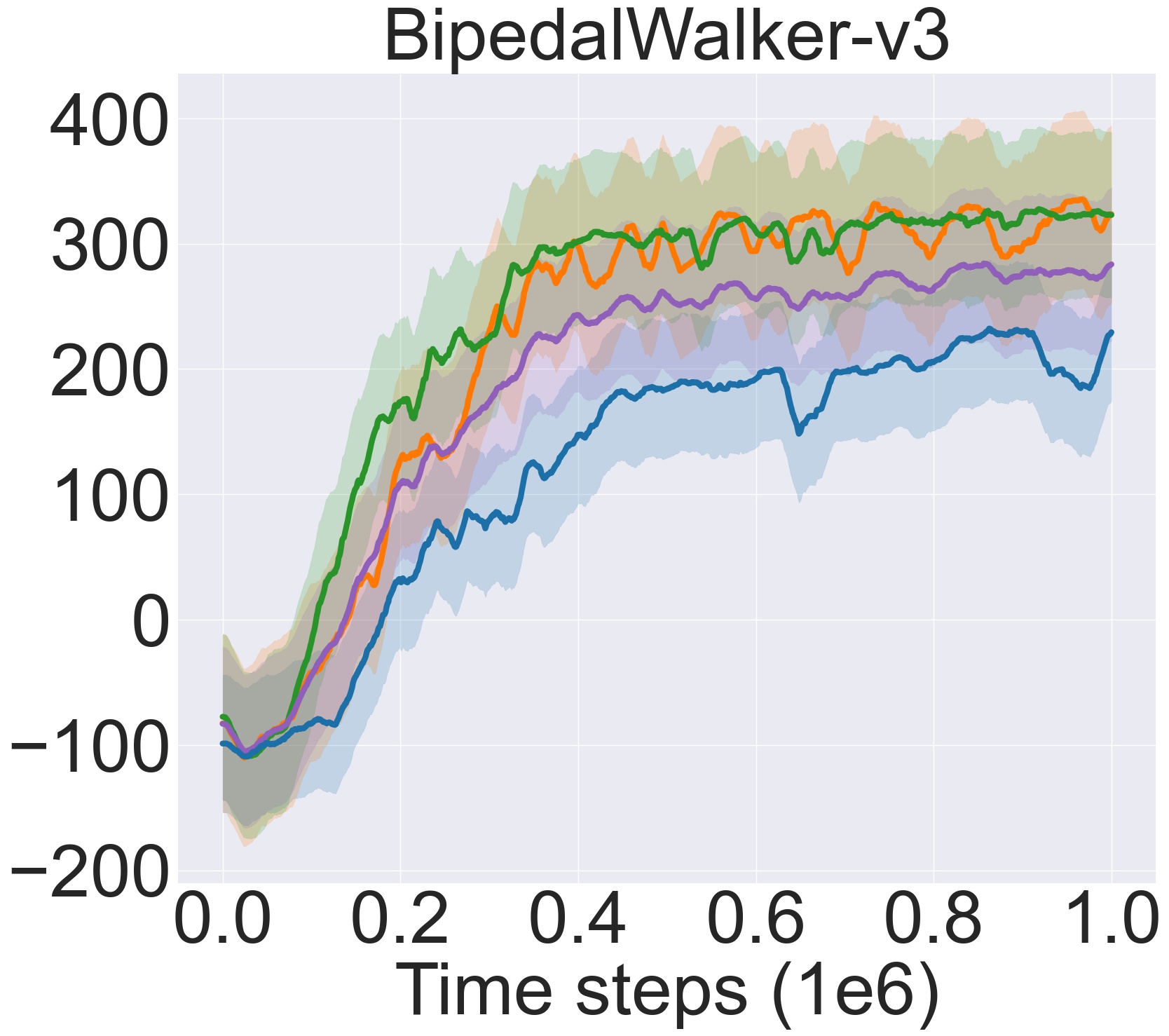}
	\includegraphics[width=1.29in, keepaspectratio]{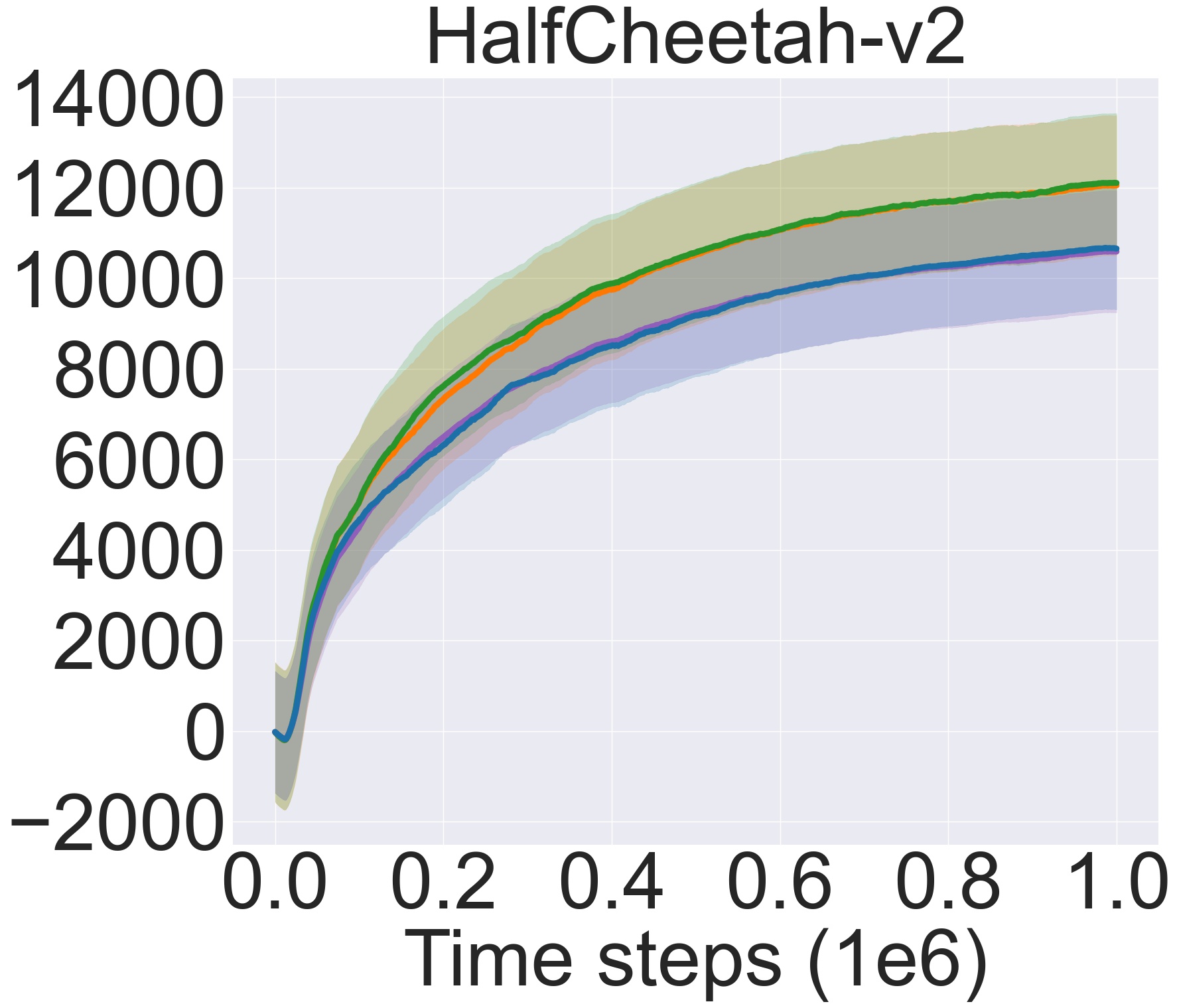}
	\includegraphics[width=1.29in, keepaspectratio]{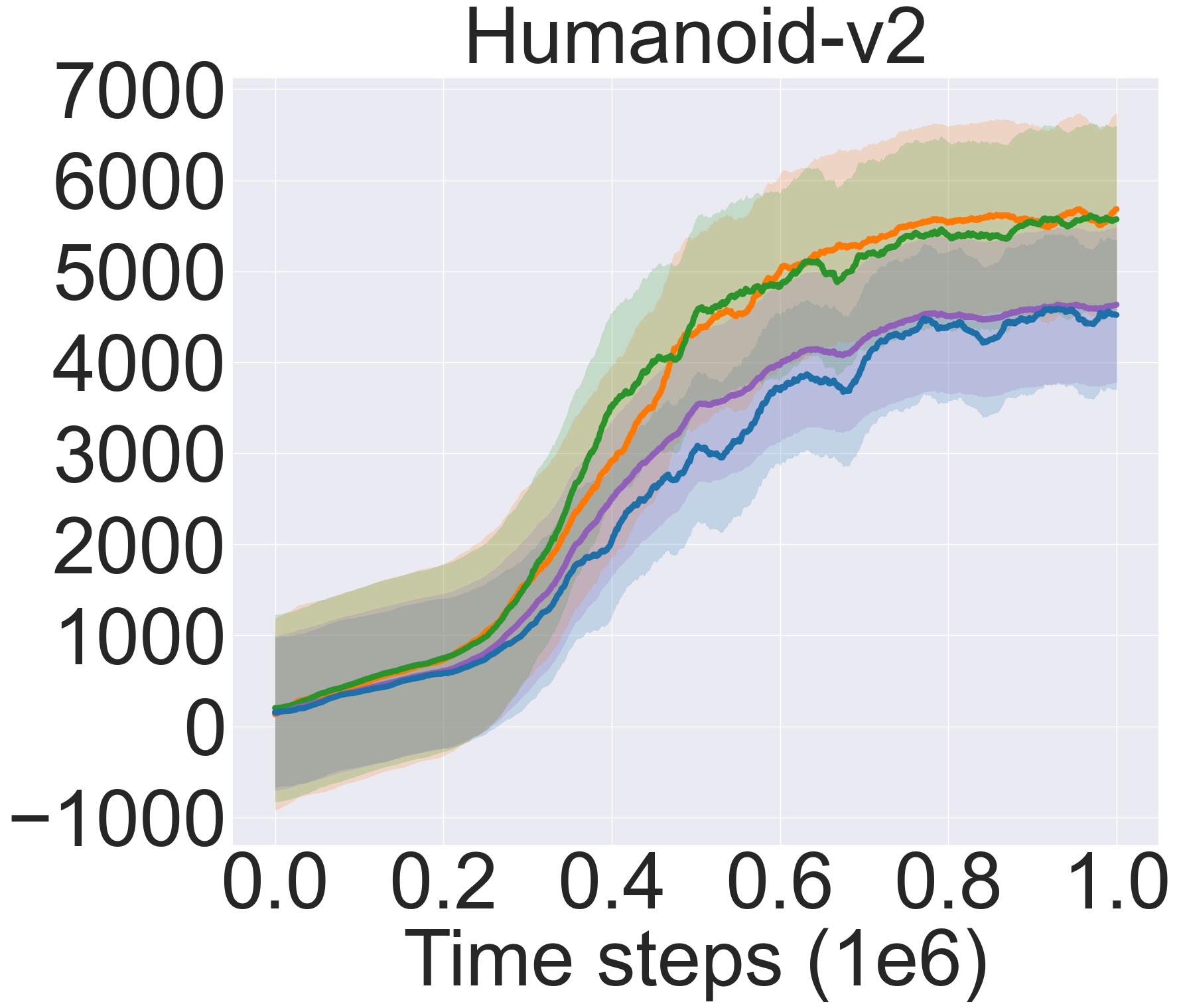}
	\includegraphics[width=1.29in, keepaspectratio]{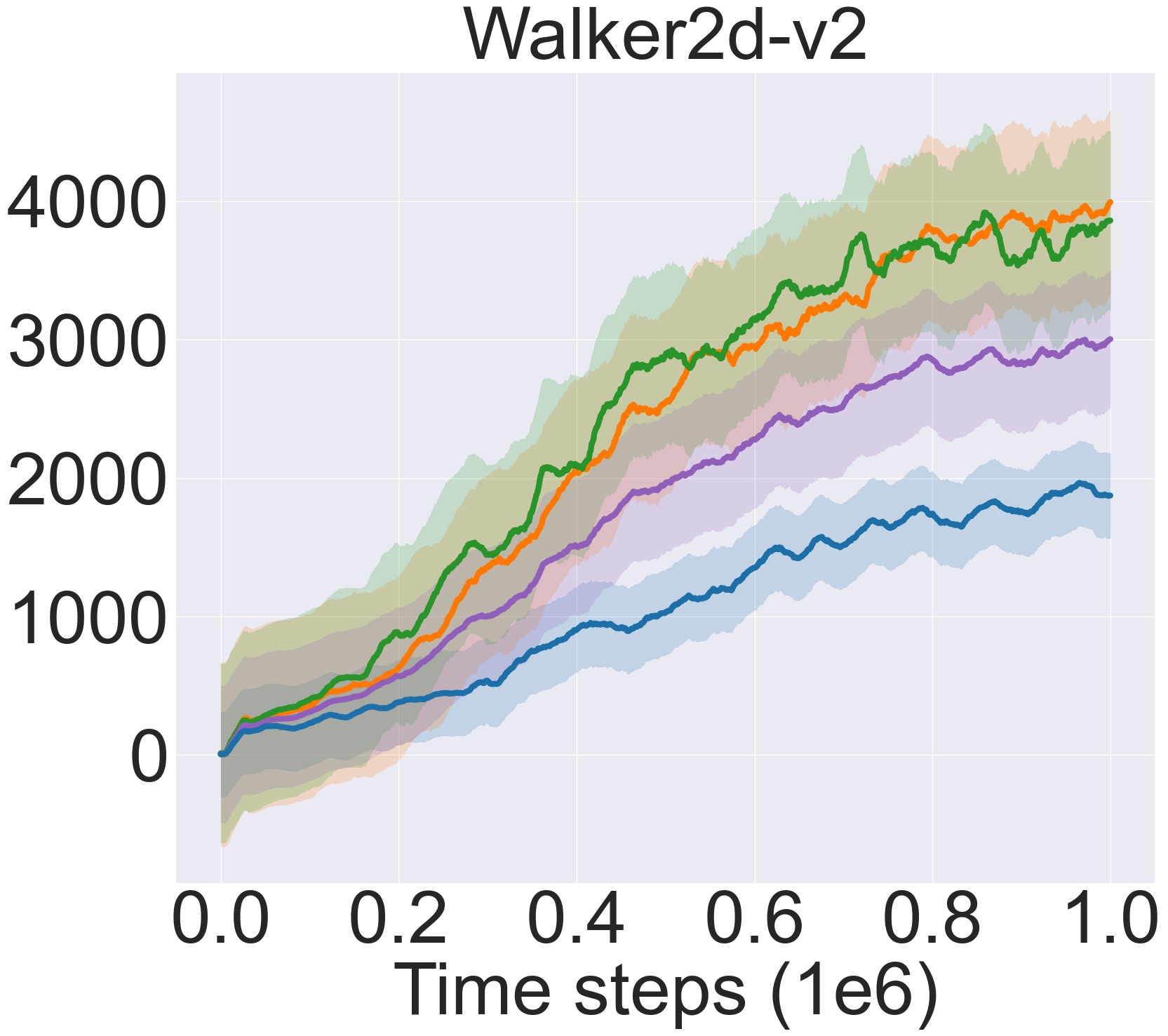}
	}
	\caption{Learning curves for the set of OpenAI Gym continuous control tasks when replay size is 1 million and 100,000. The shaded region represents half a standard deviation of the average evaluation return over 10 random seeds. A sliding window smoothes curves for visual clarity.}
	\label{fig:eval_results}
\end{figure*}
This simple architecture can accelerate learners using different GPUs and agents to be distributed across many machines for the cases in which a single learner fails to reach optimal returns, e.g., limited memory for the replay buffer. DASE can also form ensembles of deterministic actor-critic methods and sampling algorithms to solve challenging continuous control tasks by safely sharing information. However, due to the asynchronous nature of the architecture, some agents may be several updates ahead of the rest, also known as \textit{policy lag}~\cite{impala}. Nonetheless, such a policy lag is corrected by DPS by maintaining only the external policies closer to the distribution under the current policy. 

We perform extensive theoretical analysis on our approach. Notably, Theorem \ref{thm:case_conv} provides a convergence guarantee for Q-learning \cite{q_learning} under DASE and Corollary \ref{cor:safe_experience_sharing} proves that DPS produces accurate importance weights such that a safe experience sharing can be achieved. All proofs are in Appendix \ref{sec:theorems} and the pseudocode for our hyper-parameter-free algorithm is given in Algorithm \ref{alg:DASE} through learner threads.

\begin{theorem}
\label{thm:case_conv}
Under the Robbins-Monro stochastic convergence conditions on the learning rate $\eta$ and standard sampling requirements from the environment, Q-learning with the DASE architecture converges to the optimal value function $Q^{*}$.
\end{theorem}

\begin{corollary}
\label{cor:safe_experience_sharing}
$\xi(s, a) \in [0, \gamma]$ is a contraction coefficient based on $(s, a)$ where $\xi(s, a) = \gamma$ if $\lambda = 0$, i.e., when there is no similarity, and close to zero when the behavioral policies corresponding to the sampled batch match the current policy.
\end{corollary}

\section{Experiments}

\subsection{Experimental Details}
We conduct experiments to evaluate the effectiveness of DASE on OpenAI Gym~\cite{gym} continuous control benchmarks. We apply our method to the state-of-the-art off-policy actor-critic algorithm, Twin Delayed Deep Deterministic Policy Gradient (TD3)~\cite{td3}. Moreover, our method is compared with a single TD3~\cite{td3} agent and the Dual Policy Distillation (DPD) algorithm~\cite{dual_policy_dist}, a student-student framework in which two learners operate in the same environment to investigate diverse viewpoints and extract knowledge from one another to help them learn more effectively, similar to our work. A complete list of hyper-parameters and experimental details are provided in Appendix \ref{sec:exp_details}. In addition to TD3~\cite{td3}, Appendix \ref{sec:additional_results} presents the results for DDPG~\cite{ddpg} and SAC (with deterministic actor)~\cite{sac}, and additional continuous control tasks.

We consider two settings of the experience replay buffer~\cite{experience_replay}: a strictly limited (of size 100,000 transitions) and unlimited. For a fair evaluation with DPD~\cite{dual_policy_dist} which utilizes two agents that simultaneously explore the environment, we run DASE with two agents, i.e., $K = 2$. Figure \ref{fig:eval_results} depicts the experimental results under the two settings of the replay memory. Note that the curves for DPD~\cite{dual_policy_dist} are the average of its two agents, while we depict both agents of DASE when $K = 2$ for further discussion. We discuss the computational complexity introduced by DASE in Appendix \ref{sec:comp_complexity}. Moreover, a comprehensive set of ablation studies is provided to analyze the effect of each DASE component in Appendix \ref{sec:ab_studies}. 

\subsection{Discussion}
From our comparative evaluations, we infer notable results. First, DASE substantially improves the TD3 algorithm~\cite{td3} and outperforms DPD~\cite{dual_policy_dist} in all of the tasks tested. As expected, both agents perform similar behavior since there is no component in our algorithm that discriminates against the agents. Although a limited replay buffer does not always correspond to worse performance, a performance difference between the considered buffer settings always exists, e.g., in the BipedalWalker and Walker2d environments. This may be due to the environment dynamics, that is, some environments can be optimally learned only by the most recent collected transitions, which is explained by the fact that on-policy methods usually outperform off-policy methods in these environments~\cite{deep_rl_that_matters}. Nevertheless, in such cases, our method is almost invariant to the replay buffer size, having a robust performance due to a diverse exploration and its safe experience sharing approach. Lastly, the DPD algorithm~\cite{dual_policy_dist} exhibits a suboptimal behavior in the majority of the tasks, which we believe is caused by the decreased convergence rate due to the additional trajectory generation that introduces substantial computational overhead. 

\section{Conclusion}
This paper introduces a novel continuous off-policy actor-critic architecture that employs multiple explorer agents and a shared experience replay buffer to obtain robust parallel learning when the allocated memory for the collected transitions is limited. Through a safe experience sharing among concurrent agents, it can overcome extrapolation error~\cite{off_policy} by a novel off-policy correction method. Experiments show that the introduced method can achieve state-of-the-art results while baseline algorithms fail to converge under a very limited replay memory condition. Moreover, it can also generalize to cases in which the replay buffer is unlimited, where the state-of-the-art is improved significantly. In practical applications of off-policy deep reinforcement learning where action spaces are large and continuous, we believe DASE will be an effective foothold for future approaches in attaining data efficiency and distributing agents.

\bibliography{references}
\bibliographystyle{icml2022}

%%%%%%%%%%%%%%%%%%%%%%%%%%%%%%%%%%%%%%%%%%%%%%%%%%%%%%%%%%%%%%%%%%%%%%%%%%%%%%%
%%%%%%%%%%%%%%%%%%%%%%%%%%%%%%%%%%%%%%%%%%%%%%%%%%%%%%%%%%%%%%%%%%%%%%%%%%%%%%%
% APPENDIX
%%%%%%%%%%%%%%%%%%%%%%%%%%%%%%%%%%%%%%%%%%%%%%%%%%%%%%%%%%%%%%%%%%%%%%%%%%%%%%%
%%%%%%%%%%%%%%%%%%%%%%%%%%%%%%%%%%%%%%%%%%%%%%%%%%%%%%%%%%%%%%%%%%%%%%%%%%%%%%%
\newpage
\appendix
\onecolumn
\section{Pseudocode}
\label{sec:pseudocode}

\begin{algorithm}[!htb]
\caption{Deterministic Policy Similarity (DPS)}
\label{alg:js_importance_weighting}
\begin{algorithmic}
\STATE {\bfseries Input:} $\pi_{\phi}, \mathcal{B}$\\
\STATE {\bfseries Output:} $\boldsymbol{\lambda}^{|\mathcal{B}| \times 1}$    
\STATE Obtain the external transitions: $(\boldsymbol{S}_{E}^{|\mathcal{B}| \times m}, \boldsymbol{A}_{E}^{|\mathcal{B}| \times n}, \boldsymbol{R}_{E}^{|\mathcal{B}| \times 1}, \boldsymbol{S}_{E}^{'|\mathcal{B}| \times m}) \sim \mathcal{B}_{E}$
\STATE Compute the current action decisions: $\boldsymbol{\hat{A}}_{E}^{|\mathcal{B}_{E}| \times n} = \pi_{\phi}(\boldsymbol{S}_{E}^{|\mathcal{B}_{E}| \times m})$
\STATE Obtain the action difference batch: $\boldsymbol{\dot{A}}^{|\mathcal{B}_{E}| \times n} \coloneqq \boldsymbol{A}_{E}^{|\mathcal{B}_{E}| \times n} - \boldsymbol{\hat{A}}_{E}^{|\mathcal{B}_{E}| \times n}$
\STATE Compute the mean of the the multivariate Gaussian: $\boldsymbol{\dot{\mu}}^{n \times 1} = \frac{1}{|\mathcal{B}_{E}|}\smashoperator{\sum_{i = 1}^{|\mathcal{B}_{E}|}}(\boldsymbol{\dot{A}}^{|\mathcal{B}_{E}| \times n}_{i})^\top$
\STATE Compute the covariance matrix of the multivariate Gaussian: $\boldsymbol{\dot{\Sigma}}^{n \times n} = \frac{1}{|\mathcal{B}_{E}| - 1}\smashoperator{\sum_{i = 1}^{|\mathcal{B}_{E}|}}\boldsymbol{a}^{n \times 1}_{i}(\boldsymbol{a}^{n \times 1}_{i})^\top$
\STATE Compute the dissimilarity metric: $\rho = \mathrm{JSD}(\mathcal{N}(\boldsymbol{\dot{\mu}}^{n \times 1}, \boldsymbol{\dot{\Sigma}}^{n \times n}) \parallel \mathcal{N}(\boldsymbol{0}^{n \times 1}, \sigma\boldsymbol{I}^{n \times n}))$
\STATE Convert the dissimilarity to the similarity to construct the DPS weights: $\boldsymbol{\lambda}^{|\mathcal{B|}_{E} \times 1} = [e^{-\rho},\ e^{-\rho},\ \dots,\ e^{-\rho}]^\top$
\STATE {\bfseries return} $\boldsymbol{\lambda}^{|\mathcal{B}| \times 1}$
\end{algorithmic}
\end{algorithm}

\begin{algorithm}[!htb]
\caption{Deterministic Actor-Critic with Shared Experience (DASE)}
\label{alg:DASE}
\begin{algorithmic}
\STATE Initialize $K$ agents with actor $\pi_{\phi_{i}}$ and critic $Q_{\theta_{i}}$ networks with parameters $\phi_{i}$ and $\theta_{i}$ for $i = 1, \dots, K$
\STATE Initialize target networks if required
\STATE Initialize global experience replay buffer $\mathcal{R}$
\FOR{each learner thread $i = 1, \dots, K$}
\FOR{each exploration time step}
\STATE Obtain transition tuple $\tau$
\STATE Store transition tuple $\tau$ in $\mathcal{R}$
\ENDFOR
\FOR{each training iteration}
\STATE Sample a batch of transitions $\mathcal{B}$ from $\mathcal{R}$
\STATE Obtain the DPS weights: $\boldsymbol{\lambda}^{|\mathcal{B}| \times 1}$ = \textbf{DPS($\pi_{\phi_{i}}$, $\mathcal{B}$)}
\STATE Weigh the external transitions by $\boldsymbol{\lambda}^{|\mathcal{B}| \times 1}$ \STATE Update $\phi_{i}$ and $\theta_{i}$ by both internal and weighted external transitions
\STATE Update target networks if required
\ENDFOR
\ENDFOR
\end{algorithmic}
\end{algorithm}

\section{Missing Proofs}
\label{sec:theorems}
\subsection{Convergence Guarantee}
\begin{lemma}
\label{lemma:q_learning_convergence_stochastic_version}
Consider a stochastic process ($\xi_{t}$, $\Delta_{t}$, $F_{t}$), $t \geq 0$ where $\xi_{t}, \delta_{t}, F_{t}: X \rightarrow \mathbb{R}$, satisfies the equations:
\begin{equation}
    \Delta_{t + 1}(x_{t}) = (1 - \xi_{t}(x_{t}))\Delta_{t}(x_{t}) + \xi_{t}(x_{t})F_{t}(x_{t}); \quad x_{t} \in X, t = 0, 1, 2, \dots.
\end{equation}
Let $\mathcal{P}_{t}$ be a sequence of increasing $\sigma$-fields such that $\eta_{0}$ and $\Delta_{0}$ are $P_{0}$-measurable and $\xi_{t}$, $\Delta_{t}$ and $F_{t - 1}$ are $\mathcal{P}_{t}$-measurable. For $t = 1, 2, \dots$, assume that the following conditions hold:
\begin{enumerate}
    \item The set $X$ is finite.
    \item $\xi_{t}(x_{t}) \in [0, 1]$, $\sum_{t}\xi_{t}(x_{t}) = \infty$, $\sum_{t}(\xi_{t})^{2} < \infty$ with probability 1 and $\forall x \neq x_{t}: \xi(x) = 0$.
    \item $||\mathbb{E}[F_{t}|\mathcal{P}_{t}]|| \leq \kappa||\Delta_{t}|| + c_{t}$ where $\kappa \in [0, 1)$ and $c_{t}$ converges to 0 with probability 1.
    \item $\mathrm{Var}(F_{t}(x_{t})|\mathcal{P}_{t}) \leq K(1 + \kappa||\Delta_{t}||)^{2}$ where $K$ is a constant.
\end{enumerate}
Where $||\cdot||$ denotes the maximum norm. Then, $\Delta_{t}$ converges to 0 with probability 1.
\end{lemma}
\begin{proof}
See~\cite{q_learning,singh_2000,melo_2001}.
\end{proof}

\begin{lemma}
\label{lemma:q_learning_convergence}
Given a finite MDP $(\mathcal{S}, \mathcal{A}, p, r)$ and the transition tuple $(s_{t}, a_{t}, r_{t}, s_{t + 1})$ at time step $t$, the Q-learning algorithm given by the update rule:
\begin{equation}
\label{eq:q_learning_update_rule}
    Q_{t + 1}(s_{t}, a_{t}) = Q_{t}(s_{t}, a_{t}) + \eta_{t}[r_{t} + \gamma\underset{a' \in \mathcal{A}}{\mathrm{max}}Q_{t}(s_{t + 1}, a') - Q_{t}(s_{t}, a_{t})],
\end{equation}
converges to the optimal Q-function denoted by $Q^{*}$ with probability 1 if
\begin{equation}
    \sum_{t}\eta_{t} = \infty, \quad \sum_{t}\eta^{2}_{t} < \infty; \quad \forall (s, a) \in \mathcal{S} \times \mathcal{A}.
\end{equation}

\end{lemma}

\begin{proof}
The proof largely relies on Lemma \ref{lemma:q_learning_convergence_stochastic_version}~\cite{singh_2000}. First, Condition 1 in Lemma  \ref{lemma:q_learning_convergence_stochastic_version} is satisfied by the finite MDP by setting $X = \mathcal{S} \times \mathcal{A}$. The assumption of Robbins-Monro stochastic convergence conditions on the learning rate $\eta_{t}$ satisfies Condition 2 by setting $\xi_{t} = \eta_{t}$. Then, let:
\begin{equation}
\label{eq:f_t_definition}
    F_{t}(s_{t}, a_{t}) = r_{t} + \gamma\underset{a' \in \mathcal{A}}{\mathrm{max}}Q(s_{t + 1}, a') - Q^{*}(s_{t}, a_{t}),
\end{equation}
\begin{equation}
\label{eq:delta_definition}
    \Delta_{t} = Q_{t}(s_{t}, a_{t}) - Q^{*}(s_{t}, a_{t}),
\end{equation}
\begin{equation}
\label{eq:p_t_definition}
    \mathcal{P}_{t} = \{Q_{0}, s_{0}, a_{0}, \eta_{0}, r_{1}, s_{1}, \dots, s_{t}, a_{t}\}.
\end{equation}
If state-action visitation and updates are performed infinitely often, and $\gamma < 1$, then by (\ref{eq:f_t_definition}), (\ref{eq:delta_definition}) and (\ref{eq:p_t_definition}), Condition 3 is satisfied by the contraction of the Bellman Operator $\mathcal{T}$~\cite{melo_2001}. Finally, Condition 4 follows from a bounded and deterministic reward function $r(s_{t}, a_{t})$~\cite{melo_2001}, i.e., $r_{t} = r(s_{t}, a_{t})$. Then, by Lemma \ref{lemma:q_learning_convergence_stochastic_version}, as $\Delta_{t}$ converges to 0 with probability 1, $Q_{t}$ converges to $Q^{*}$ with probability 1. 

\end{proof}

\begin{theorem}
Under the Robbins-Monro stochastic convergence conditions on the learning rate $\eta$ and standard sampling requirements from the environment, Q-learning with the DASE architecture converges to the optimal value function $Q^{*}$.
\end{theorem}
\begin{proof}
Follows from the proof of Lemma \ref{lemma:q_learning_convergence}. If the sequences of increasing $\sigma$-fields are split into the fields correponding to the internal and external sequences, convergence of Q-learning with internal transitions are already given by Lemma \ref{lemma:q_learning_convergence}. For external transitions, Conditions 3 and 4 are altered due to DPS weights $\lambda$, by (\ref{eq:f_t_definition}), (\ref{eq:delta_definition}) and (\ref{eq:p_t_definition}), we have:
\begin{equation}
\label{eq:CASE_third_condition}
    \lambda||\mathbb{E}[F_{t} | \mathcal{P}_{t}]|| \leq \kappa\lambda||\Delta_{t}|| + c_{t},
\end{equation}
\begin{equation}
\label{eq:CASE_fourth_condition}
    \lambda^{2}\mathrm{Var}(F_{t}(x_{t}) | \mathcal{P}_{t}) \leq K(1 + \kappa\lambda||\Delta_{t}||)^{2}.
\end{equation}
As $\lambda \in [0, 1]$, clearly (\ref{eq:CASE_third_condition}) is satisfied. Moreover, since $\lambda^{2}\mathrm{Var}(F_{t}(x_{t}) | \mathcal{P}_{t}) \leq \lambda^{2}K(1 + \kappa||\Delta_{t}||)^{2}$ and $\lambda^{2}K(1 + \kappa||\Delta_{t}||)^{2} \leq K(1 + \kappa\lambda||\Delta_{t}||)^{2}$, we also have (\ref{eq:CASE_fourth_condition}) satisfied. Furthermore, the internal and external experiences are the samples of the same finite MDP $X$, and Robbins-Monro stochastic convergence conditions also apply on the external sequences which yield Conditions 1 and 2 to be satisfied. Therefore, Q-learning with the DASE architecture converges to the optimal Q-function under the requirements of infinitely many state-action visitation and updates, and $\gamma < 1$. 
\end{proof}

\subsection{Safe Experience Sharing}
\begin{definition}
\label{def:exp_op_definition}
    The general expectation operator for one-step importance sampling in return-based off-policy algorithms is defined by:
    \begin{equation}
        \label{eq:def_eq}
            \mathcal{H}Q(s, a) \coloneqq Q(s, a) + \mathbb{E}_{\eta}[r + \gamma \mathbb{E}_{\pi}Q(s', \cdot) - Q(s, a)],
        \end{equation}
    for some non-negative one-step importance sampling coefficient $\lambda$, and any behavioral policy $\eta$, where we write $\mathbb{E}_{\pi}Q(s, \cdot) \coloneqq \sum_{a}\pi(a | s)Q(s, a)$.
\end{definition}

\begin{lemma}
\label{lem:difference}
    The difference between $\mathcal{H}Q$ and its fixed point $Q^{\pi}$ is expressed by:
        \begin{equation}
            \mathcal{H}Q(s, a) - Q^{\pi}(s, a) = \mathbb{E}_{\eta}[\gamma(\mathbb{E}_{\pi}[(Q - Q^{\pi})(s, \cdot)] - \lambda(Q - Q^{\pi})(s, a))]
        \end{equation}
\end{lemma}
\begin{proof}
   Follows from the proof of Lemma 1 in \cite{munos_safe}. First, let $\Delta Q \coloneqq Q - Q^{\pi}$. Then, by rewriting Eq. (\ref{eq:def_eq}):
   \begin{equation}
       \mathcal{H}Q(s, a) = \mathbb{E}_{\eta}[r + \gamma(\mathbb{E}_{\pi}Q(s', \cdot) - \lambda'Q(s', a'))],
   \end{equation}
   where $\lambda'$ is the coefficient of the next transition. As $Q^{\pi}$ is the fixed point of $\mathcal{H}$, we have:
   \begin{equation}
           Q^{\pi}(s, a) = \mathcal{H}Q^{\pi}(s, a)
           = \mathbb{E}_{\eta}[r + \gamma(\mathbb{E}_{\pi}Q^{\pi}(s', \cdot) - \lambda'Q^{\pi}(s', a'))],
   \end{equation}
   from which we infer that:
   \begin{align}
       \begin{split}
           \mathcal{H}Q(s, a) - Q^{\pi}(s, a) &= \mathbb{E}_{\eta}[\gamma(\mathbb{E}_{\pi}\Delta Q(s', \cdot) - \lambda'\Delta Q(s', a'))], \\
           &= \gamma \mathbb{E}_{\eta}[\mathbb{E}_{\pi}\Delta Q(s, \cdot) - \lambda \Delta Q(s, a)], \\
           &= \mathbb{E}_{\eta}[\gamma\left(\mathbb{E}_{\pi}\Delta Q(s, \cdot) - \lambda \Delta Q(s, a)\right)].
       \end{split}
   \end{align}
    
\end{proof}

\begin{theorem}
\label{thm:safe_off_policy}
    The operator $\mathcal{H}$ defined by Definition \ref{def:exp_op_definition} has a unique fixed point $Q^{\pi}$. Moreover, if for each action selected by the policy $a \in \mathcal{A}$ and sampled batch of transitions $\mathcal{B}$, we have $\lambda = \lambda(a, \mathcal{B}) \in [0, e^{-\rho}]$. Then for any Q-function $Q$, we have:
    \begin{equation}
        || \mathcal{H}Q - Q^{\pi} || \leq \gamma ||Q - Q^{\pi} ||,
    \end{equation}
    under the current policy $\pi$.
\end{theorem}
\begin{proof}
\label{thm:case_safe}
    Follows from the adaptation of proof of Theorem 1 in \cite{munos_safe} to one-step importance sampling. It is trivial to observe from Definition \ref{def:exp_op_definition} that $Q^{\pi}$ is the fixed point of the operator $\mathcal{H}$ since:
    \begin{equation}
        \mathbb{E}_{s' \sim P(\cdot|s, a)}[r + \gamma\mathbb{E}_{\pi}Q^{\pi}(s', \cdot) - Q^{\pi}(s, a)] = (\mathcal{T}^{\pi}Q^{\pi} - Q^{\pi})(s, a) = 0,
    \end{equation}
    as $Q^{\pi}$ is the fixed point of $\mathcal{T}^{\pi}$. Let $\Delta Q \coloneqq Q - Q^{\pi}$, and from Lemma \ref{lem:difference}, we have:
    \begin{align}
        \mathcal{H}Q(s, a) - Q^{\pi}(s, a) &= \mathbb{E}_{\eta}[\gamma\left(\mathbb{E}_{\pi}\Delta Q(s, \cdot) - \lambda \Delta Q(s, a)\right)], \\
       &= \gamma \mathbb{E}_{\eta}[\mathbb{E}_{\pi}\Delta Q(s, \cdot) - \lambda\Delta Q(s, a)], \\
       &=\gamma\mathbb{E}_{\eta}[\mathbb{E}_{\pi}\Delta Q(s, \cdot) - \mathbb{E}_{a}[\lambda(a, \mathcal{B})\Delta Q(s, a)|\mathcal{B}]], \\
       &=\gamma\mathbb{E}_{\eta}[\sum_{b}(\pi(b|s) - \eta(b|s)\lambda(b|\mathcal{B}))\Delta Q(s, b)]. 
    \end{align}
    Now, since $\lambda \in [0, 1]$, we have:
    \begin{equation}
        \mathcal{H}Q(s, a) - Q^{\pi}(s, a) = \sum_{y, b}w_{y, b}\Delta Q(y, b),
    \end{equation}
    which is a linear combination of $\Delta Q(y, b)$ weighted by:
    \begin{equation}
        w_{y, b} \coloneqq \gamma \mathbb{E}_{\eta}[(\pi(b | s) - \eta(b|s)\lambda(b|\mathcal{B}))\mathbb{I}\{s = y\}],
    \end{equation}
    where $\mathbb{I}(\cdot)$ is the indicator function. The sum of those coefficients over $y$ and $b$ is:
    \begin{align}
    \sum_{y, b}\omega_{y, b} &= \gamma \mathbb{E}_{\eta}[\sum_{b}(\pi(b|s) - \eta(b|s)\lambda(b, \mathcal{B}))], \\
    &= \gamma \mathbb{E}_{\eta}[\mathbb{E}_{a}[1 - \lambda(a, \mathcal{B})|\mathcal{B}]] \\
    &= \gamma \mathbb{E}_{\eta}[1 - \lambda] \\
    &= \gamma - \gamma \Lambda,
    \end{align}
    where $\Lambda = \mathbb{E}_{\eta}[\lambda]$. As $0 \leq \Lambda \leq 1$, we have $\sum_{y, b}\omega_{y, b} \leq \gamma$. Therefore, $\mathcal{H}Q(s, a) - Q^{\pi}(s, a)$ is a sub-convex combination of $\Delta Q(y, b)$ weighted by non-negative coefficients $\omega_{y, b}$ which sum to at most $\gamma$. Hence, $\mathcal{H}$ is a $\gamma$-contraction mapping around $Q^{\pi}$.
\end{proof}

\begin{corollary}
\label{cor:safe_off_policy}
    In the proof of Theorem \ref{thm:safe_off_policy}, notice that the term $\gamma\mathbb{E}_{\eta}[\lambda]$ depends on $(s, a)$. Let:
    \begin{equation}
        \xi(s, a) \coloneqq \gamma - \gamma\mathbb{E}_{\eta}[\lambda].
    \end{equation}
    Then, we have:
    \begin{equation}
        |\mathcal{H}Q(s, a) - Q^{\pi}(s, a)| \leq \xi(s, a)||Q - Q^{\pi}||. 
    \end{equation}
    Thus, $\xi(s, a) \in [0, \gamma]$ is a contraction coefficient based on $(s, a)$ where $\xi(s, a) = \gamma$ if $\lambda = 0$, i.e., when there is no similarity, and close to zero when the behavioral policies corresponding to the sampled batch match the current policy.
\end{corollary}

\section{Experimental Details}
\label{sec:exp_details}

All networks are trained with PyTorch (version 1.8.1)~\cite{pytorch}, using default values for all unmentioned hyper-parameters. 

\subsection{Environment}

Performances of all methods are evaluated in MuJoCo (mujoco-py version 1.50)~\cite{mujoco}, and Box2D (version 2.3.10)~\cite{box2d} physics engines interfaced by OpenAI Gym (version 0.17.3)~\cite{gym}, using v3 environment for BipedalWalker and v2 for rest of the environments. The environment dynamics, state and action spaces, and reward functions are not pre-processed and modified for easy reproducibility and fair evaluation procedure with the baseline algorithms. Each environment episode runs for a maximum of 1000 steps until a terminal condition is encountered. The multi-dimensional action space for all environments is within the range (-1, 1) except for Humanoid, which uses the range of (-0.4, 0.4).

\subsection{Experimental Setup}

All experiments are run for 1 million time steps with evaluations every 1000 time steps, where an evaluation of an agent records the average reward over 10 episodes without exploration noise and updates. We report the average evaluation return of 10 random seeds for each environment, including the initialization of behavioral policies, simulators, and network parameters. All agents are initialized with different seeds in the DASE architecture to obtain randomness in the explored state-action spaces. Unless stated otherwise, each agent is trained by one training iteration after each time step. Agents are trained by transition tuples $(s, a, r, s')$ uniformly sampled from the shared experience replay~\cite{experience_replay}. 

\subsection{Implementation}

Our implementation of the off-policy actor-critic algorithms, DDPG~\cite{ddpg}, SAC~\cite{sac} and TD3~\cite{td3}, and the baseline algorithm, DPD~\cite{dual_policy_dist}, closely follows the set of hyper-parameters given in the respective papers. For the implementation of TD3~\cite{td3}, we use the author's GitHub repository \footnote{\url{https://github.com/sfujim/TD3}} for the fine-tuned version of the algorithm and the DDPG~\cite{ddpg} implementation. For the implementation of the DPD algorithm~\cite{dual_policy_dist}, we use the author's GitHub repository \footnote{\url{https://github.com/kiminh/dual-policy-distillation}}. We also give the hyper-parameter setting given in~\cite{td3} for the sake of comparison in Table \ref{table:algo_specific_parameters}. The SAC algorithm~\cite{sac} follows the same setup and hyper-parameter settings for the deterministic policy, as given in the paper. We implement the DASE architecture on top of the baseline algorithms separately. The implementation distributes the agents to different threads while each agent can access the shared experience replay~\cite{experience_replay} contained in the RAM.

\subsection{Architecture and Hyper-parameter Setting}

\label{sec:implementation_details}
Different from the paper, we increase the batch size in DDPG algorithm~\cite{ddpg} to 256 in order for agents to sufficiently see other agents' experiences and replace Ornstein–Uhlenbeck exploration noise with a zero-mean Gaussian with a standard deviation of 0.1. SAC~\cite{sac} follows the same hyper-parameter setting given in the paper except for the deterministic policy. As no exploration noise is applied to the stochastic actor of SAC~\cite{sac}, we add Gaussian noise with a standard deviation of 0.1 to the actions selected by the deterministic policy of SAC~\cite{sac}. For DPD~\cite{dual_policy_dist}, we apply the same set of hyper-parameters in DDPG + DPD~\cite{dual_policy_dist} to TD3~\cite{td3} and SAC~\cite{sac}. Shared and algorithm-specific hyper-parameters are given in Table \ref{table:shared_parameters} and \ref{table:algo_specific_parameters}, respectively.

\subsection{Hyper-parameter Optimization}

No hyper-parameter optimization was performed on DDPG~\cite{ddpg} and TD3~\cite{td3}. For SAC~\cite{sac}, reward scale for the LunarLanderContinuous and BipedalWalker environments as they are not presented in the original paper. We tried $5$, $10$, and $20$ for the reward scale. The reward scale value of $5$ is the one that gave the highest average return over the last 10 evaluations over 10 trials of 1 million time steps, for both environments.

\subsection{Evaluation}

Evaluations occur every 1000 steps, where an evaluation is an average reward over 10 episodes without exploration noise and network updates. We utilize a new environment with a fixed seed (the training seed + a constant) for each evaluation to decrease the variation caused by different seeds. Hence, each evaluation uses the same set of initial start states.

\subsection{Visualization}

Learning curves are used to show performance, and they are given as an average of 10 trials with a shaded zone added to reflect a half standard deviation across the trials. The curves are smoothed uniformly over a sliding window of 25 evaluations for visual clarity.

\begin{table}[htb]
\caption{Shared hyper-parameters.\label{table:shared_parameters}}
\vskip 0.15in
\begin{center}
\begin{small}
\begin{sc}
\begin{tabular}{lc}
\toprule
\textbf{Hyper-parameter} & \textbf{Value} \\
\midrule
    Actor Regularization & None \\
    Optimizer & Adam~\cite{adam} \\
    Nonlinearity & ReLU \\
    Discount Factor ($\gamma$) & 0.99 \\
    Gradient Clipping & False \\
    Number of Hidden Layers (All Networks) & 2 \\
\bottomrule
\end{tabular}
\end{sc}
\end{small}
\end{center}
\vskip -0.1in
\end{table}

\begin{table}[htb]
\caption{Algorithm specific hyper-parameters used for the implementation of the baseline algorithms.
\label{table:algo_specific_parameters}}
\vskip 0.15in
\begin{center}
\begin{small}
\begin{sc}
\begin{tabular}{lccc}
\toprule
\textbf{Hyper-parameter} & \textbf{DDPG} & \textbf{SAC} & \textbf{TD3} \\
\midrule
        Critic Learning Rate & $10^{-3}$ & $3 \times 10^{-4}$ & $3 \times 10^{-4}$ \\
        Critic Regularization & $10^{-2} \times ||\theta||^{2}$ & None & None \\
        Actor Learning Rate & $10^{-4}$ & $3 \times 10^{-4}$ & $3 \times 10^{-4}$ \\
        Target Update Rate ($\tau$) & $10^{-3}$ & $5 \times 10^{-3}$ & $5 \times 10^{-3}$ \\
        Batch Size & 256 & 256 & 256 \\
        Updates per optimization step & 1 & 1 & 1 \\
        Critic update interval & 1 & 1 & 1 \\
        Actor update interval & 1 & 1 & 1 \\
        Reward Scaling & 1 & 5 (20 for Humanoid) & 1 \\
        Normalized Observations & True & False & False \\
        Exploration Policy & $\mathcal{N}(0, 0.1)$ & $\mathcal{N}(0, 0.1)$ & $\mathcal{N}(0, 0.1)$ \\
        Start (Exploration) time steps & 25000 & 25000 & 25000 \\
        Number of Hidden Units in the First Layer & 400 & 256 & 256 \\
        Number of Hidden Units in the Second Layer & 300 & 256 & 256 \\
\bottomrule
\end{tabular}
\end{sc}
\end{small}
\end{center}
\vskip -0.1in
\end{table}

\clearpage

\section{Complete Evaluation Results}
\label{sec:additional_results}

% Evaluation Results Unlimited Buffer Condition for the Supplementary Material
\begin{figure*}[!hbt]
    \centering
    \begin{align*}
        &\text{{\blue} DDPG (single agent)} \quad &&\text{{\orange} DDPG + DASE ($1^{\text{st}}$ agent)} \\
        &\text{\text{{\purple} DDPG + DPD (average of two agents)}} \quad &&\text{{\green} DDPG + DASE ($2^{\text{nd}}$ agent)}
    \end{align*}
	\subfloat[\textbf{Replay Size:} 100,000]{
		\includegraphics[width=1.40in, keepaspectratio]{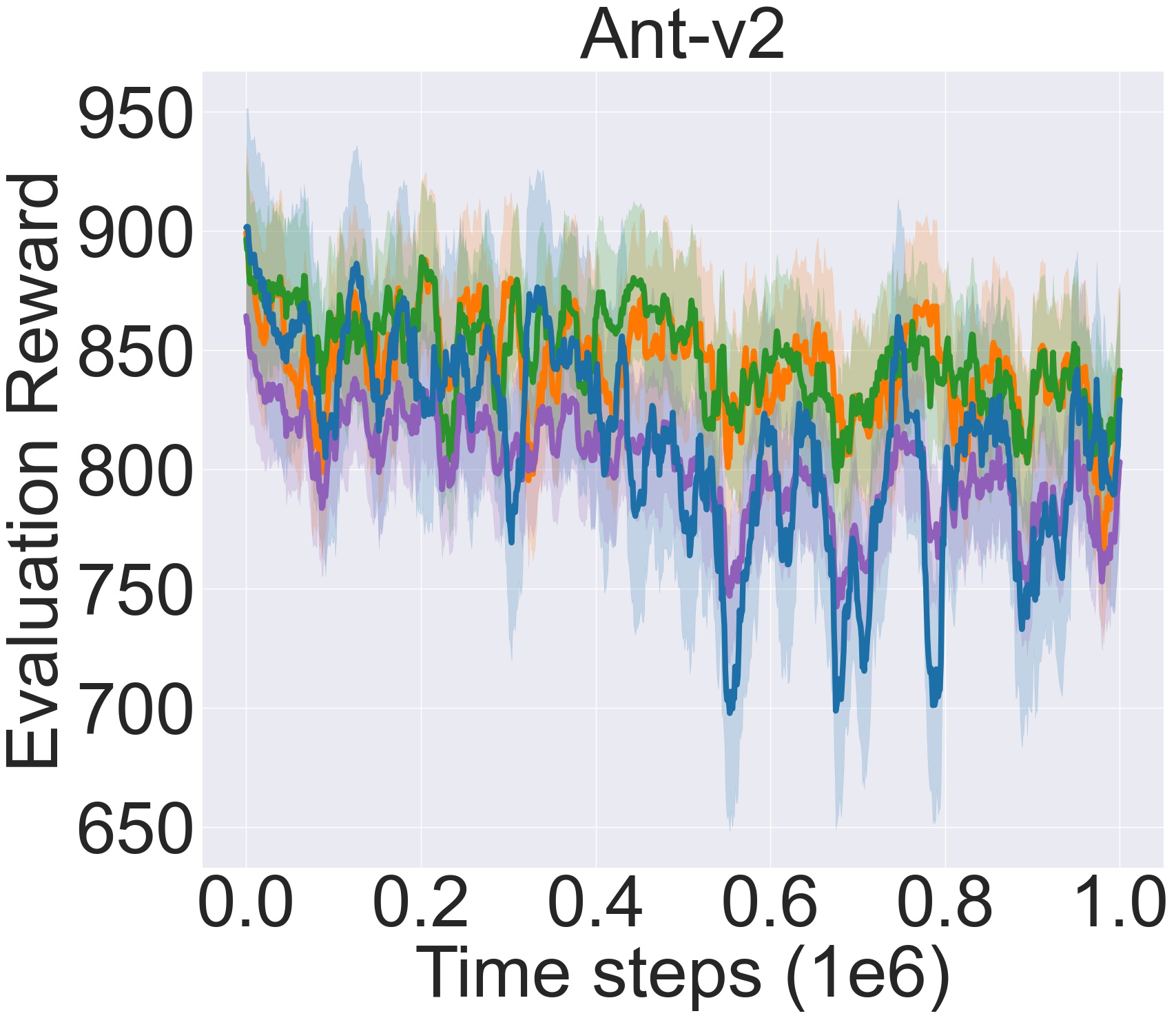}
		\includegraphics[width=1.40in, keepaspectratio]{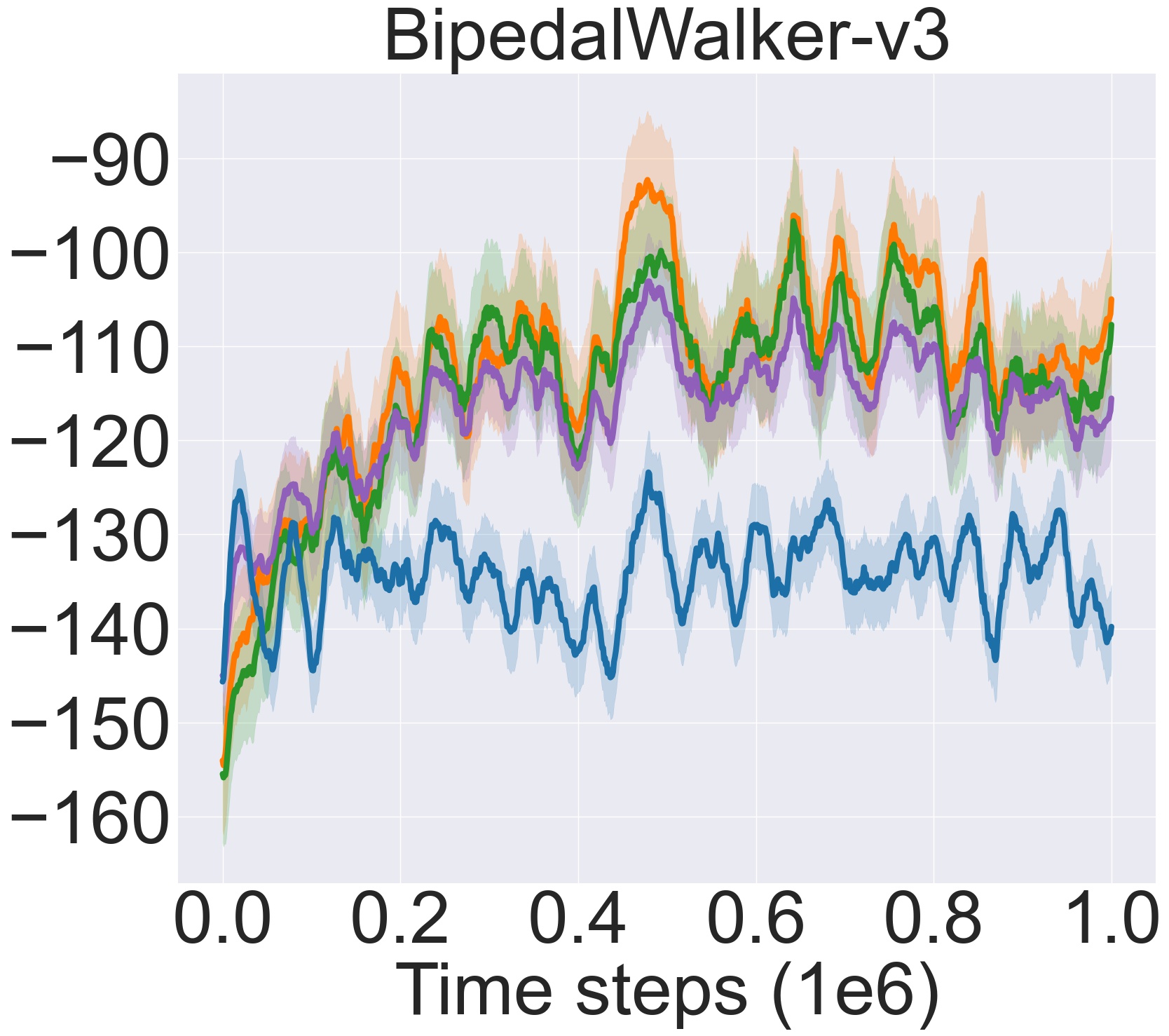}
		\includegraphics[width=1.40in, keepaspectratio]{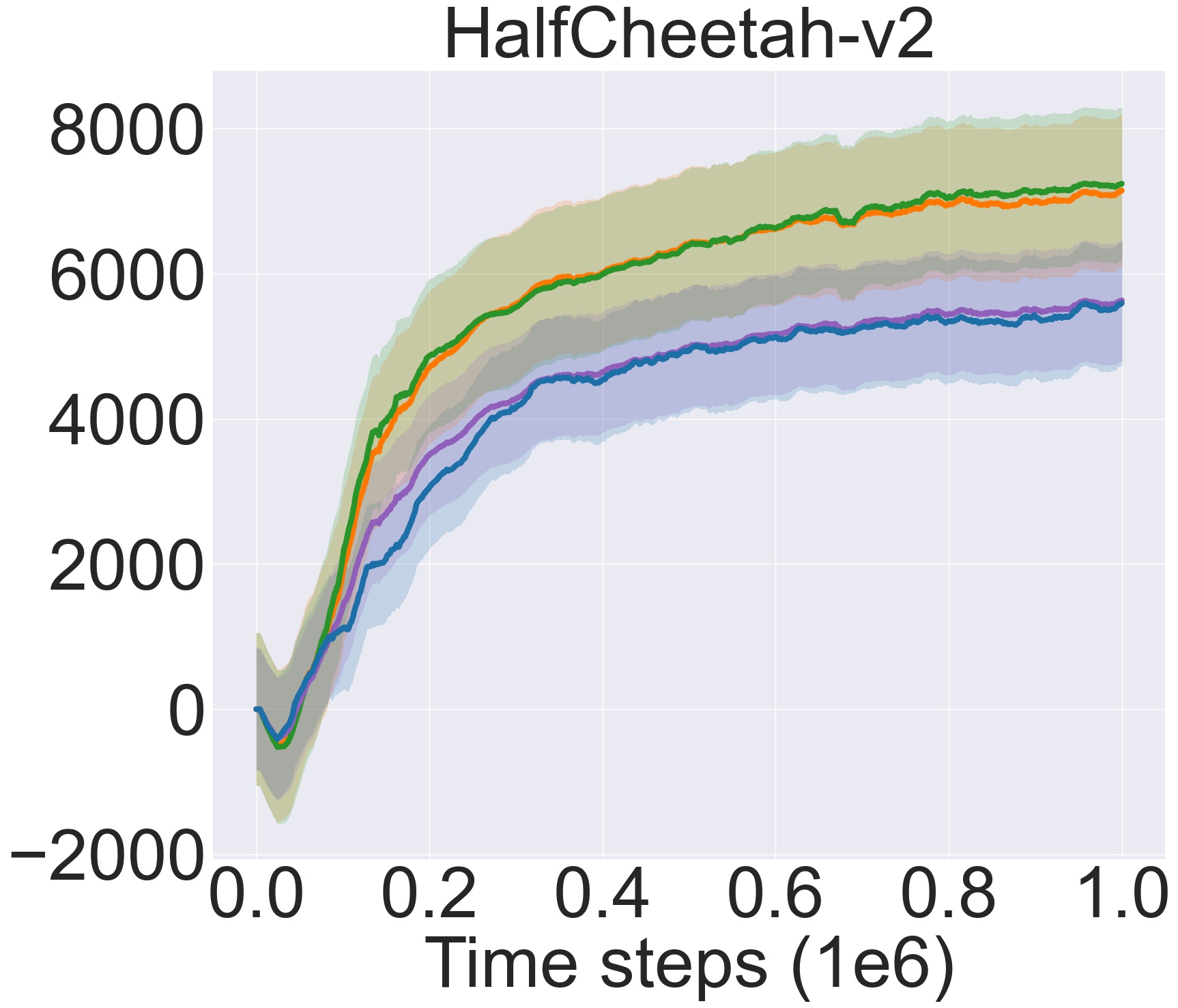}
		\includegraphics[width=1.40in, keepaspectratio]{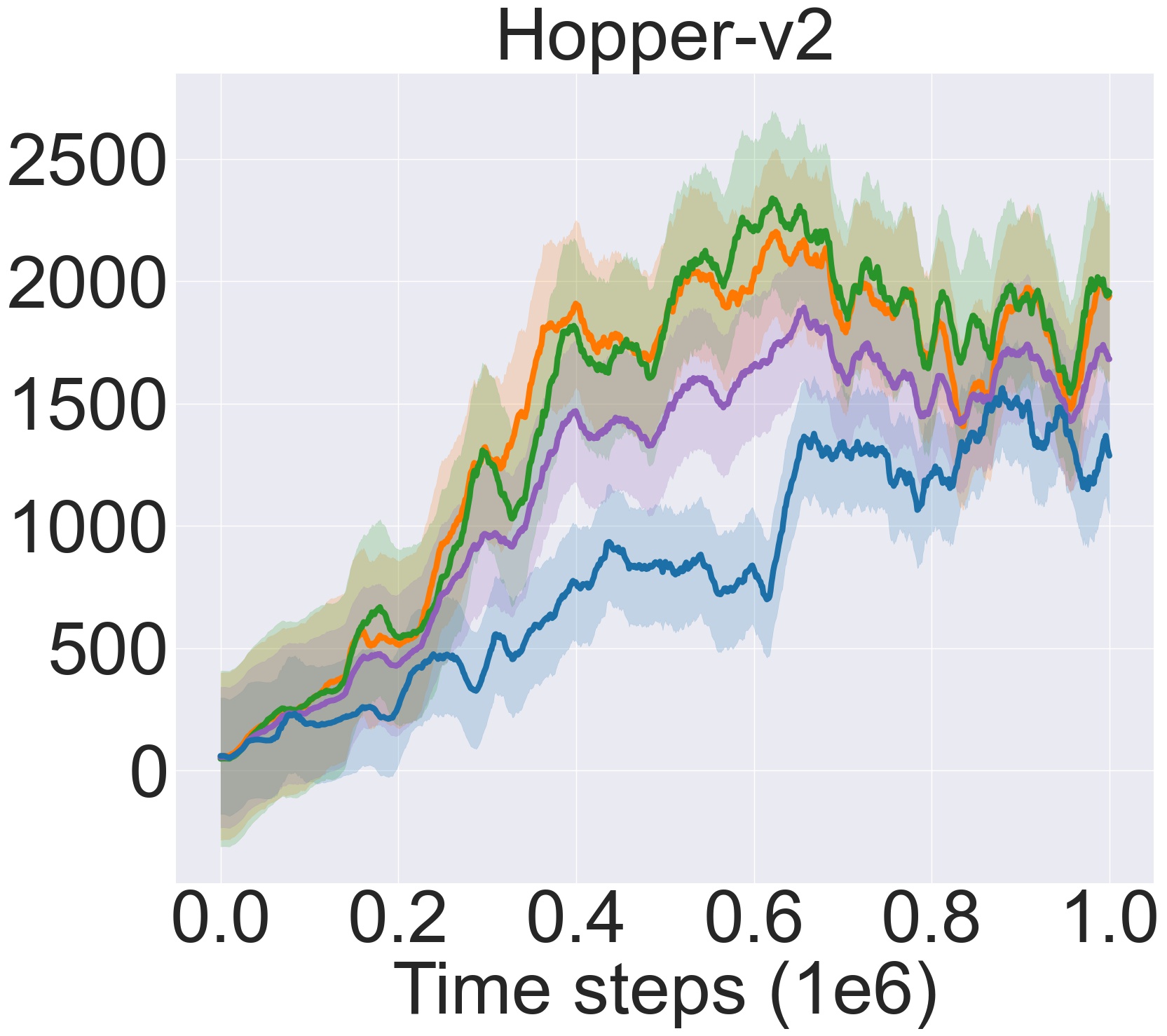}
	} \\
	\subfloat[\textbf{Replay Size:} 1,000,000]{
	    \includegraphics[width=1.40in, keepaspectratio]{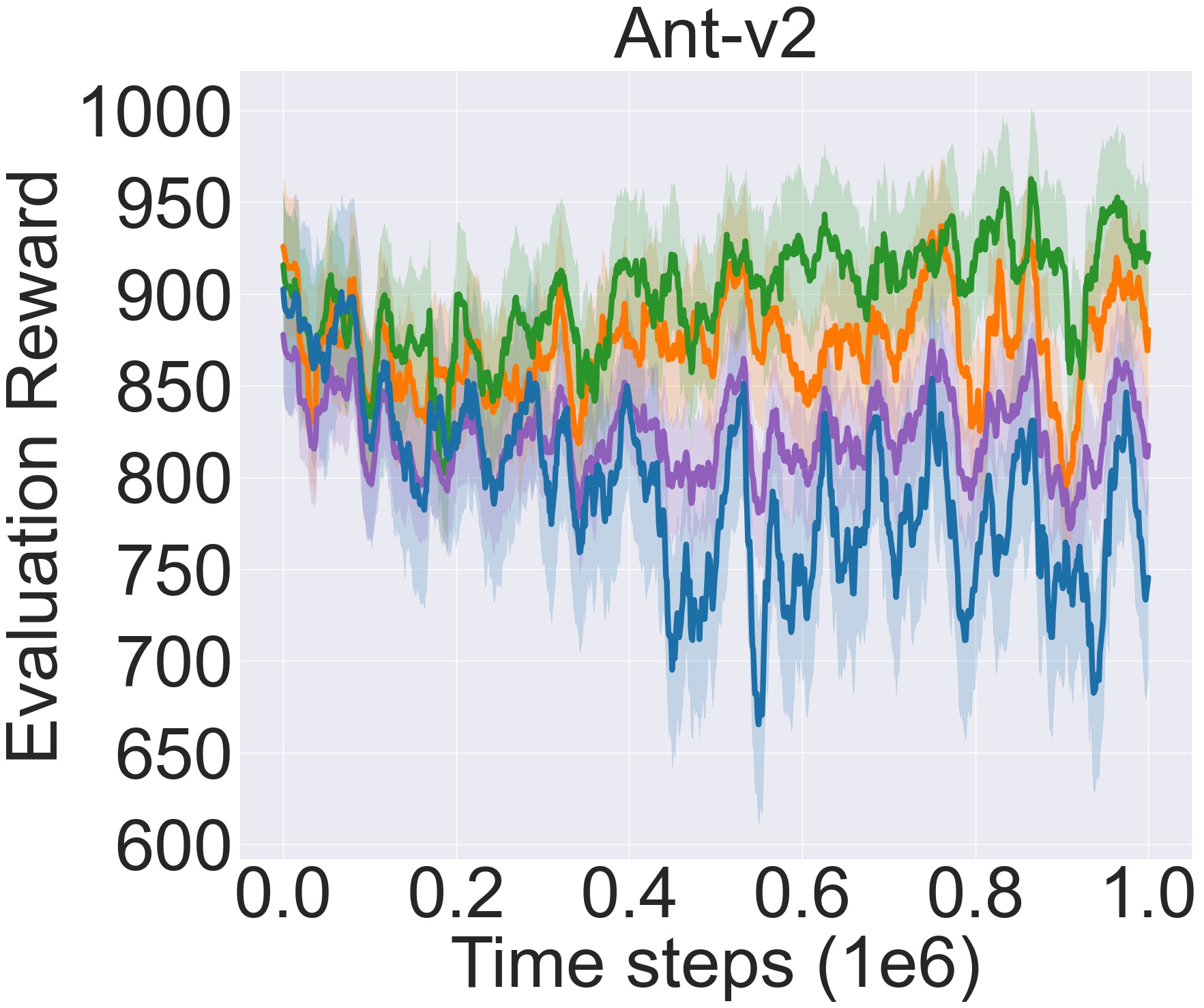}
		\includegraphics[width=1.40in, keepaspectratio]{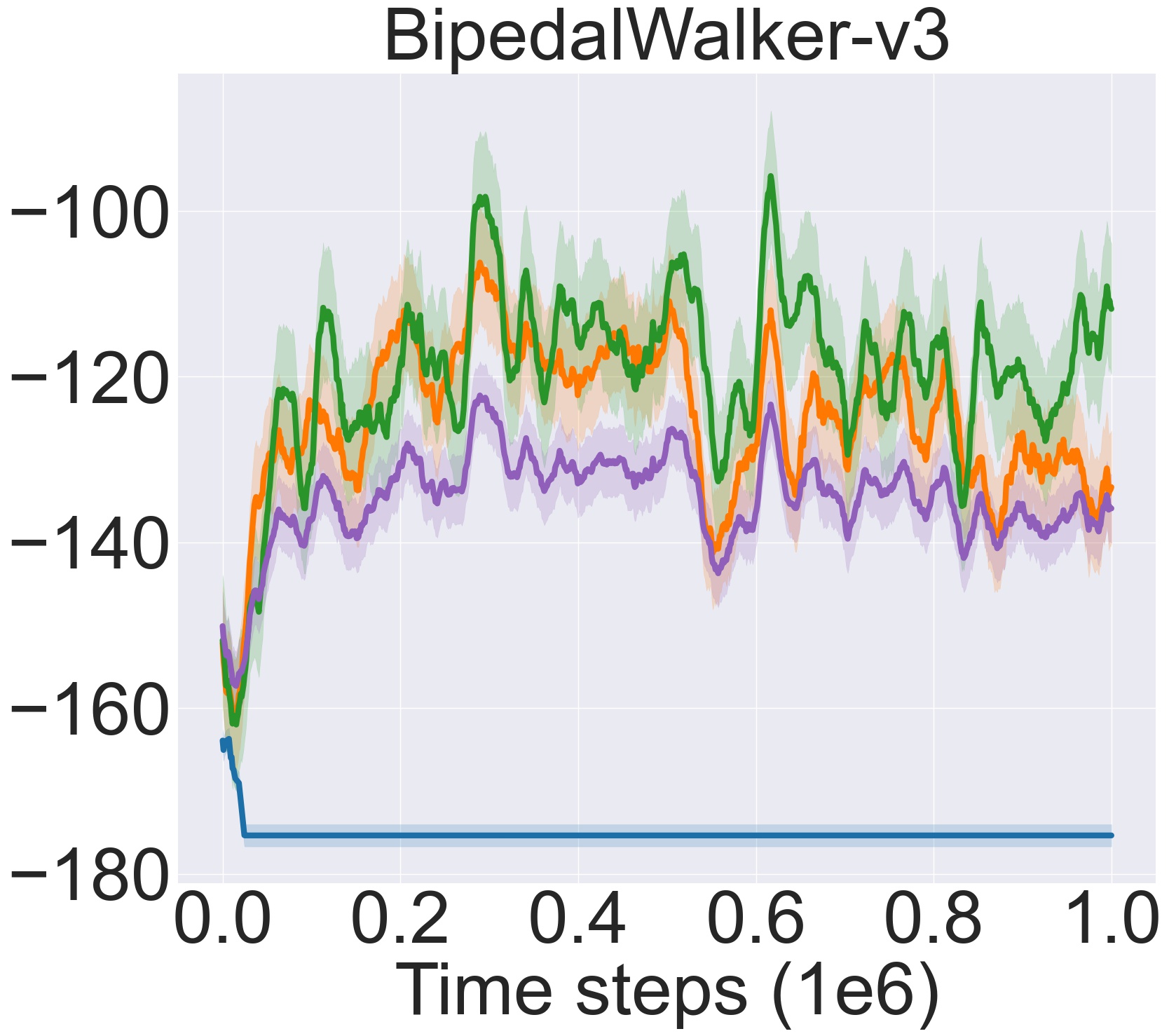} 
		\includegraphics[width=1.40in, keepaspectratio]{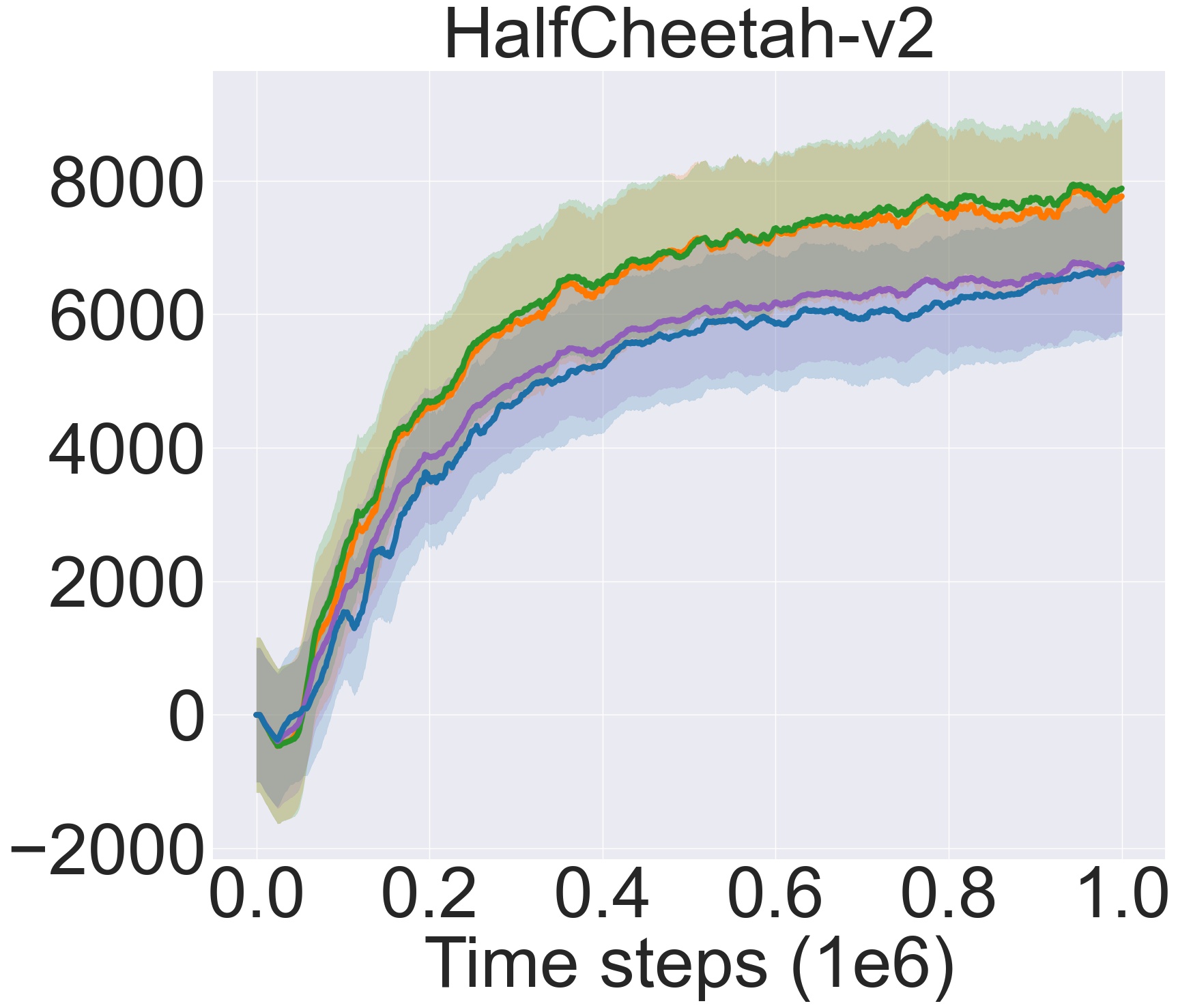}
		\includegraphics[width=1.40in, keepaspectratio]{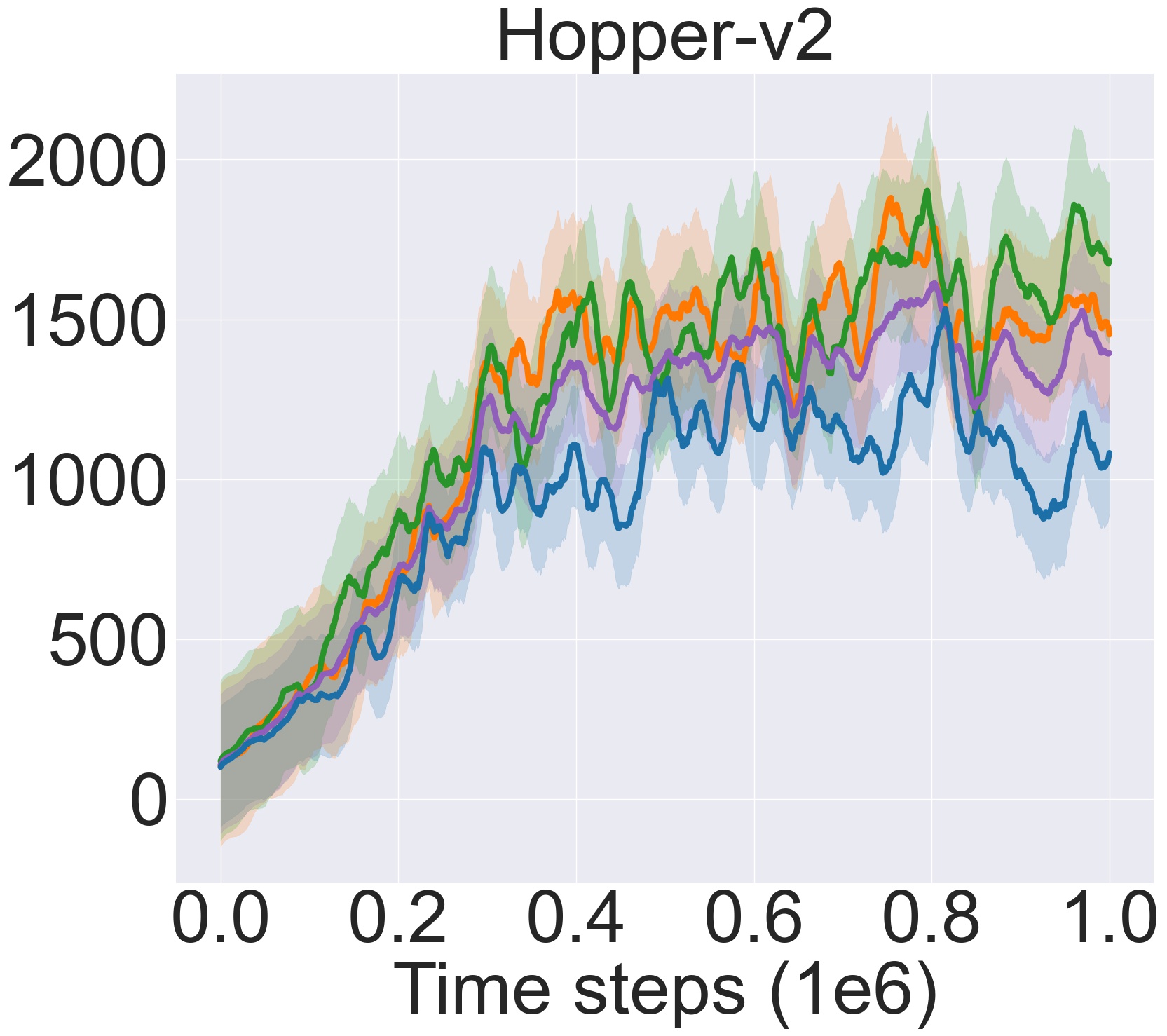}
	} \\
	\subfloat[\textbf{Replay Size:} 100,000]{
	    \includegraphics[width=1.40in, keepaspectratio]{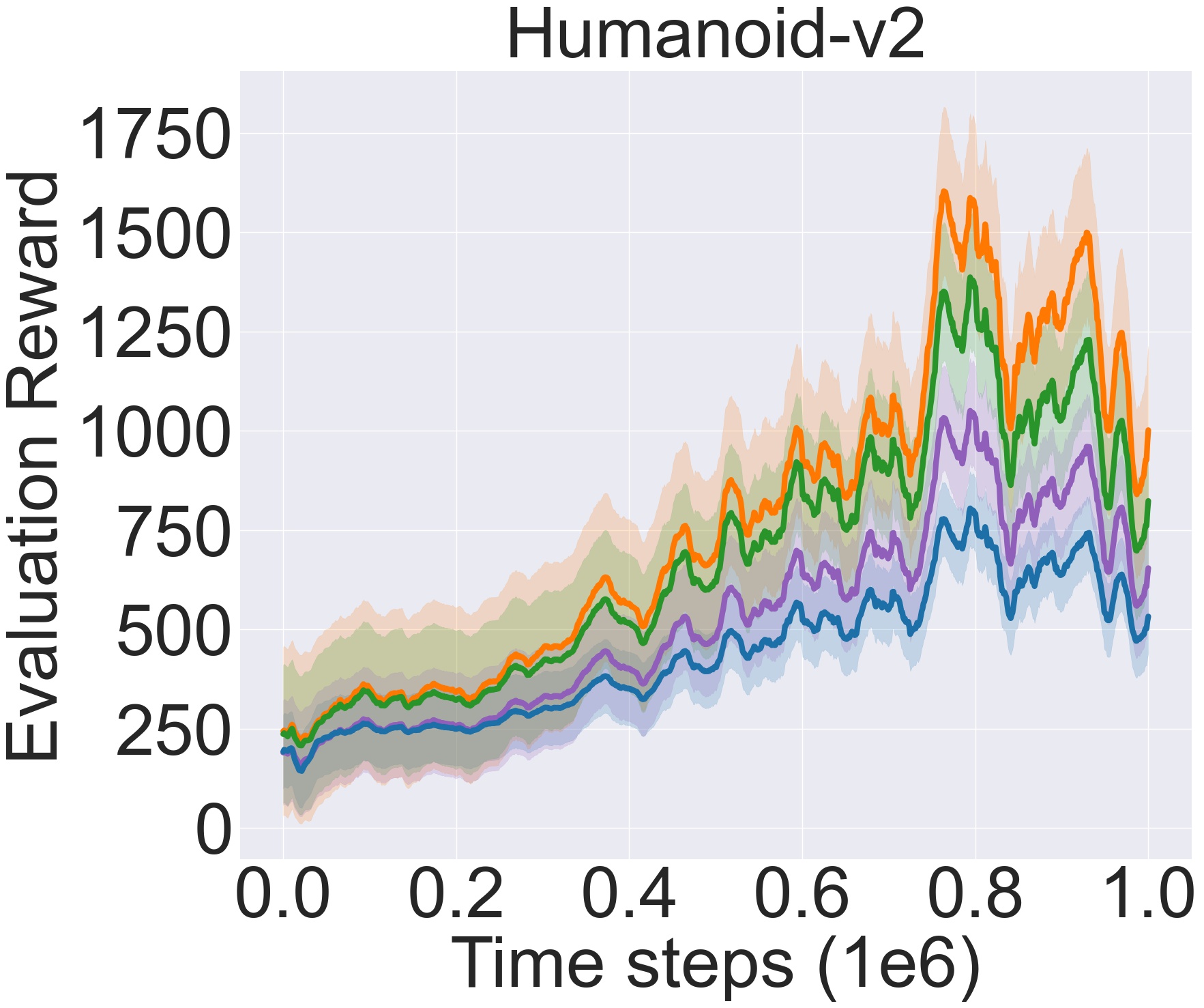}
		\includegraphics[width=1.40in, keepaspectratio]{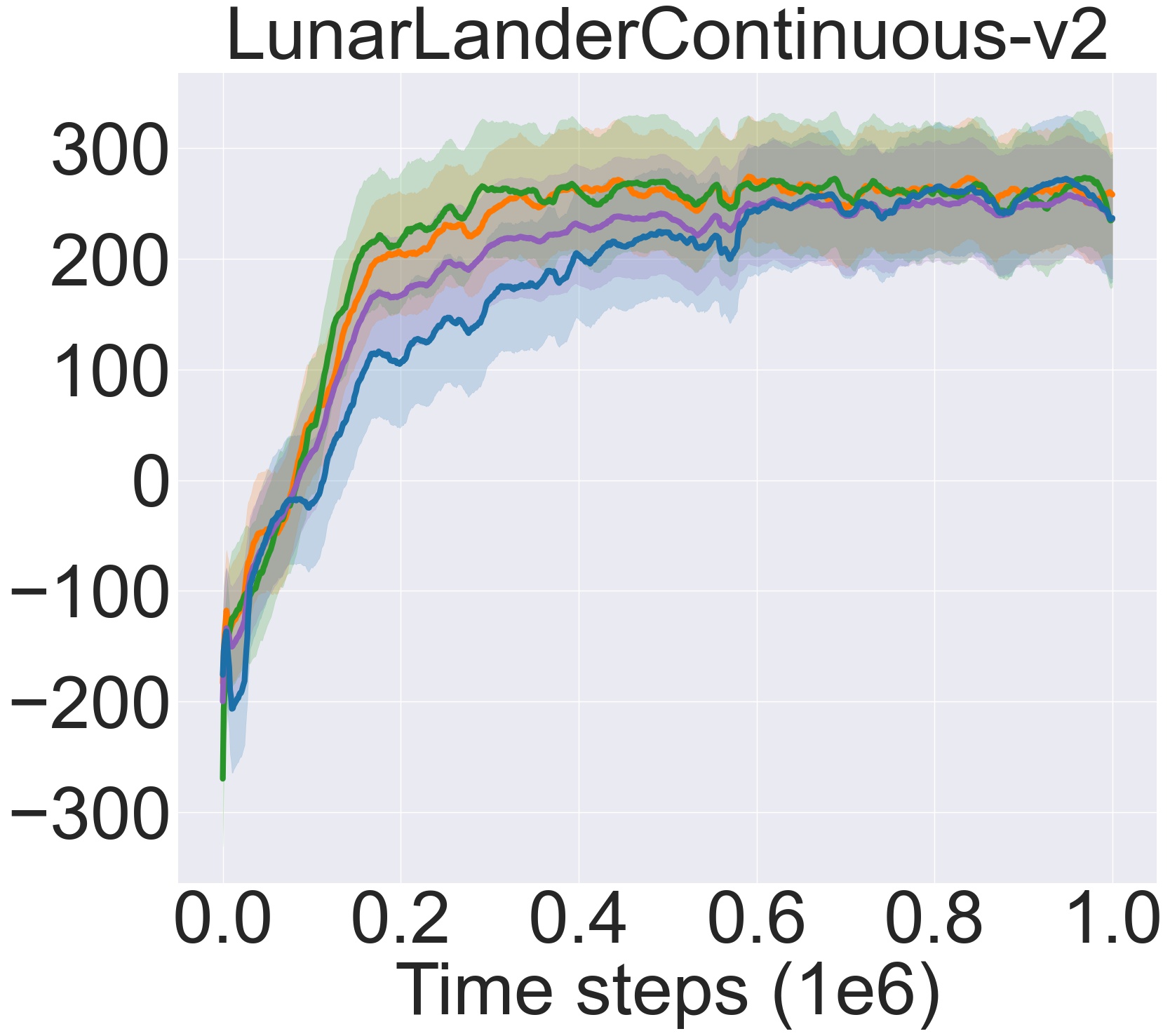}
		\includegraphics[width=1.40in, keepaspectratio]{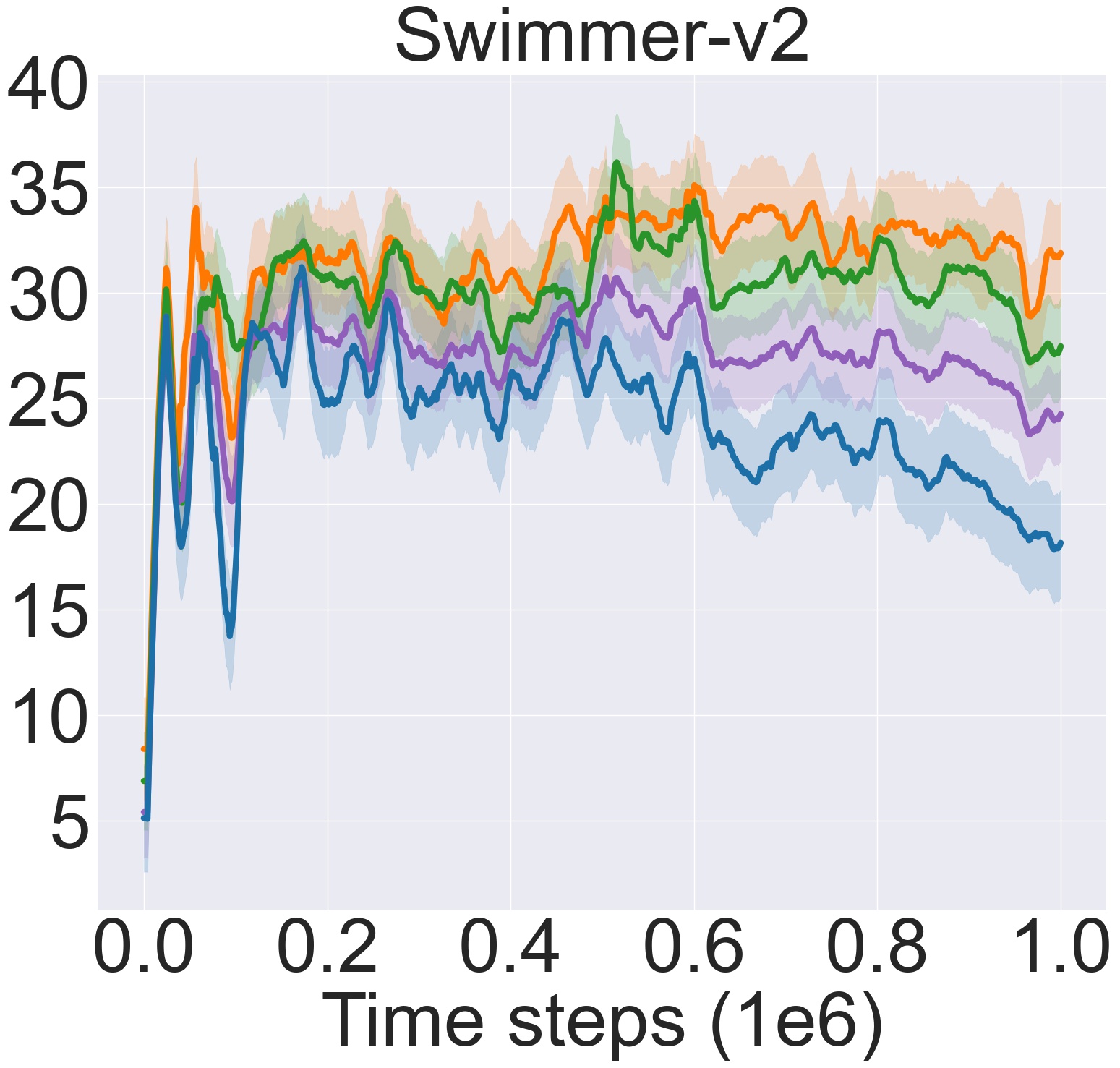}
		\includegraphics[width=1.40in, keepaspectratio]{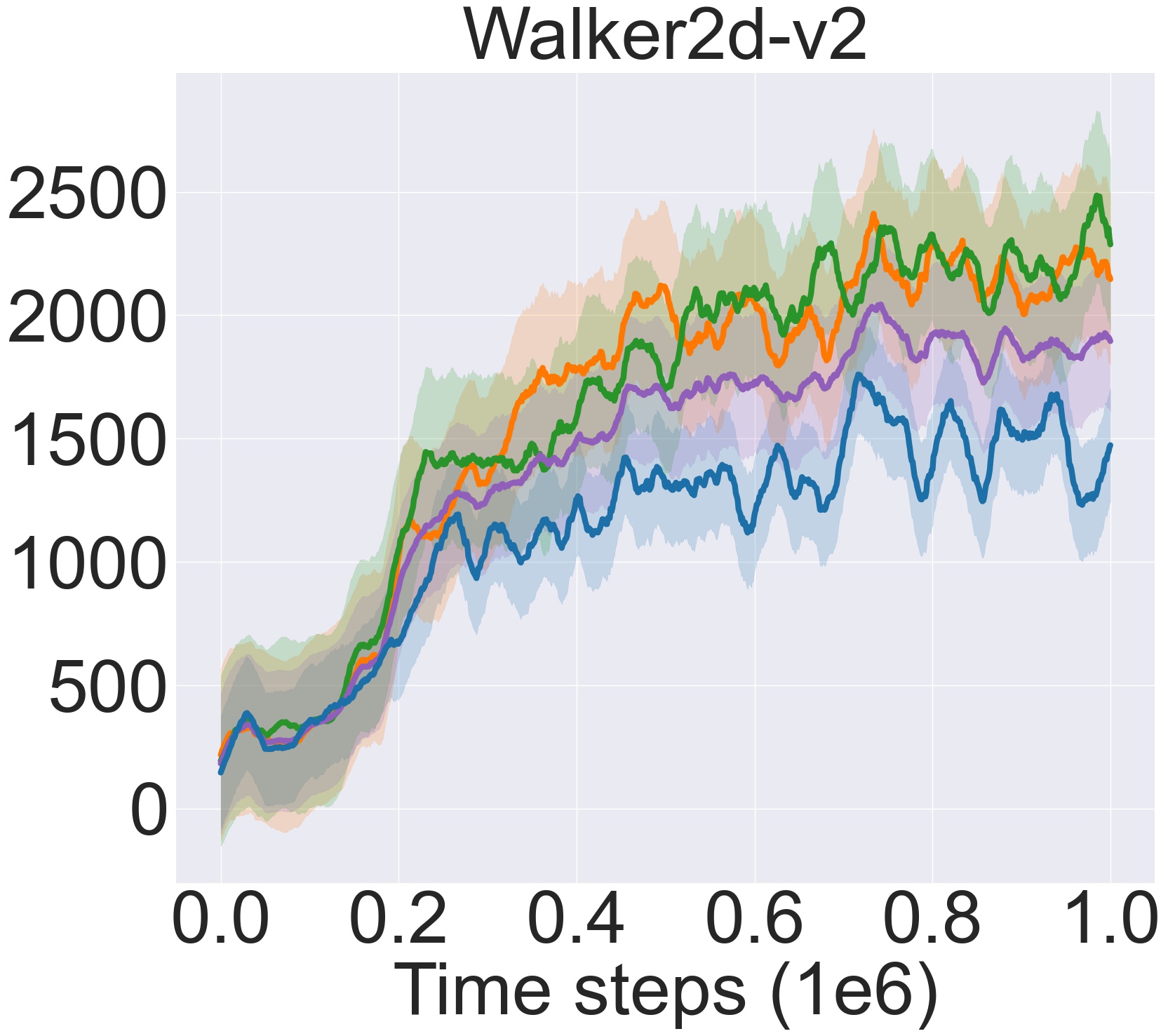}
	} \\
	\subfloat[\textbf{Replay Size:} 1,000,000]{
		\includegraphics[width=1.40in, keepaspectratio]{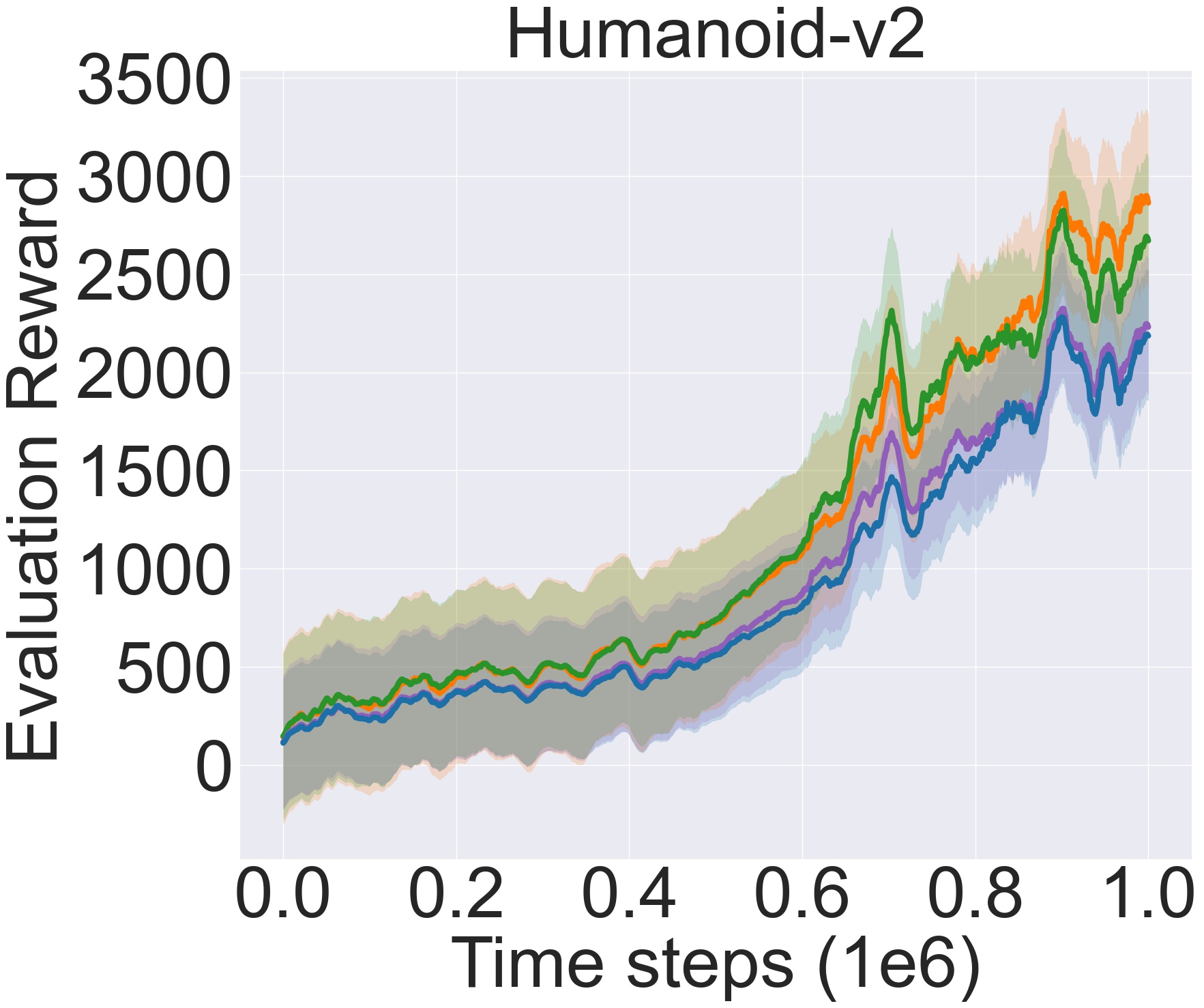}
		\includegraphics[width=1.40in, keepaspectratio]{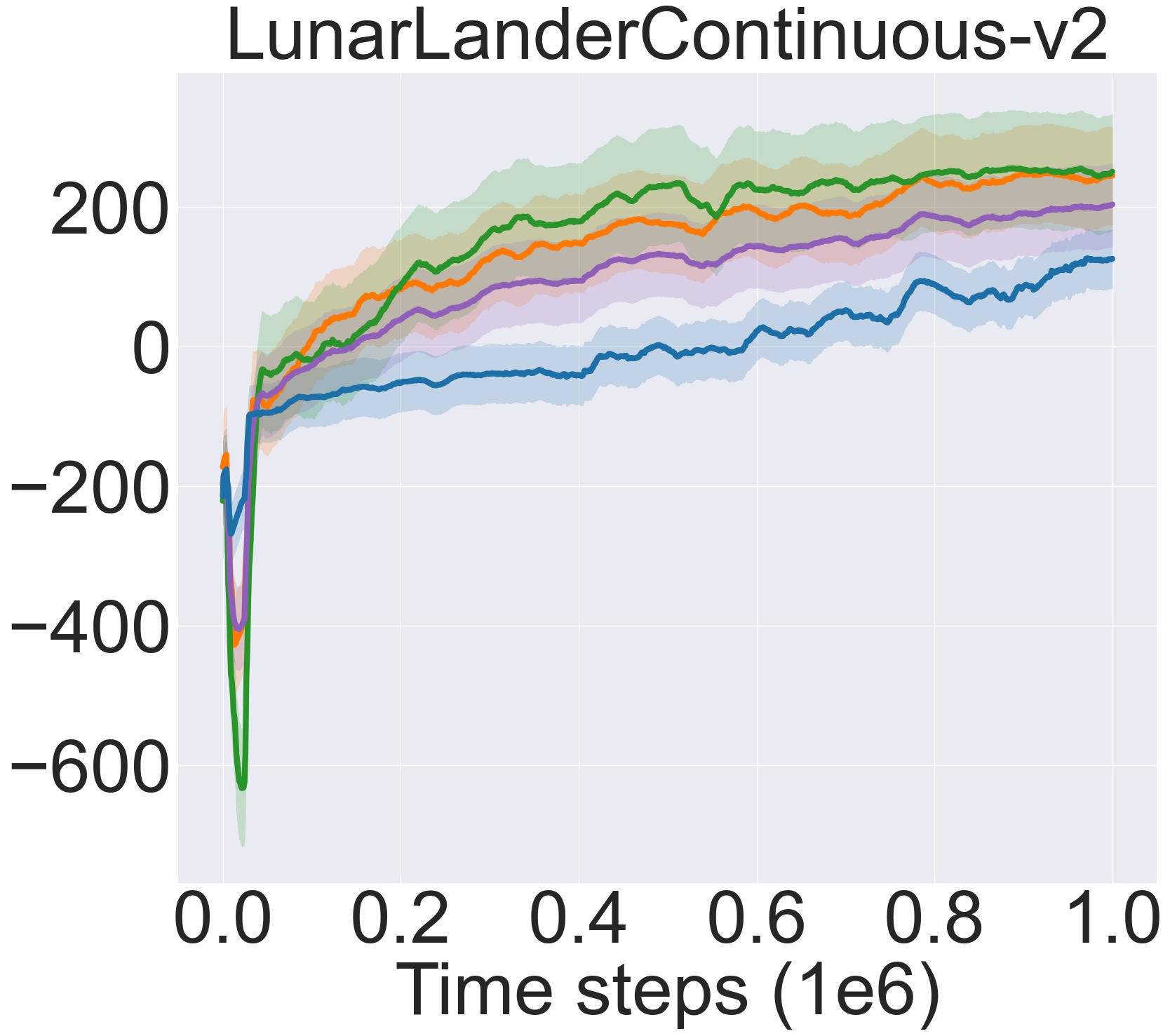}
		\includegraphics[width=1.40in, keepaspectratio]{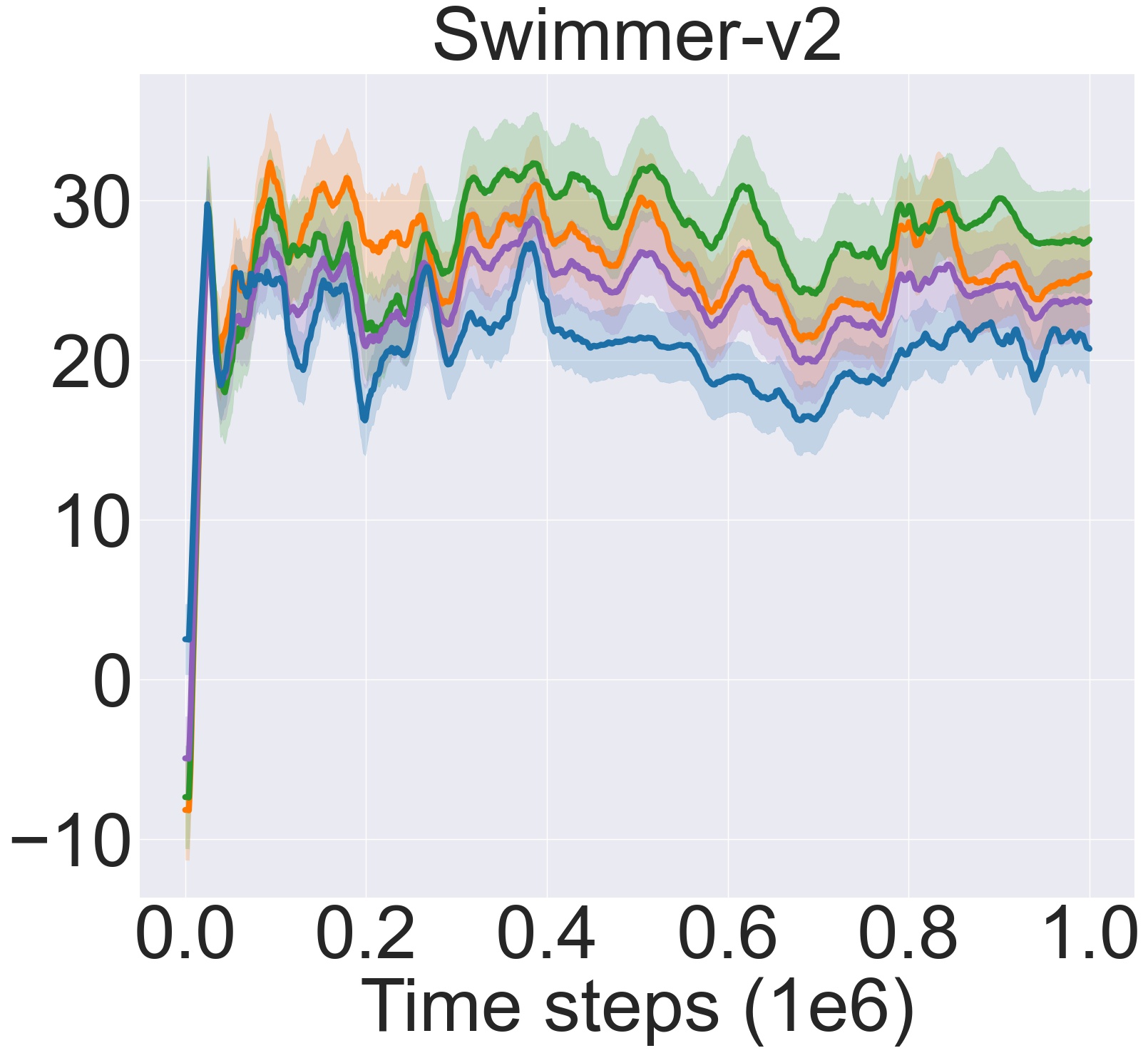}
		\includegraphics[width=1.40in, keepaspectratio]{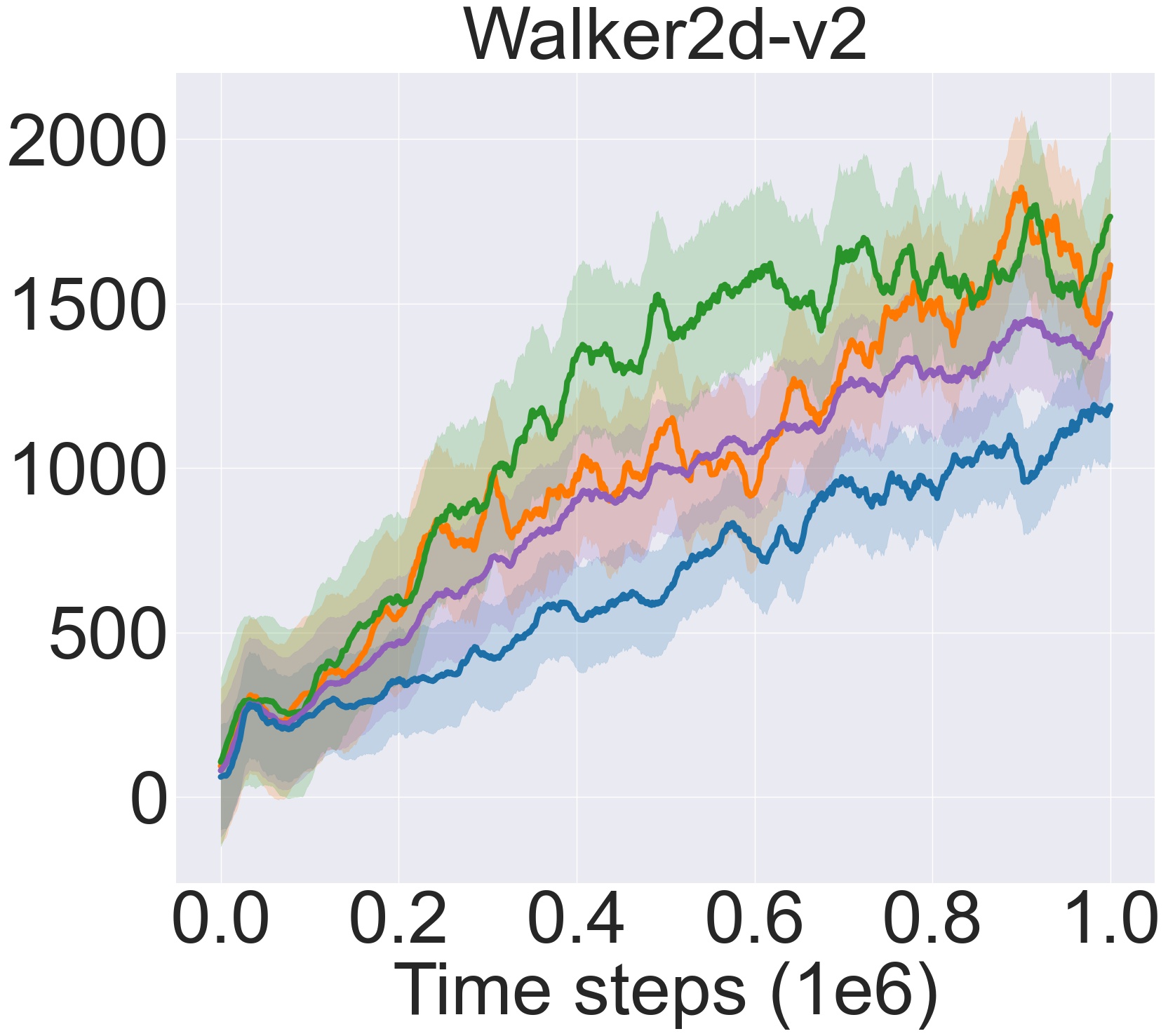}
	}
	\caption{Learning curves for the set of OpenAI Gym continuous control tasks under the DDPG algorithm. The shaded region represents half a standard deviation of the average evaluation return over 10 random seeds. A sliding window smoothes curves for visual clarity.}
\end{figure*}

\begin{figure*}[!]
    \centering
    \begin{align*}
        &\text{{\blue} SAC (single agent)} \quad &&\text{{\orange} SAC + DASE ($1^{\text{st}}$ agent)} \\
        &\text{\text{{\purple} SAC + DPD (average of two agents)}} \quad &&\text{{\green} SAC + DASE ($2^{\text{nd}}$ agent)}
    \end{align*}
	\subfloat[\textbf{Replay Size:} 100,000]{
		\includegraphics[width=1.40in, keepaspectratio]{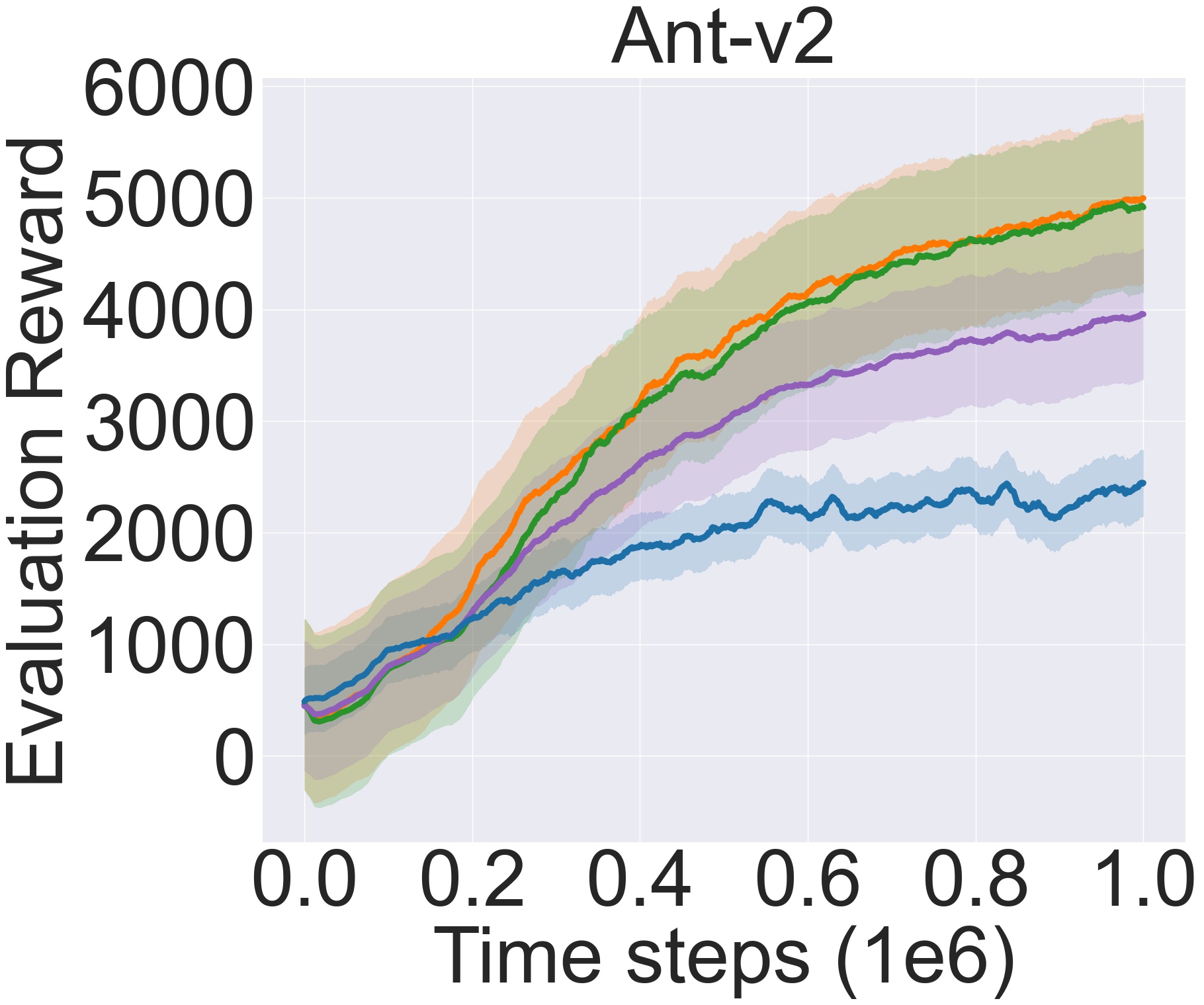}
		\includegraphics[width=1.40in, keepaspectratio]{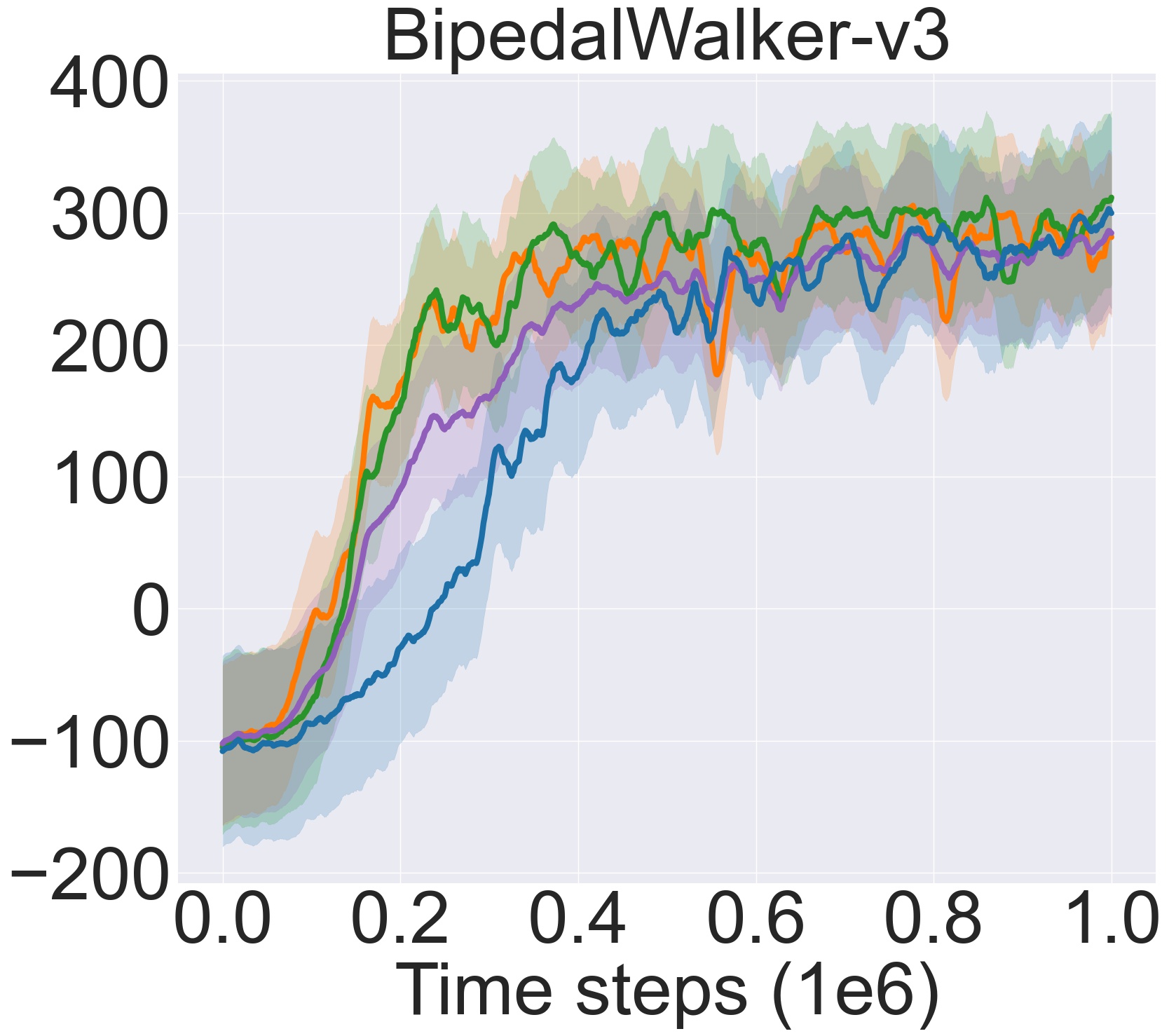}
		\includegraphics[width=1.40in, keepaspectratio]{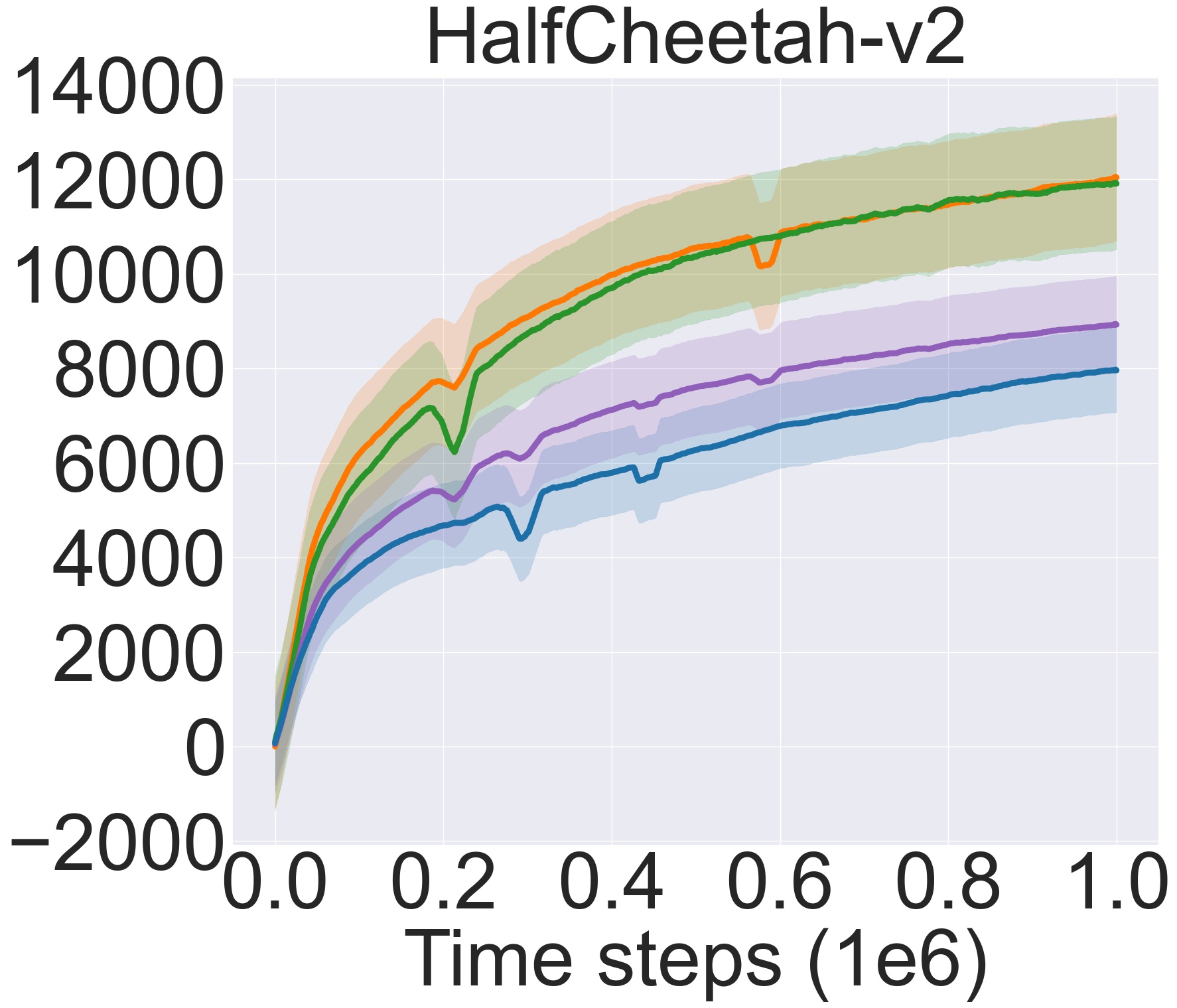}
		\includegraphics[width=1.40in, keepaspectratio]{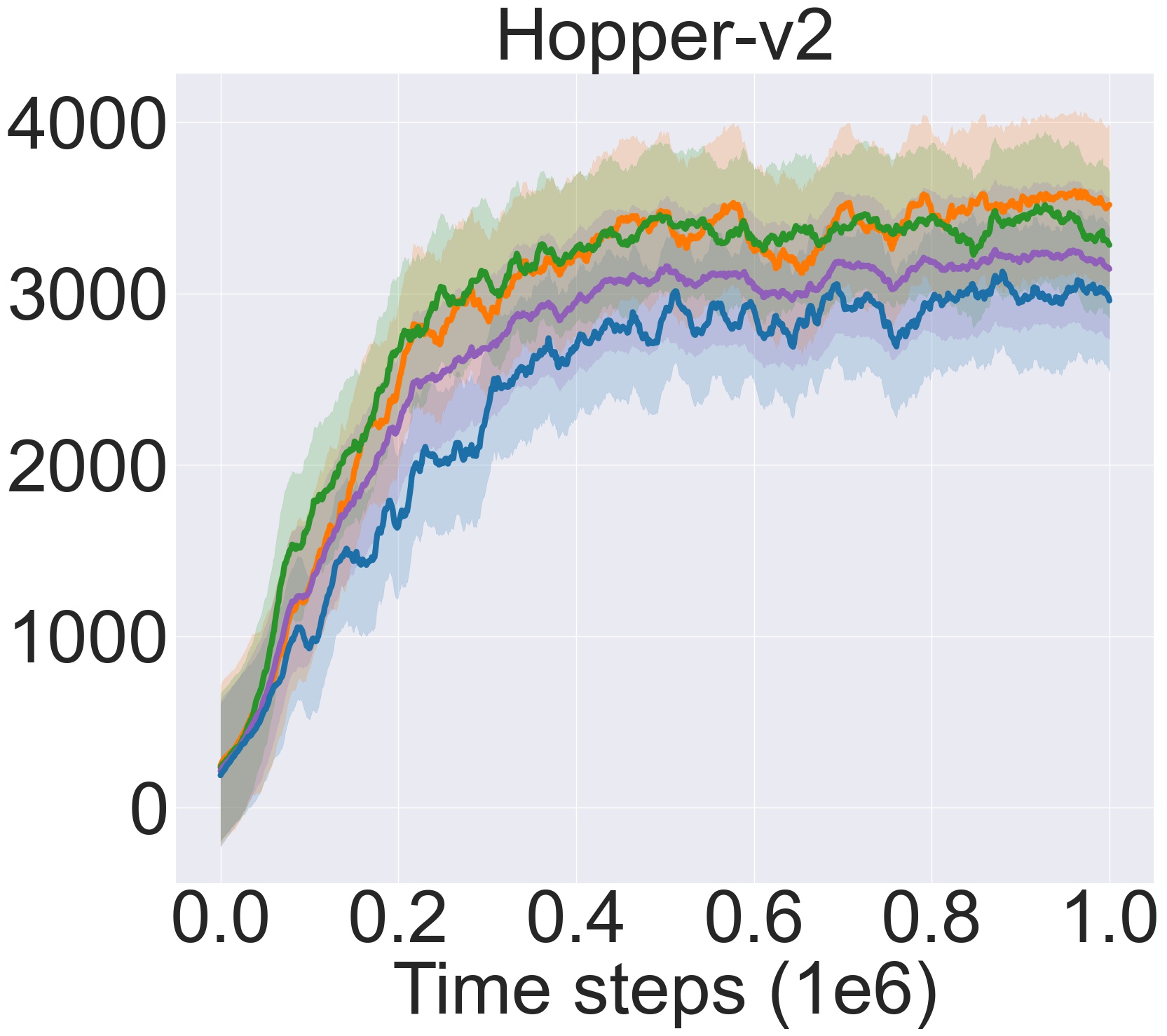}
	} \\
	\subfloat[\textbf{Replay Size:} 1,000,000]{
	    \includegraphics[width=1.40in, keepaspectratio]{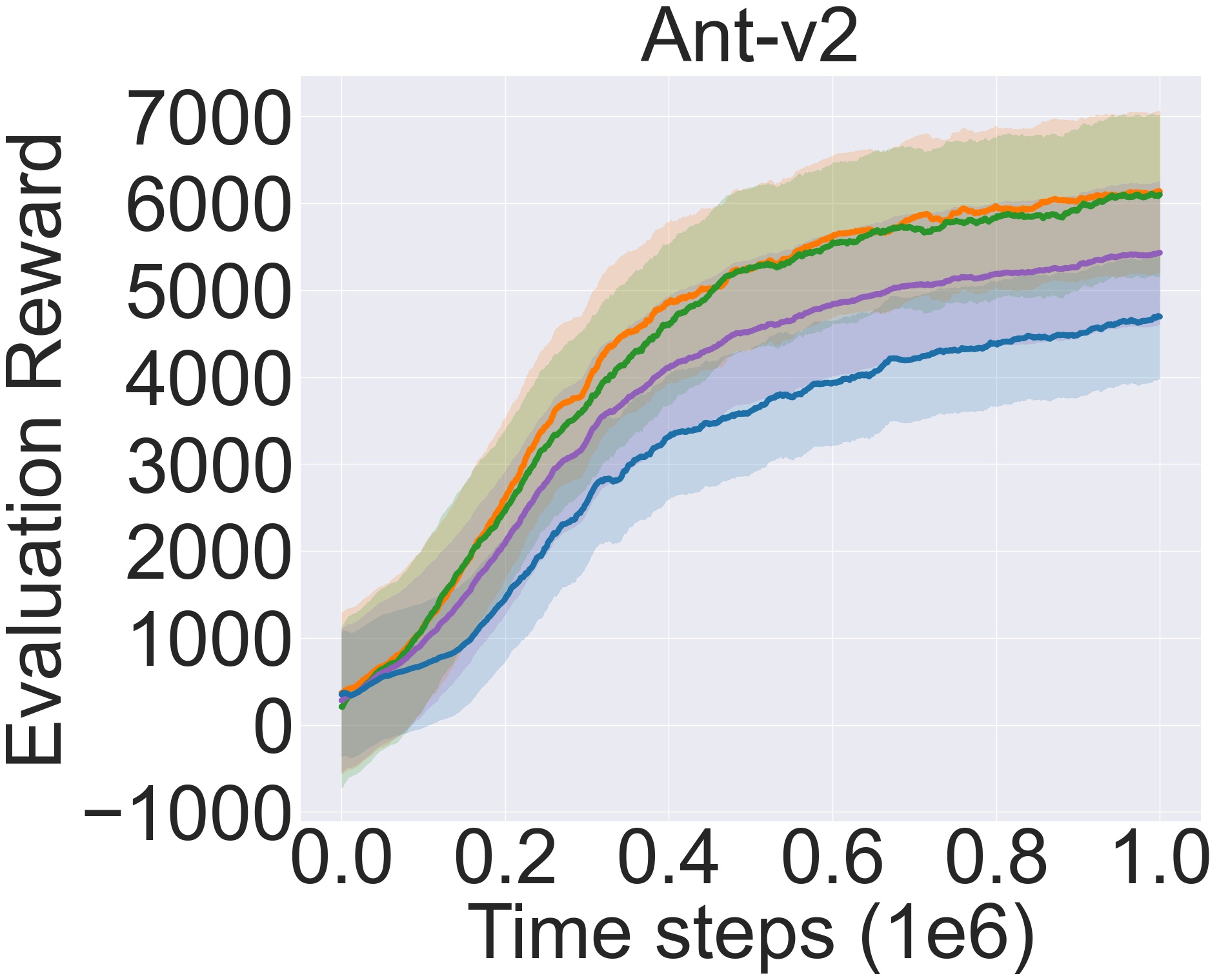}
		\includegraphics[width=1.40in, keepaspectratio]{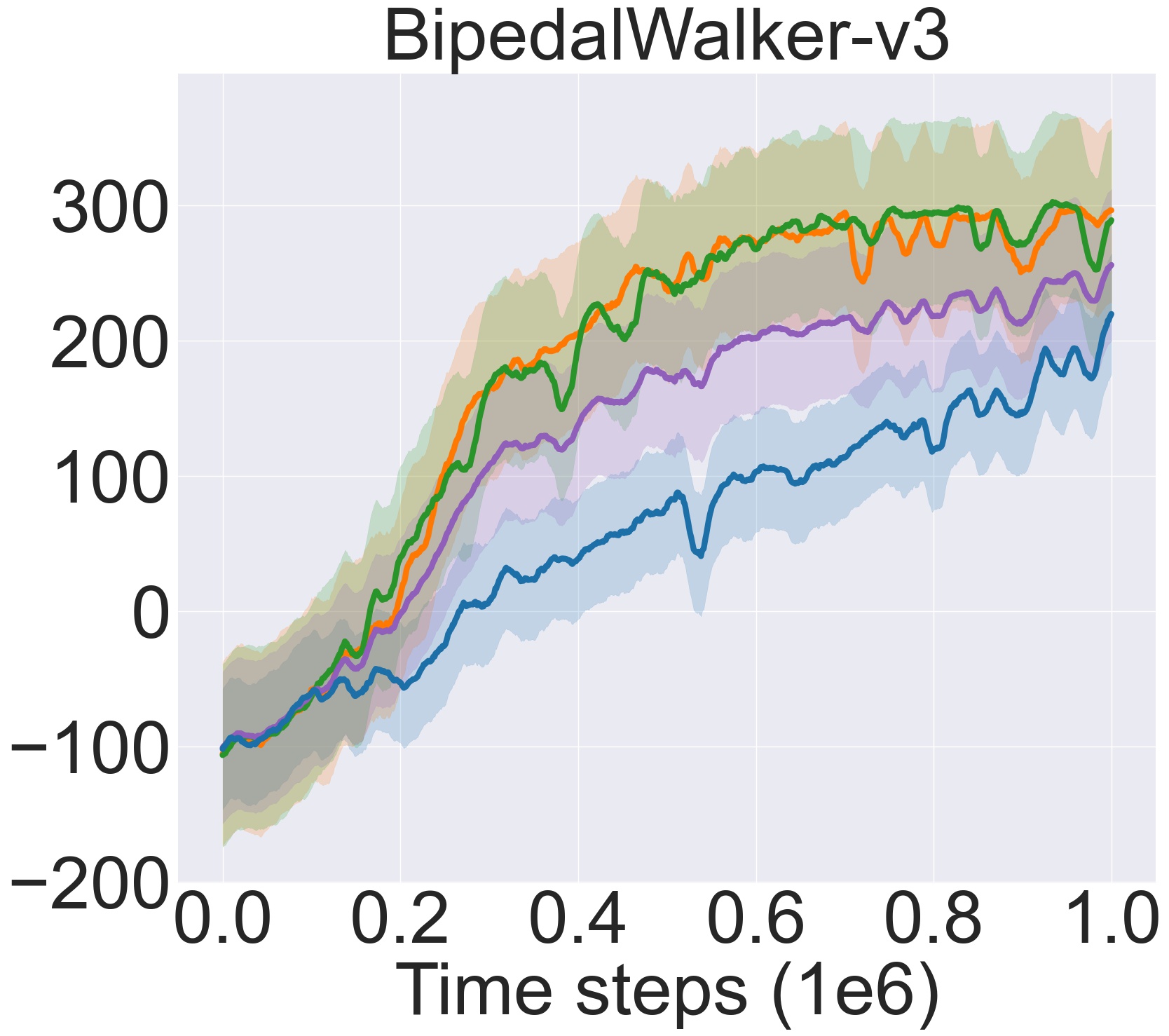} 
		\includegraphics[width=1.40in, keepaspectratio]{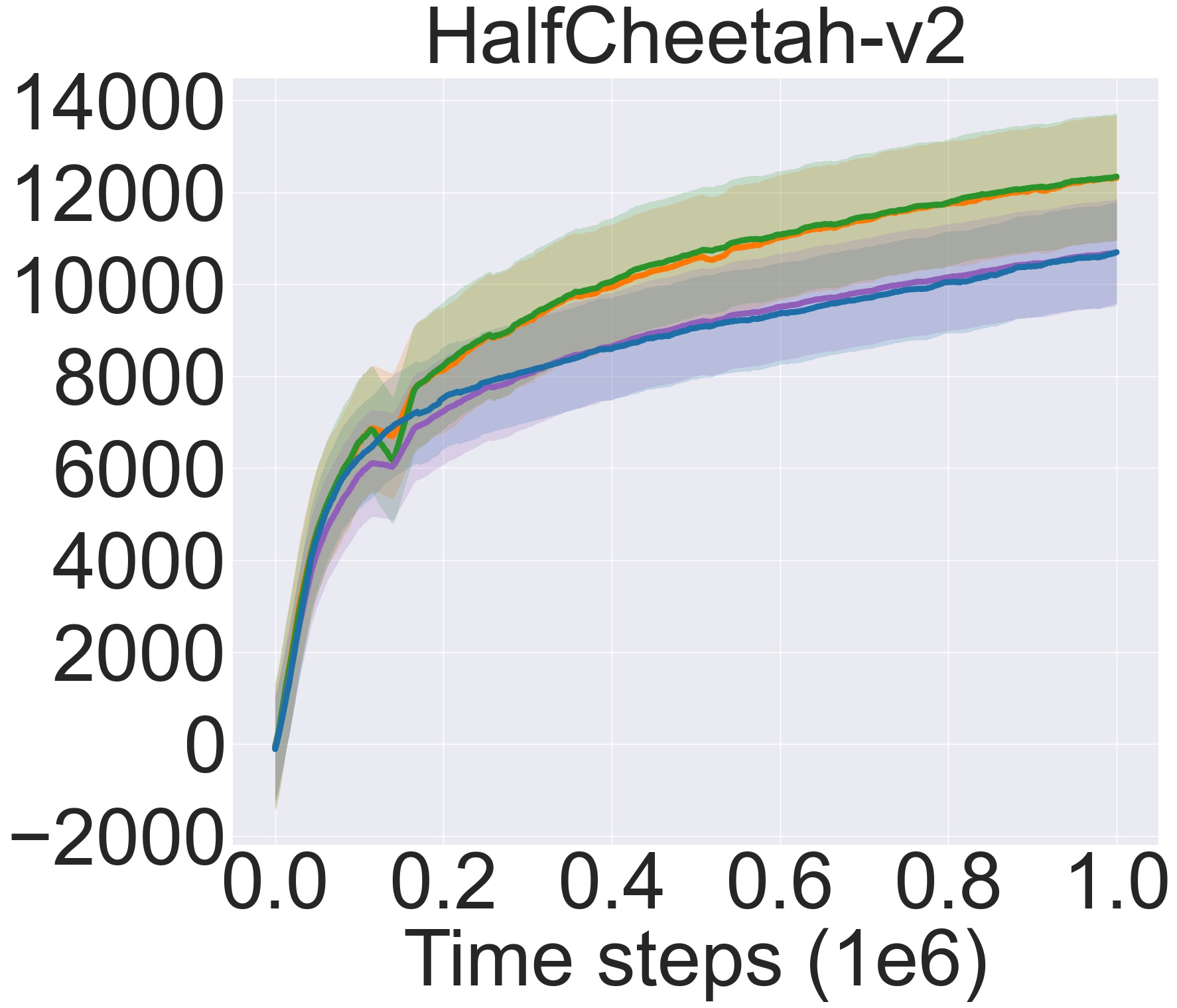}
		\includegraphics[width=1.40in, keepaspectratio]{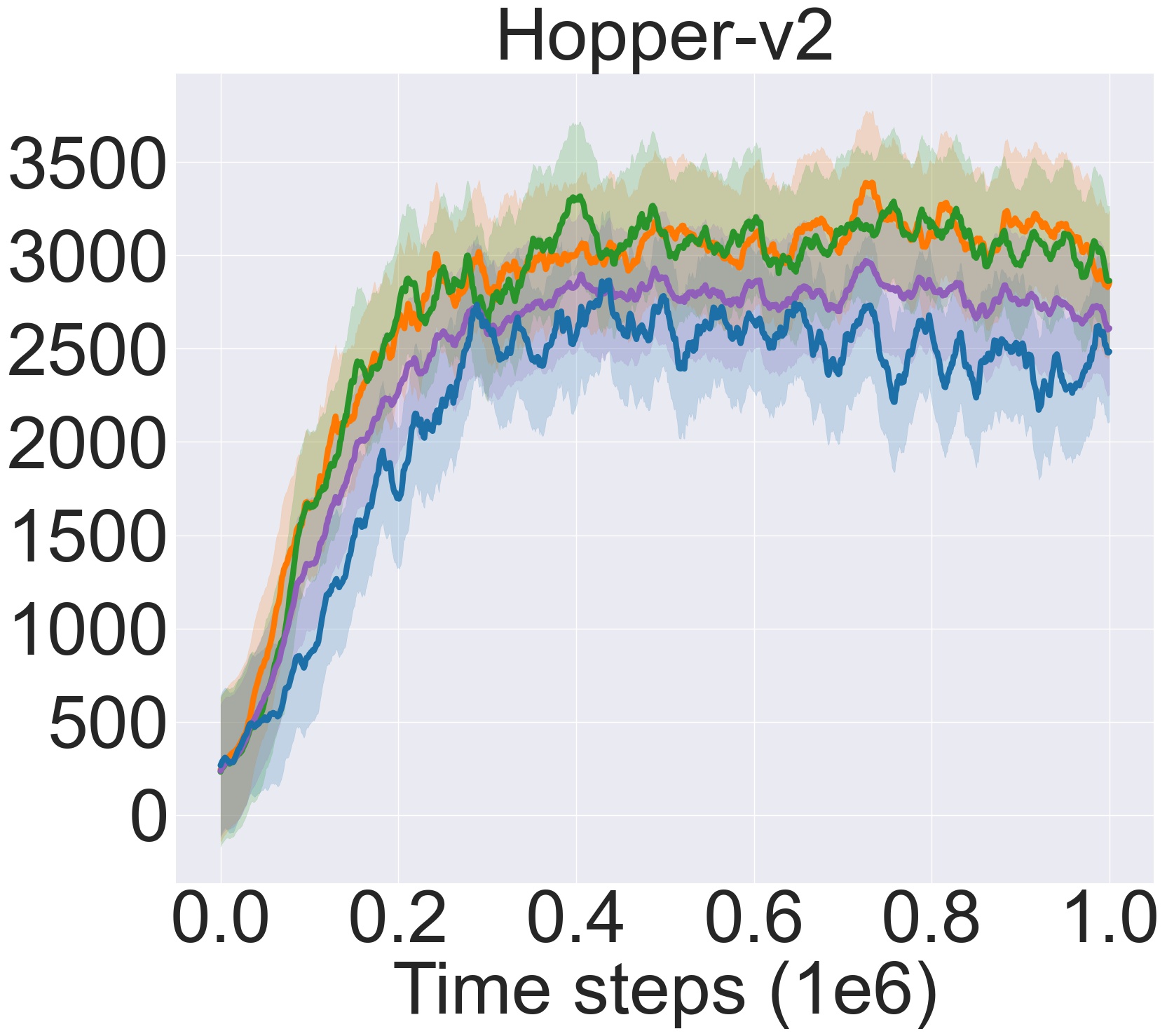}
	} \\
	\subfloat[\textbf{Replay Size:} 100,000]{
	    \includegraphics[width=1.40in, keepaspectratio]{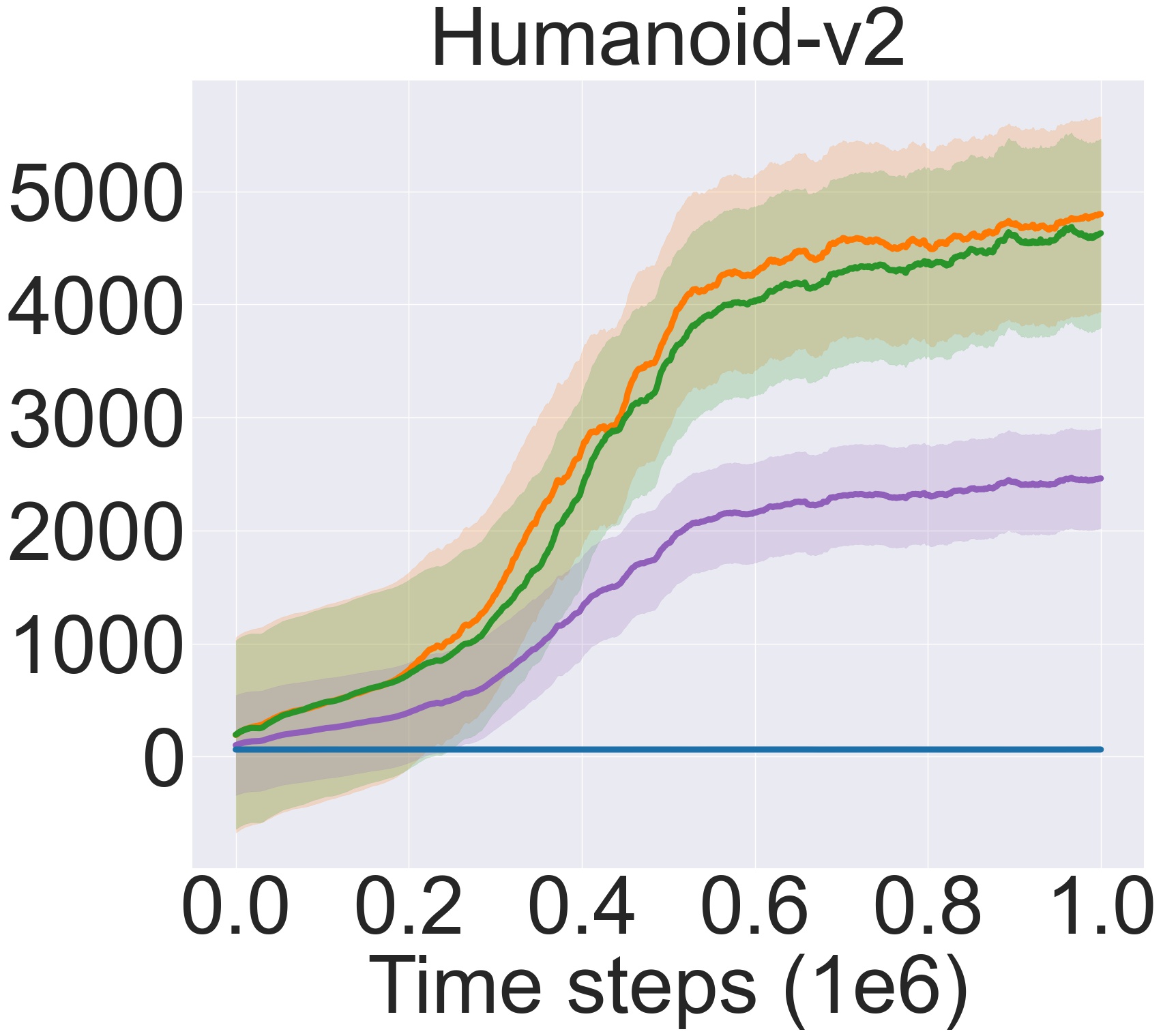}
		\includegraphics[width=1.40in, keepaspectratio]{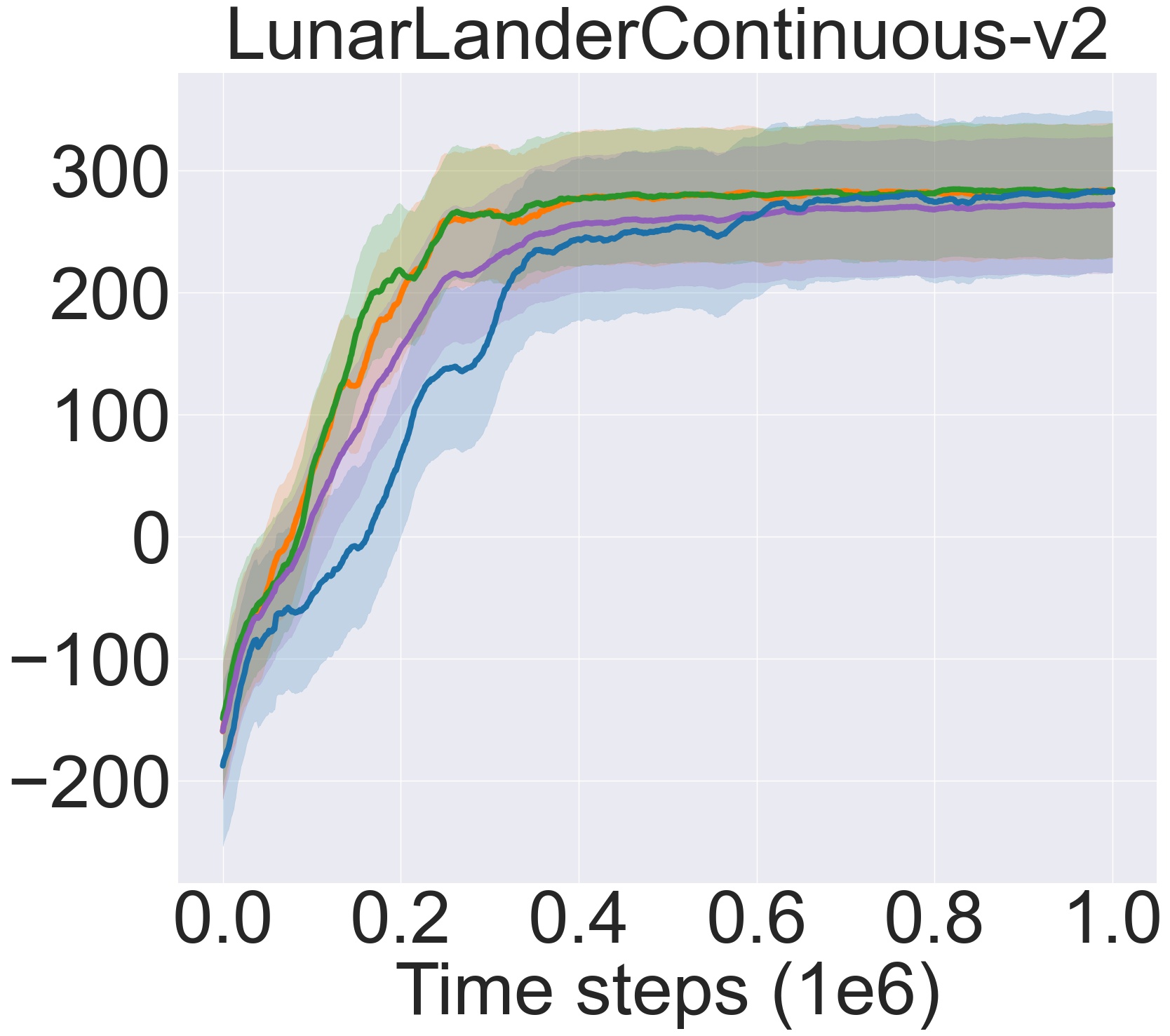}
		\includegraphics[width=1.40in, keepaspectratio]{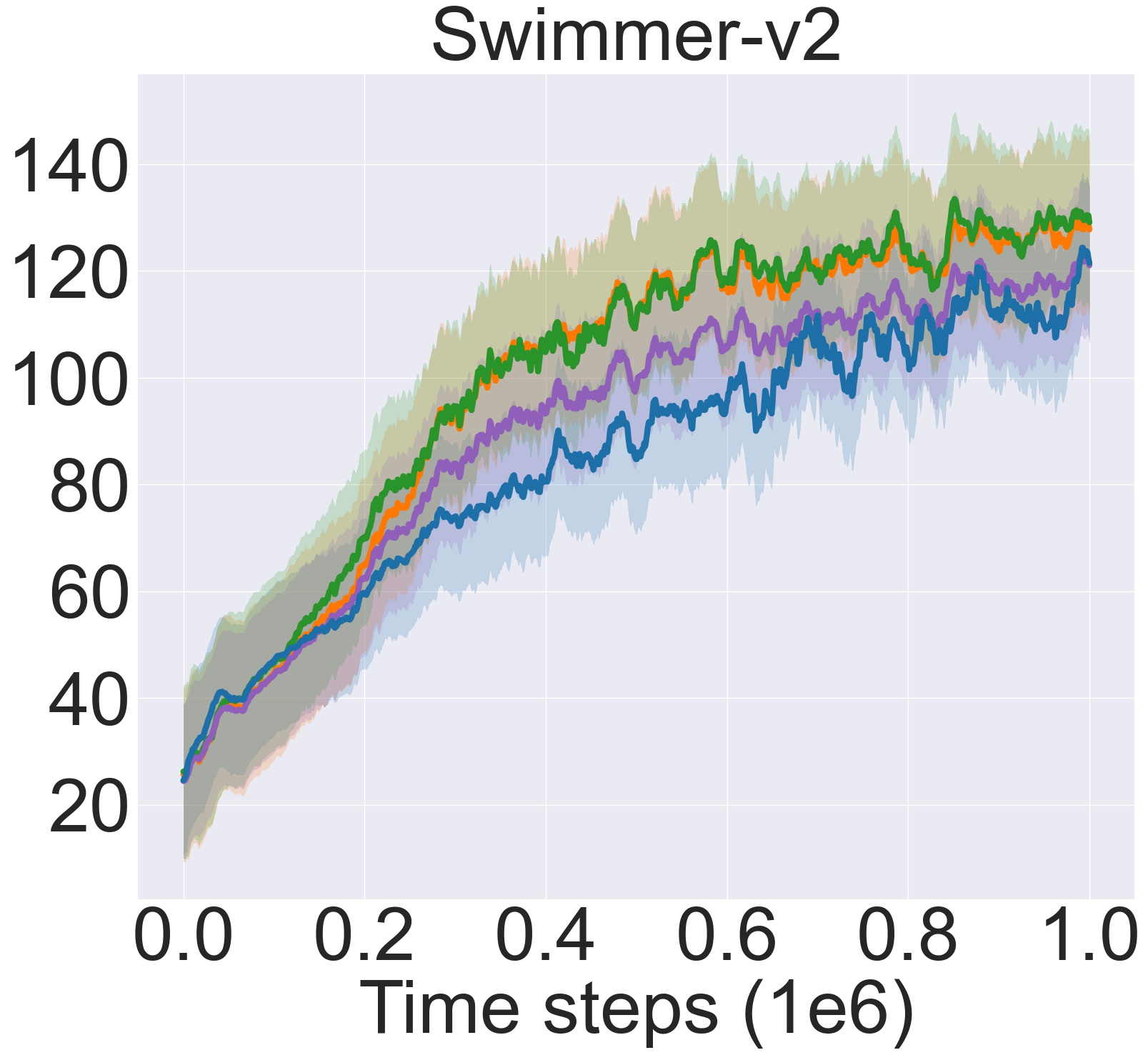}
		\includegraphics[width=1.40in, keepaspectratio]{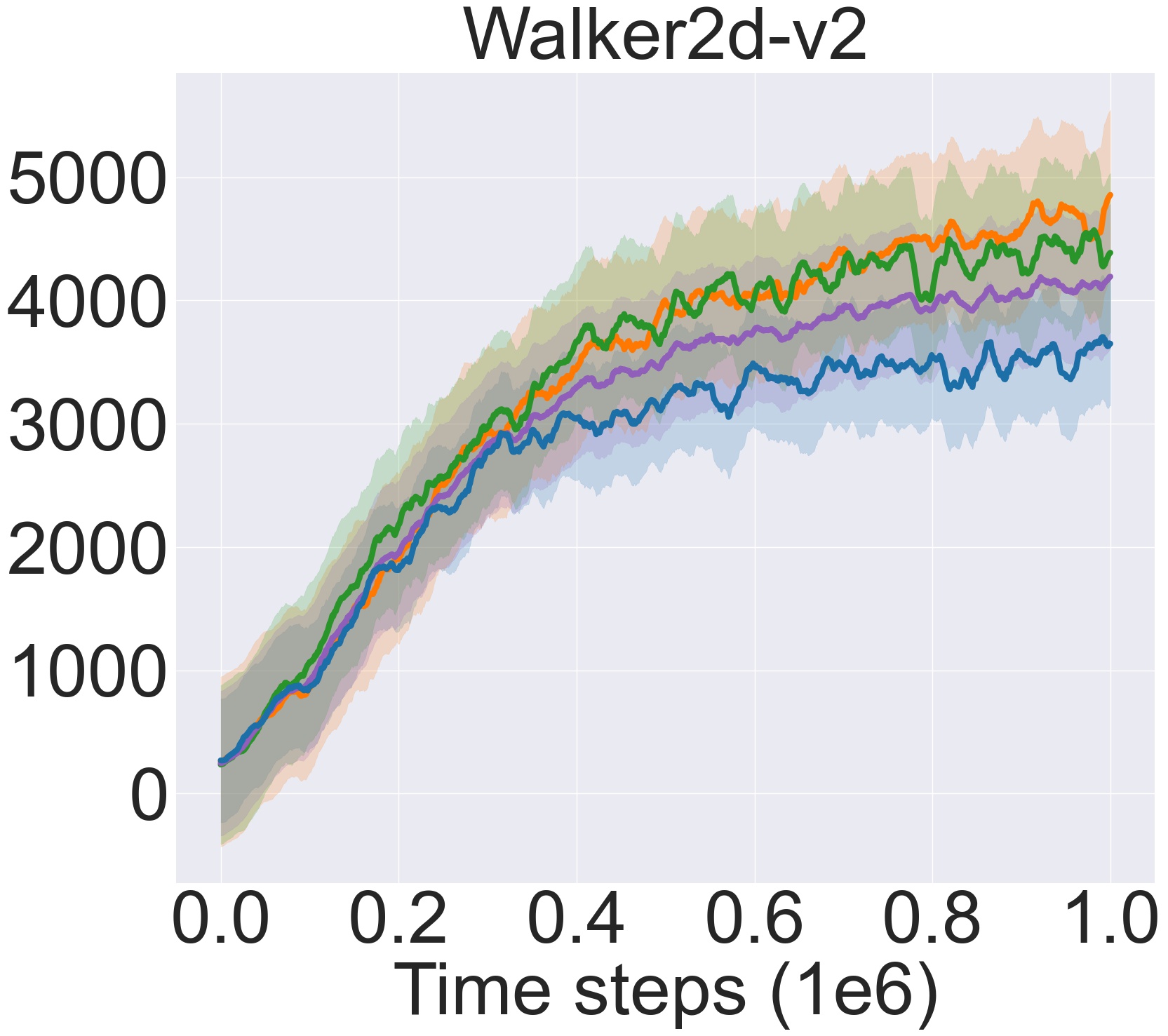}
	} \\
	\subfloat[\textbf{Replay Size:} 1,000,000]{
		\includegraphics[width=1.40in, keepaspectratio]{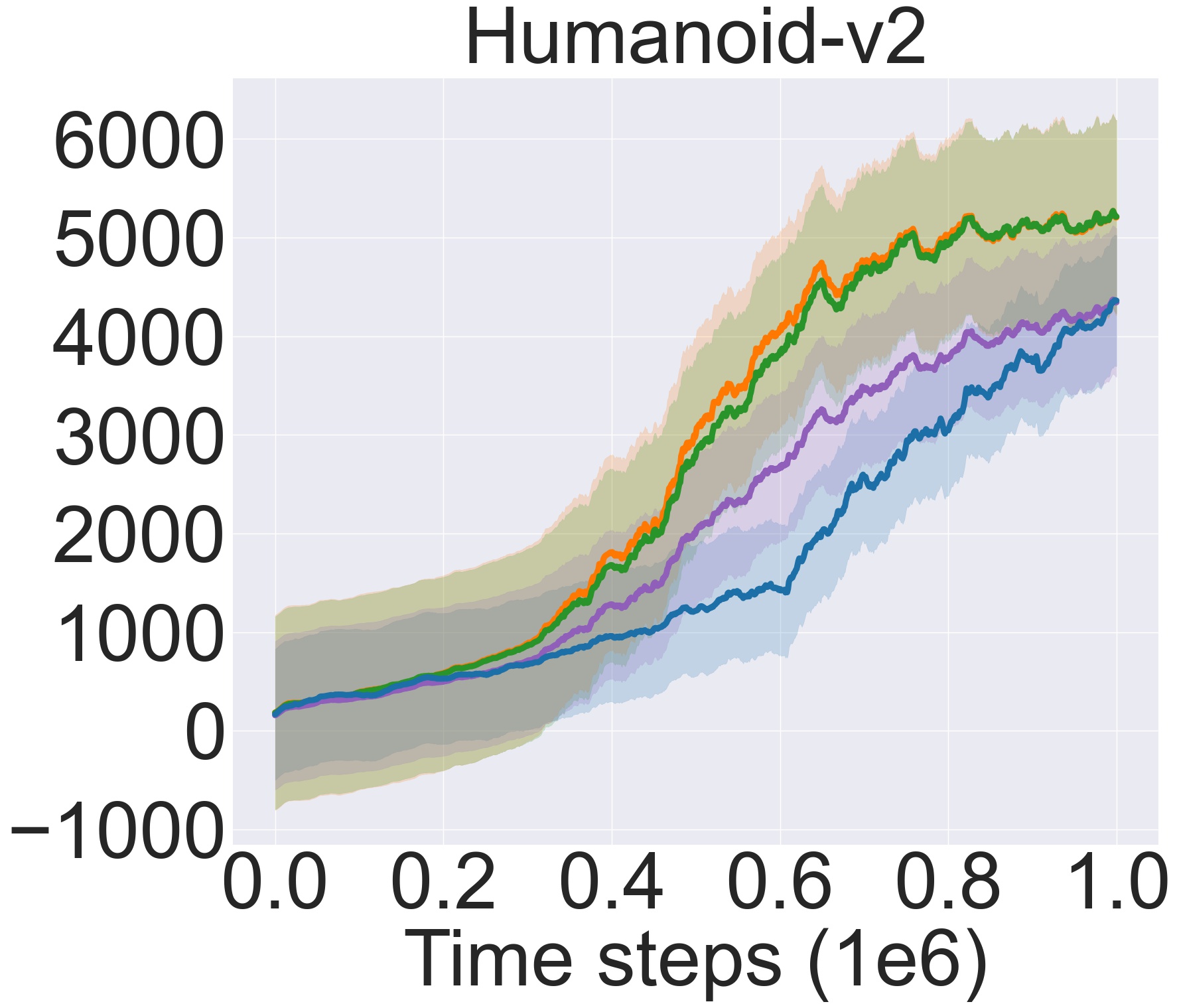}
		\includegraphics[width=1.40in, keepaspectratio]{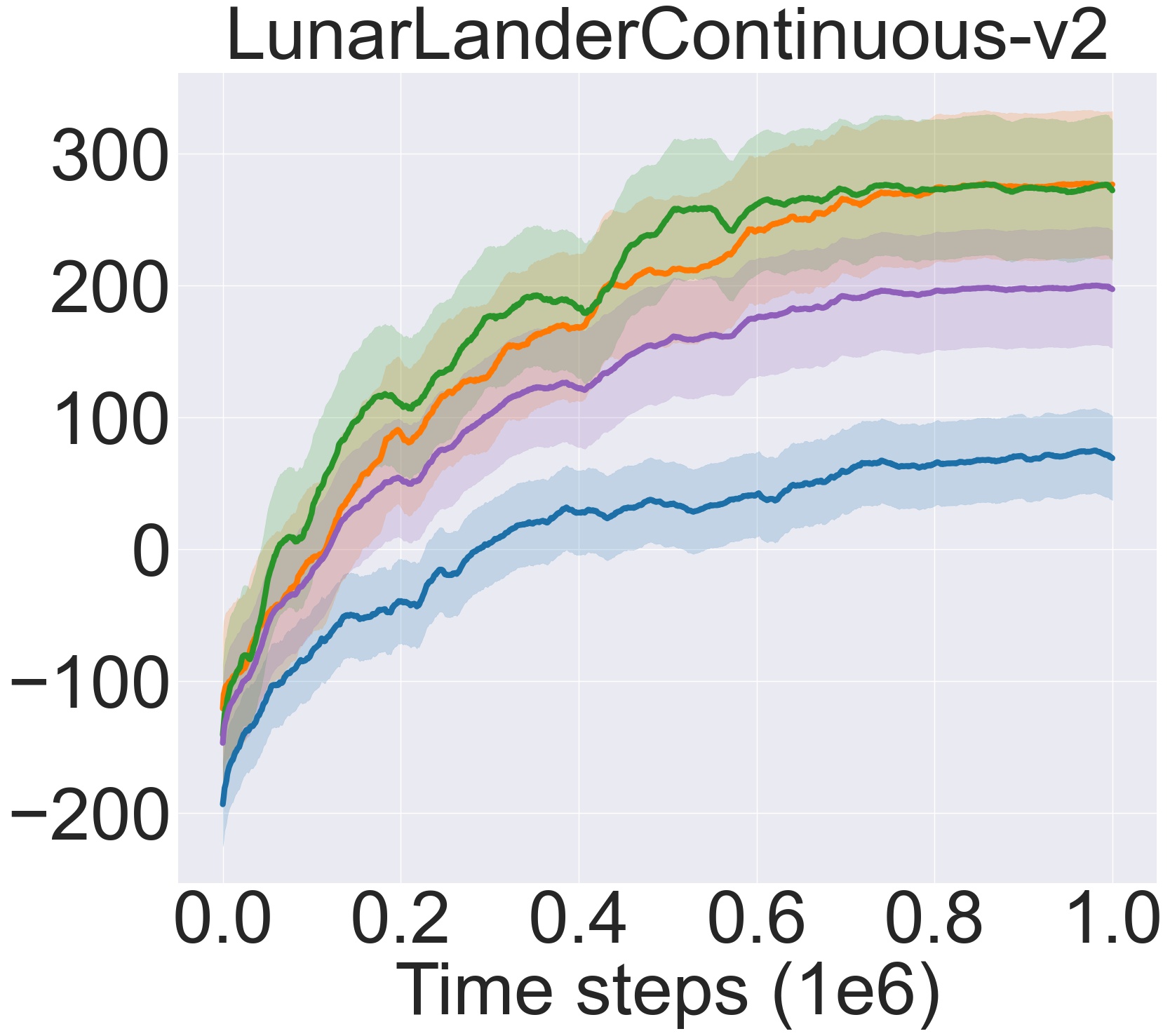}
		\includegraphics[width=1.40in, keepaspectratio]{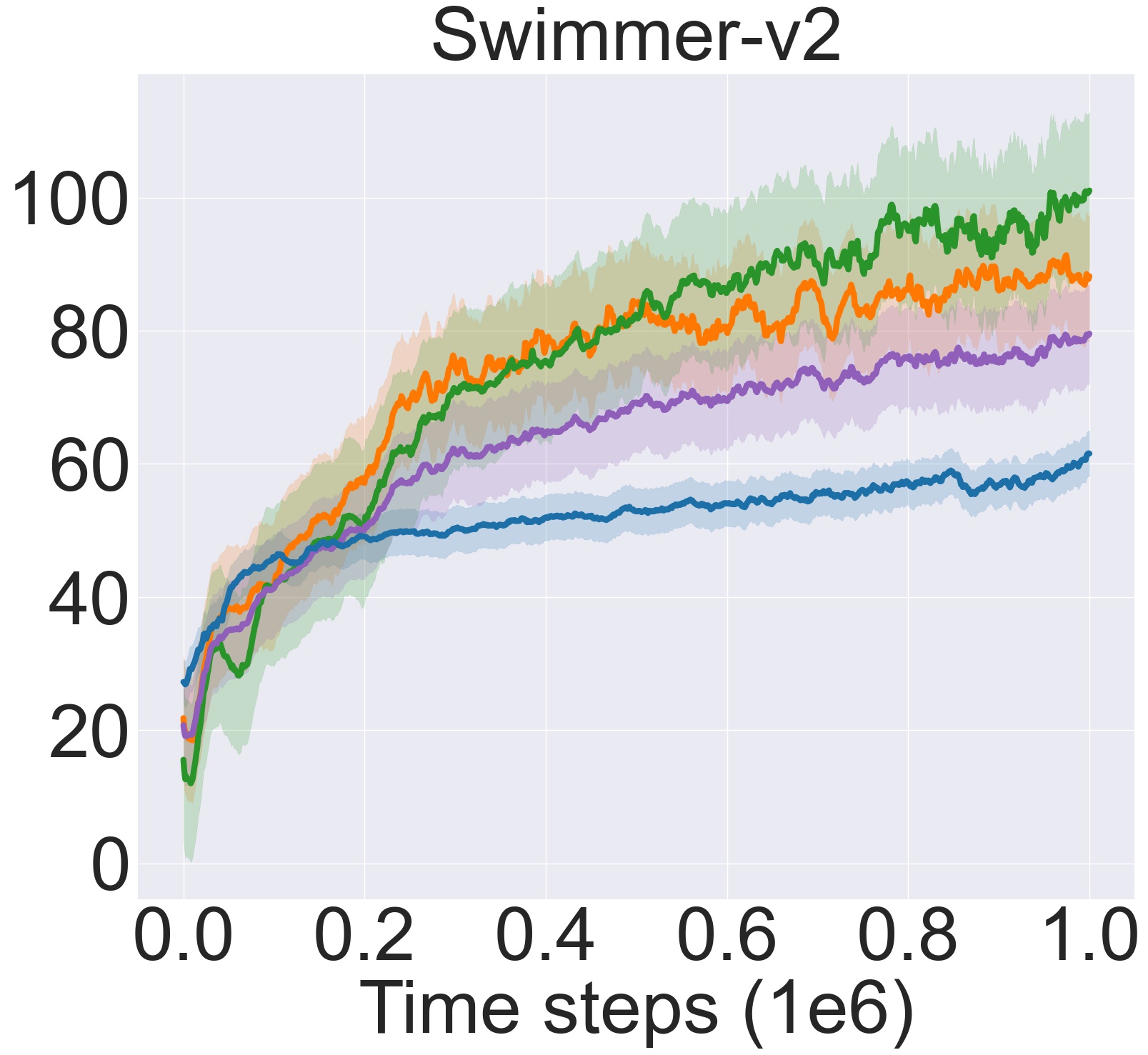}
		\includegraphics[width=1.40in, keepaspectratio]{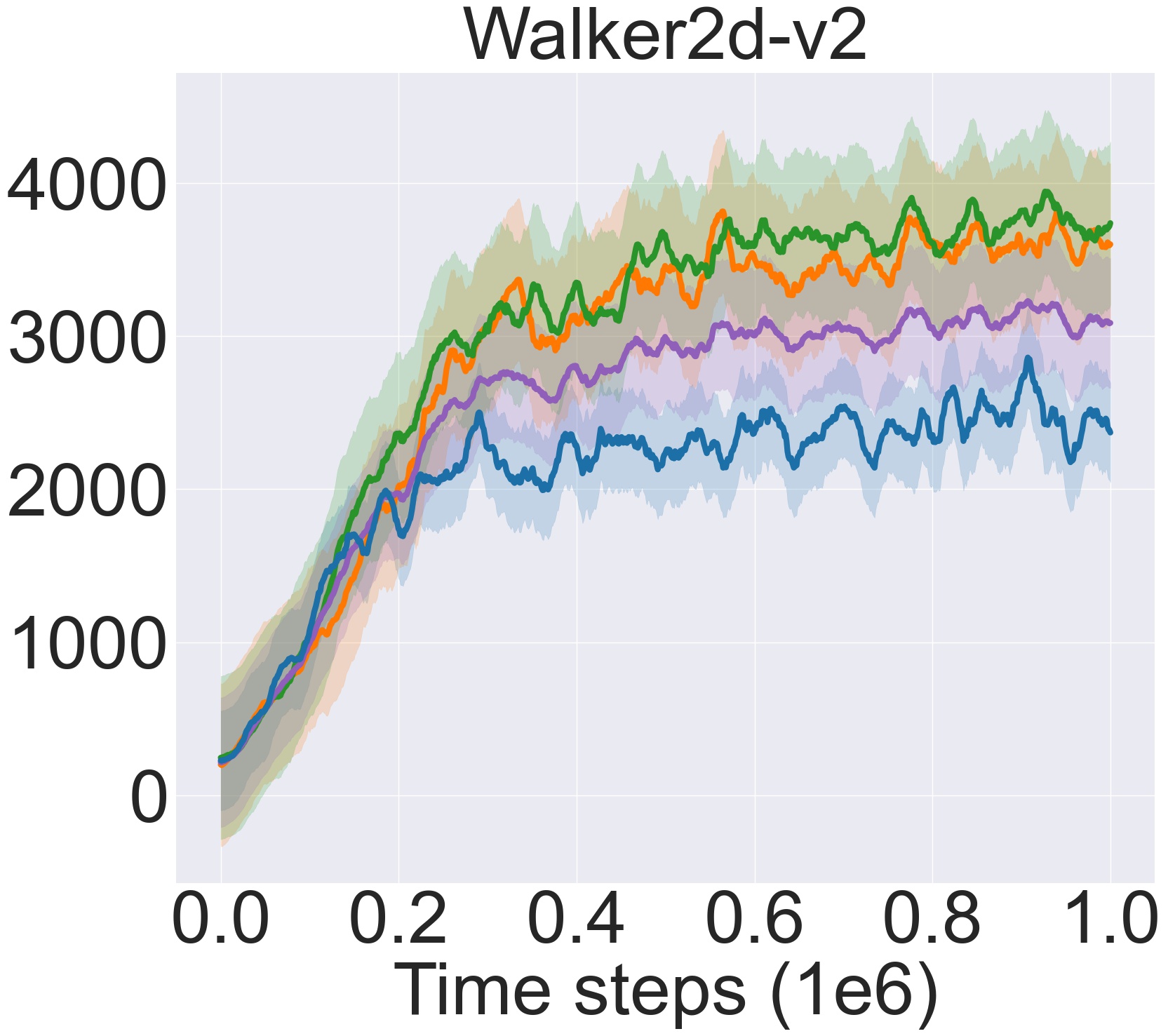}
	}
	\caption{Learning curves for the set of OpenAI Gym continuous control tasks under the SAC algorithm. The shaded region represents half a standard deviation of the average evaluation return over 10 random seeds. A sliding window smoothes curves for visual clarity.}
\end{figure*}

\begin{figure*}[!]
    \centering
   \begin{align*}
        &\text{{\blue} TD3 (single agent)} \quad &&\text{{\orange} TD3 + DASE ($1^{\text{st}}$ agent)} \\
        &\text{\text{{\purple} TD3 + DPD (average of two agents)}} \quad &&\text{{\green} TD3 + DASE ($2^{\text{nd}}$ agent)}
    \end{align*}
	\subfloat[\textbf{Replay Size:} 100,000]{
		\includegraphics[width=1.40in, keepaspectratio]{Figures/CASE_TD3_Ant-v2_Limited_Buffer.jpg}
		\includegraphics[width=1.40in, keepaspectratio]{Figures/CASE_TD3_BipedalWalker-v3_Limited_Buffer.jpg}
		\includegraphics[width=1.40in, keepaspectratio]{Figures/CASE_TD3_HalfCheetah-v2_Limited_Buffer.jpg}
		\includegraphics[width=1.40in, keepaspectratio]{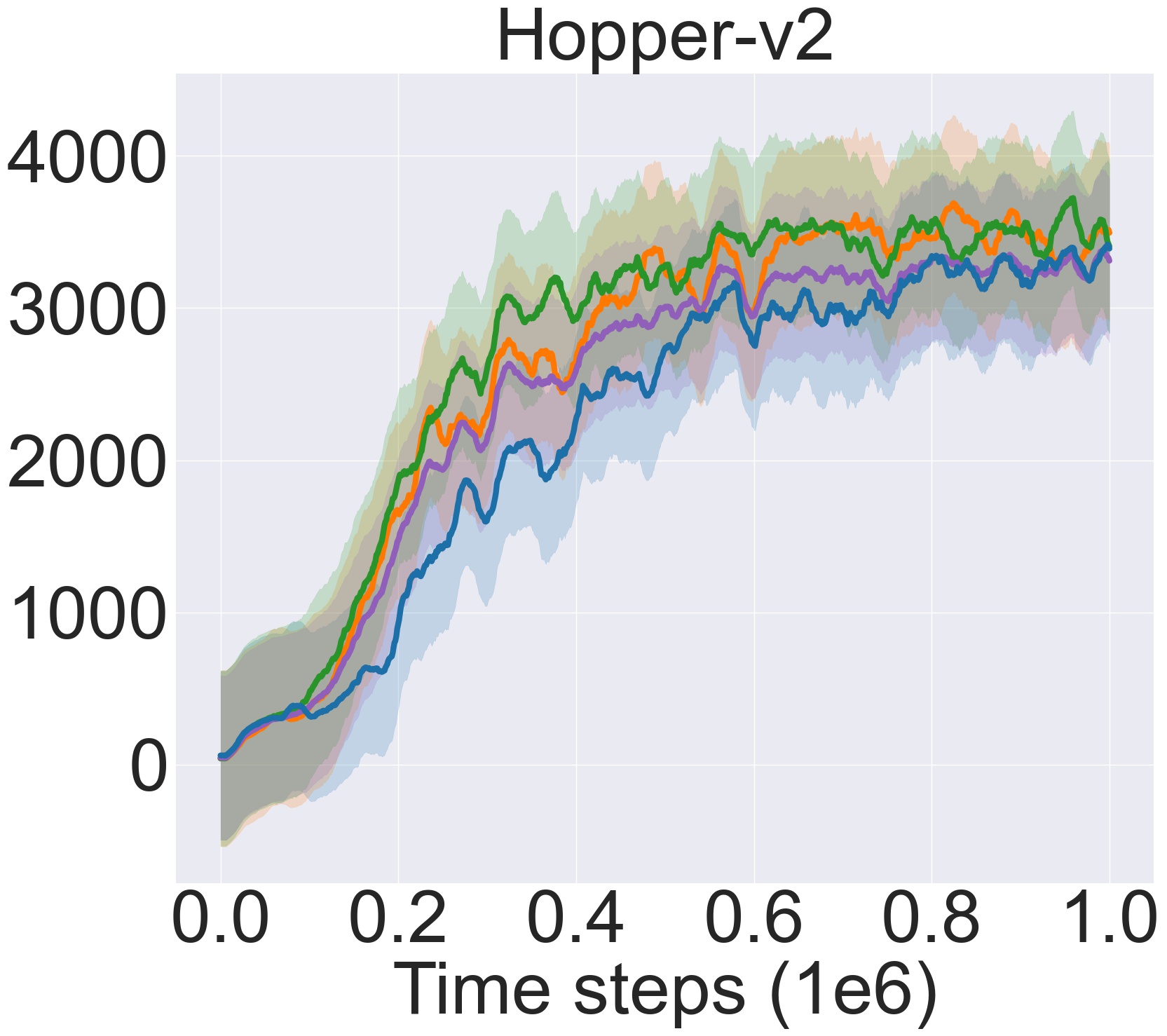}
	} \\
	\subfloat[\textbf{Replay Size:} 1,000,000]{
	    \includegraphics[width=1.40in, keepaspectratio]{Figures/CASE_TD3_Ant-v2_Unlimited_Buffer.jpg}
		\includegraphics[width=1.40in, keepaspectratio]{Figures/CASE_TD3_BipedalWalker-v3_Unlimited_Buffer.jpg} 
		\includegraphics[width=1.40in, keepaspectratio]{Figures/CASE_TD3_HalfCheetah-v2_Unlimited_Buffer.jpg}
		\includegraphics[width=1.40in, keepaspectratio]{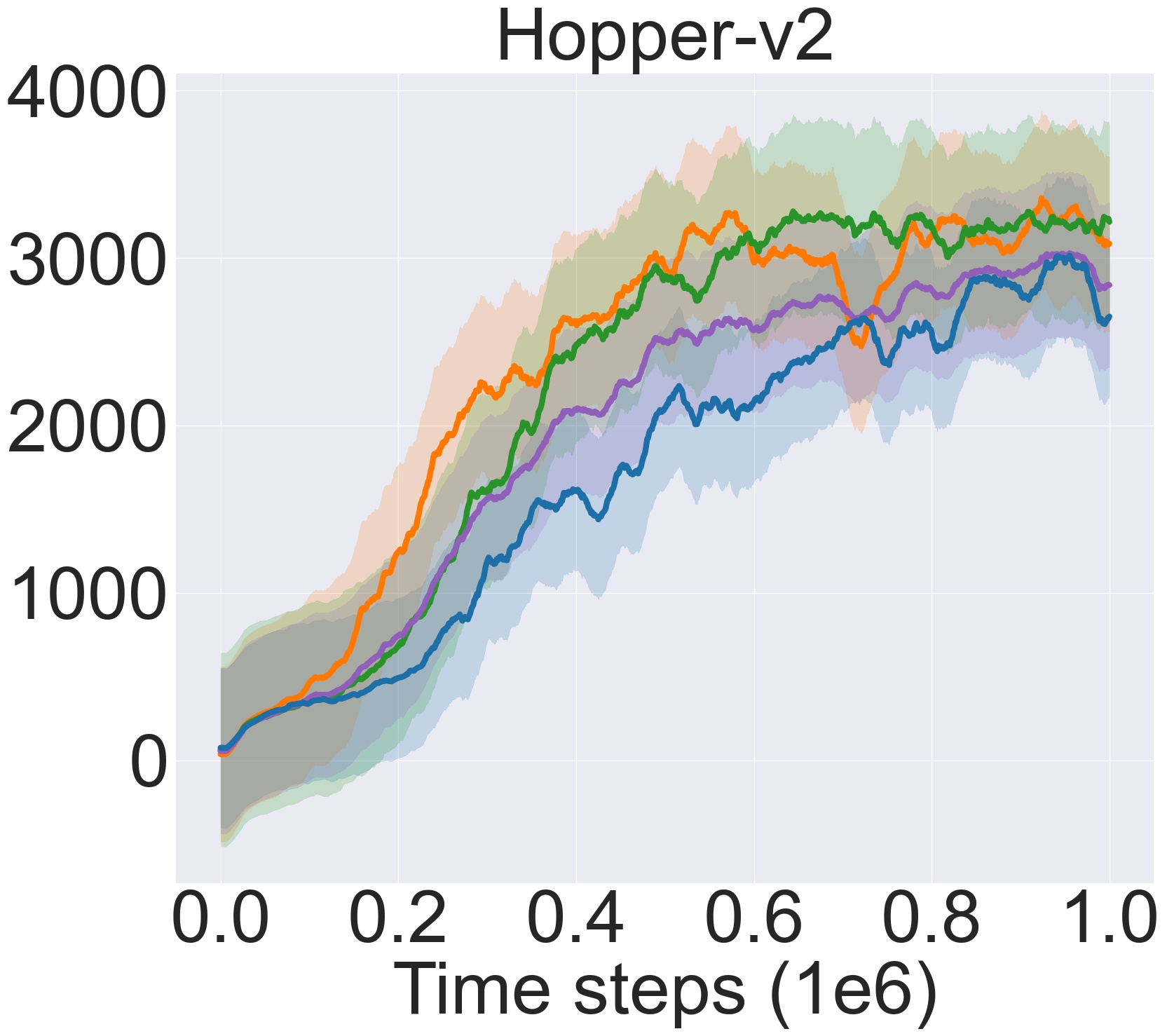}
	} \\
	\subfloat[\textbf{Replay Size:} 100,000]{
	    \includegraphics[width=1.40in, keepaspectratio]{Figures/CASE_TD3_Humanoid-v2_Limited_Buffer.jpg}
		\includegraphics[width=1.40in, keepaspectratio]{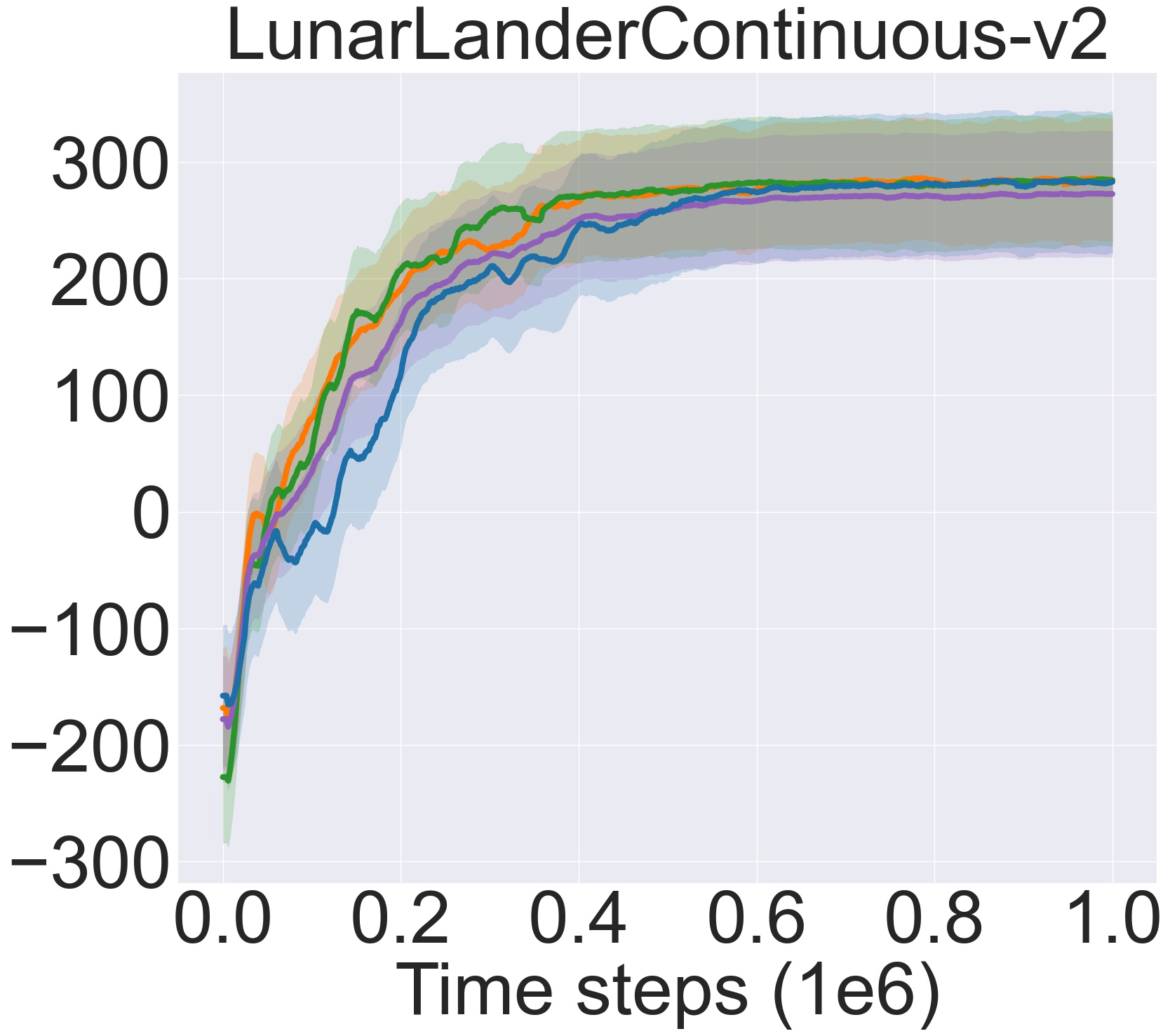}
		\includegraphics[width=1.40in, keepaspectratio]{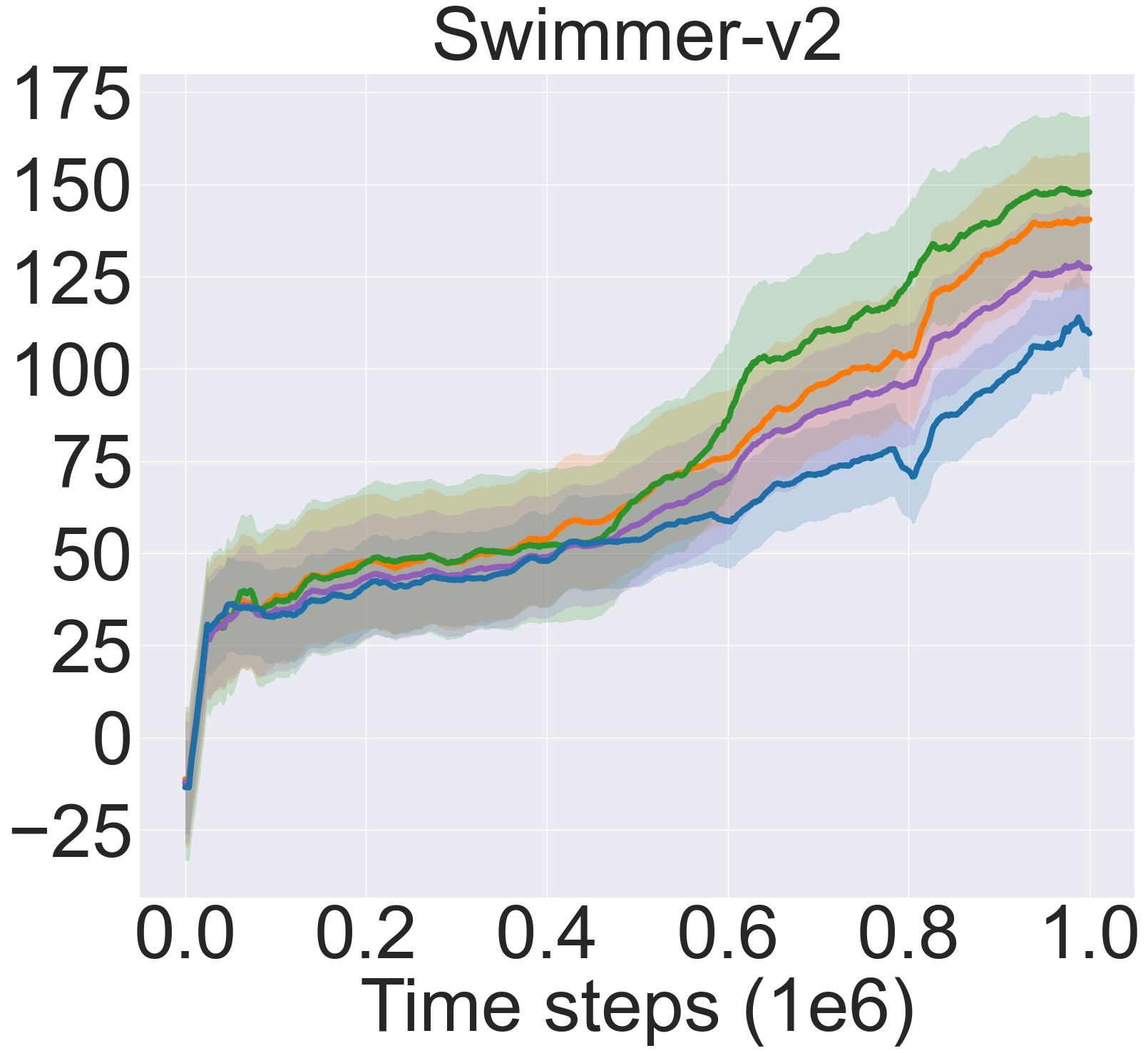}
		\includegraphics[width=1.40in, keepaspectratio]{Figures/CASE_TD3_Walker2d-v2_Limited_Buffer.jpg}
	} \\
	\subfloat[\textbf{Replay Size:} 1,000,000]{
		\includegraphics[width=1.40in, keepaspectratio]{Figures/CASE_TD3_Humanoid-v2_Unlimited_Buffer.jpg}
		\includegraphics[width=1.40in, keepaspectratio]{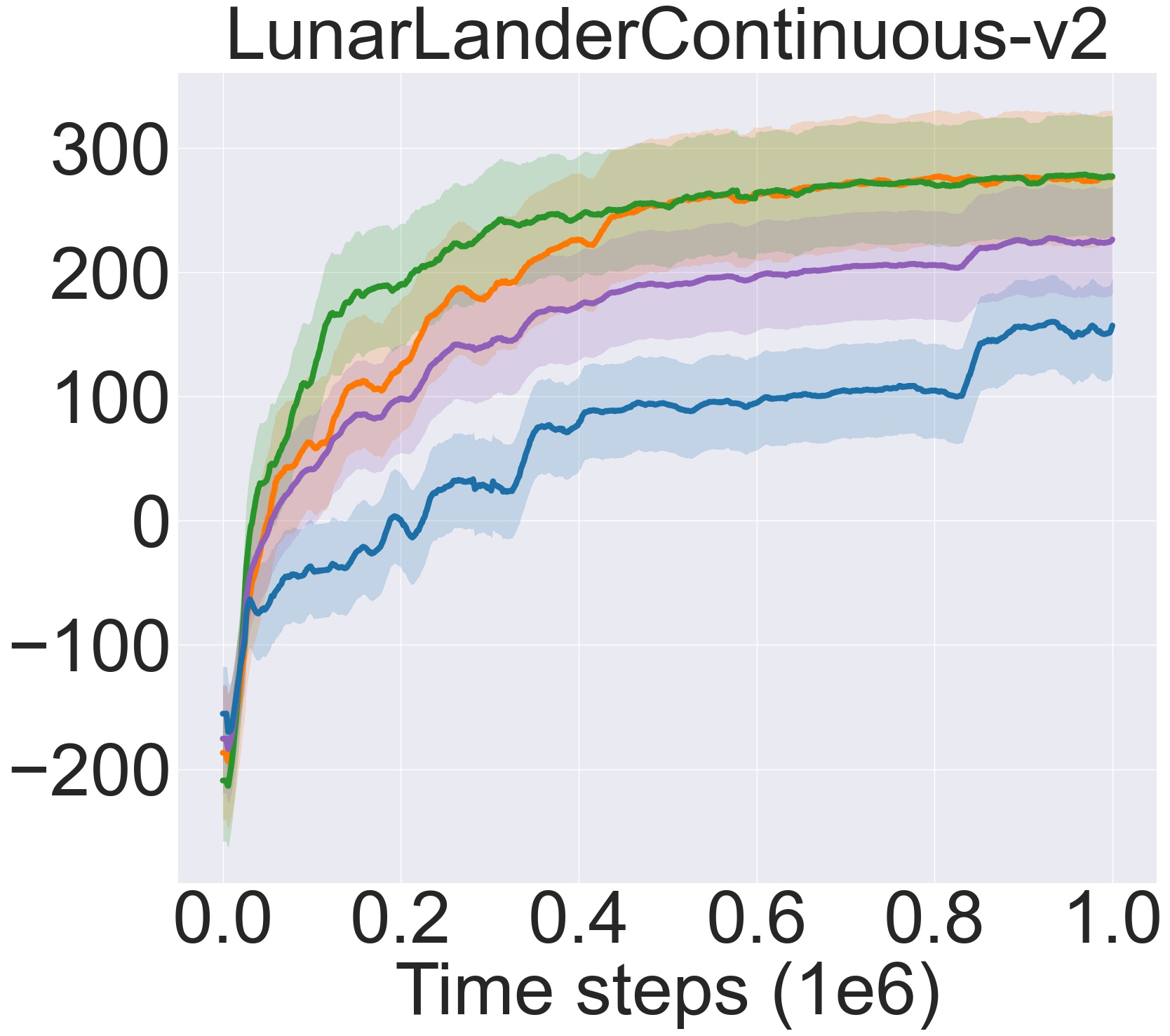}
		\includegraphics[width=1.40in, keepaspectratio]{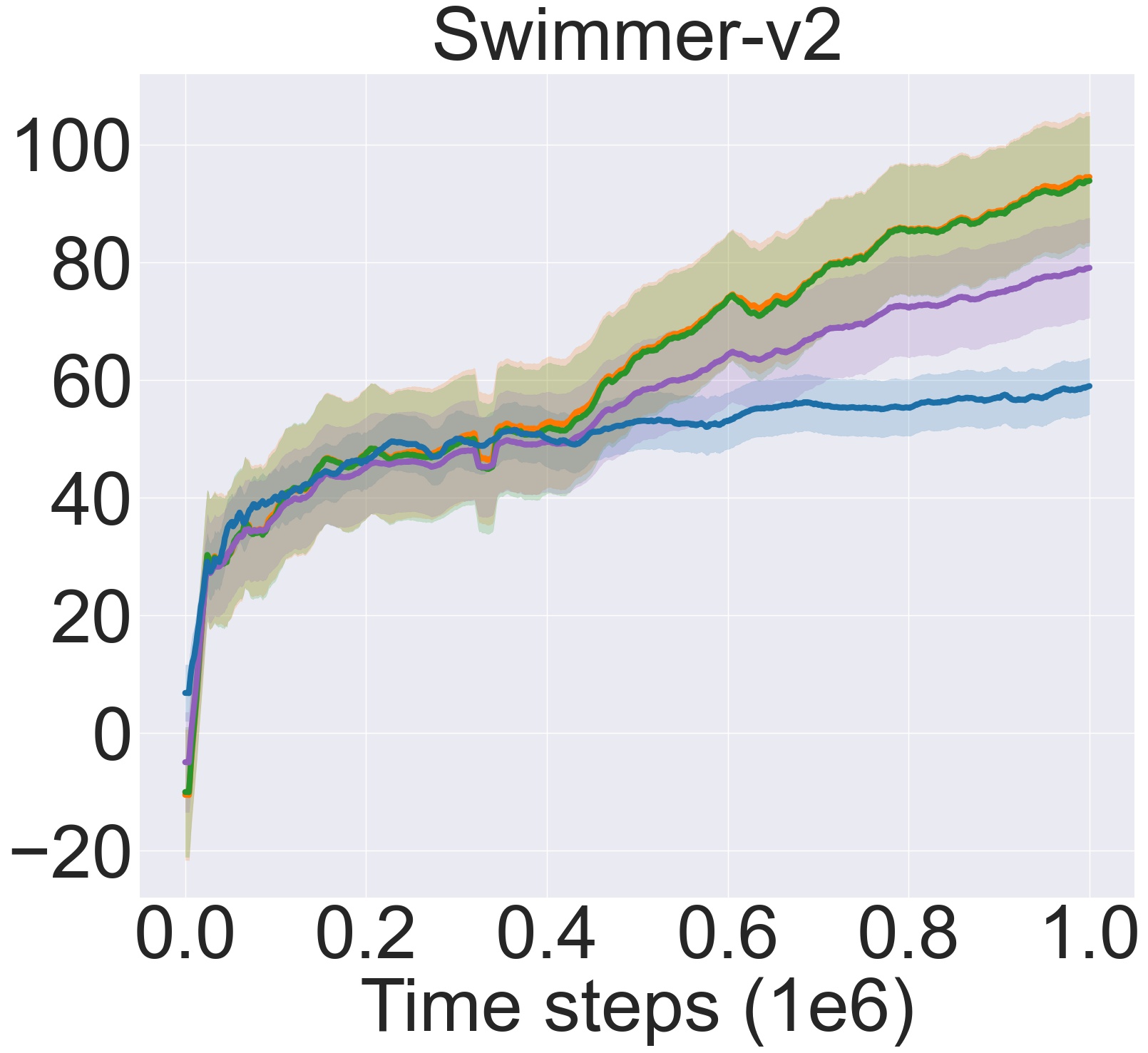}
		\includegraphics[width=1.40in, keepaspectratio]{Figures/CASE_TD3_Walker2d-v2_Unlimited_Buffer.jpg}
	}
	\caption{Learning curves for the set of OpenAI Gym continuous control tasks under the TD3 algorithm. The shaded region represents half a standard deviation of the average evaluation return over 10 random seeds. A sliding window smoothes curves for visual clarity.}
\end{figure*}

\clearpage

\section{Computational Complexity Results}
\label{sec:comp_complexity}

Table \ref{table:comp_complexity_limited} contains the average memory required and time to run baseline algorithms and DASE overall environments, seeds, and the two settings of experience replay buffer~\cite{experience_replay}. We report the computational complexity results when the DASE agents are distributed to multi-threads and multi-processes.

If multi-threads are used, the memory allocation for RAM and GPU does not change as the agents are trained with the same memory allocation for a single agent. For multi-processes, the memory allocation increases for both RAM and GPU. We find that GPU memory allocation is doubled, and RAM allocation is slightly less than the double of the RAM requirement for a single agent due to the shared experience replay~\cite{experience_replay}. We expect the run time of DASE with threads to be approximately the quadruple of the baseline run time and for processes to be double. However, the run time increase for both distribution types is greater than the initial expectations. Such a significant increase is due to the simulation benchmarks we use to conduct our experiments. These simulations run on the CPU. Therefore, latency due to the data marshaling increases the run time and becomes a dominant factor when multiple agents are employed to explore the same environment.

Nonetheless, our algorithm can attain optimal evaluation returns while baseline algorithms may suffer from divergence or be stuck at suboptimal policies. Hence, there is a trade-off between our architecture's computational complexity and performance. Such a trade-off should be carefully handled by considering the environment complexity, such as determining the number of agents by considering how challenging the environment is to be explored and solved or to use threads or processes.

\begin{table}[htb]
\caption{Average memory requirement and run time for the baseline algorithms with and without DASE when agents are distributed to multi-threads and multi-processes. Run time is computed over 1 million time steps. Required memories are in megabytes and running times are in minutes. \label{table:comp_complexity_limited}}
\vskip 0.15in
\begin{center}
\begin{small}
\begin{sc}
\begin{tabular}{lcccc}
\toprule
\textbf{Algorithm} &  \textbf{Memory (RAM)} & \textbf{Memory (GPU)} & \textbf{Run Time} & \textbf{Run Time Increase (\%)} \\
\midrule
    DDPG & 4056 & 1271 & 81.75 & - \\
    SAC & 4087 & 1281 & 99.20  & - \\
    TD3 & 4012 & 1273 & 95.67  & - \\
    DDPG + DASE (thread) & 4102 & 1283 & 401.38 & 491\% \\
    SAC + DASE (thread) & 4118 & 1283 & 497.25 & 501\% \\
    TD3 + DASE (thread) & 4134 & 1287 & 465.41 & 486\% \\
    DDPG + DASE (process) & 7862 & 2542 & 216.98 & 265\% \\
    SAC + DASE (process) & 7924 & 2562 & 269.55 & 271\% \\
    TD3 + DASE (process) & 7774 & 2546 & 254.48 & 266\% \\
\bottomrule
\end{tabular}
\end{sc}
\end{small}
\end{center}
\vskip -0.1in
\end{table}

\newpage

\section{Ablation Studies}
\label{sec:ab_studies}

\begin{table}[htb]
\caption{Ablation results when the replay size is 100,000. Bold values represent the maximum for each environment. Scores represent the average return over all agents of DASE.}
\label{table:ab_limited}
\vskip 0.15in
\begin{center}
\begin{small}
\begin{sc}
\begin{tabular}{lcccc}
\toprule
\textbf{Method}  & \textbf{Ant} & \textbf{HalfCheetah} & \textbf{LunarLanderContinuous} & \textbf{Swimmer} \\
\midrule
    DASE ($K = 2$) & 5123.03 & 12375.37 & 283.95 & 144.28 \\
    DASE ($K = 5$) & 5490.69 & 12764.66 & 306.43 & 155.38 \\
    DASE ($K = 10$) & \textbf{5876.86} & \textbf{13624.52} & \textbf{327.29} & \textbf{164.64} \\
    KL-DASE & 4786.68 & 11036.75 & 236.26 & 119.42 \\
    ES-DASE & 817.75 & 8328.72 & -7.57 & 65.63 \\
\bottomrule
\end{tabular}
\end{sc}
\end{small}
\end{center}
\vskip -0.1in
\end{table}

\begin{table}[htb]
\caption{Ablation results when the replay size is 1,000,000. Bold values represent the maximum for each environment. Scores represent the average return over all agents of DASE.}
\label{table:ab_unlimited}
\vskip 0.15in
\begin{center}
\begin{small}
\begin{sc}
\begin{tabular}{lcccc}
\toprule
\textbf{Method}  & \textbf{Ant} & \textbf{HalfCheetah} & \textbf{LunarLanderContinuous} & \textbf{Swimmer} \\
\midrule
    DASE ($K = 2$) & 5434.32 & 12086.72 & 277.20 & 93.81 \\
    DASE ($K = 5$) & 5921.87 & 12826.30 & 297.88 & 100.43 \\
    DASE ($K = 10$) & \textbf{6217.24} & \textbf{13674.15} & \textbf{318.34} & \textbf{107.22} \\
    KL-DASE & 5094.35 & 11075.49 & 227.78 & 77.25 \\
    ES-DASE & 865.57 & 8354.15 & 0.34 & 64.17 \\
\bottomrule
\end{tabular}
\end{sc}
\end{small}
\end{center}
\vskip -0.1in
\end{table}

We perform ablation studies to analyze the effects of the components: the usage of JS-divergence, number of agents, and off-policy correction (DPS). For this, we compare ablation over DASE under $K = \{2, 5, 10\}$, DASE with KL-divergence (KL-DASE), DASE without off-policy correction (no DPS is applied), and only with experience sharing (ES-DASE). As DASE is orthogonal to any off-policy deterministic policy gradient algorithm, ablation studies are performed under the TD3 algorithm~\cite{td3}. Ablation results are given in Table \ref{table:ab_limited} and \ref{table:ab_unlimited}, where average return over the last 10 evaluations over 10 trials of 1 million time steps is reported. Learning curves for the ablation studies can be found in our repository\footref{our_repo}.

The complete algorithm outperforms every combination except when the number of agents increases. As the off-policy samples by other agents are safely corrected, the increasing number of agents yields more diverse exploration and thus slightly higher returns and faster convergence. However, training more agents linearly increases the training duration. We then replace the JS-divergence with KL-divergence (KL-DASE) in DPS. We obtain higher returns with JSD due to the symmetric measurement of the policies. Directed similarity measurement slightly degrades the algorithm's performance as two policies in the same environment cannot be completely distinct. Finally, we remove the off-policy correction in learning from other agents' experiences. The algorithm cannot converge and exhibits randomness in the action selection. Our theoretical approach is reflected empirically, that is, extrapolation error~\cite{off_policy} prevents agents from converging high evaluation returns due to the mismatch between the distributions under the agent's policy and samples collected by other agents.

\end{document}